\documentclass{article}
\usepackage[accepted, nohyperref]{icml/icml2020}

\usepackage{microtype}
\usepackage{graphicx}
\usepackage[caption=false,font=footnotesize]{subfig} % sub-figures
\usepackage{booktabs} % for professional tables

\icmltitlerunning{On the Convergence of Nesterov's Accelerated Gradient Method}

\usepackage{amssymb}        % used to create math symbols (e.g., \gtrapprox)
\usepackage{amsthm}         % used to define theorems and proofs
\usepackage{xspace}         % used to auto-insert spaces for math commands
\usepackage{mathtools}      % used to make math symbols (e.g., \coloneqq)
\usepackage{nicefrac}       % used to put fractions inline
\usepackage{chngcntr}

\newcommand{\itr}[2]{ {#1}_{#2} }

\newcommand{\noise}{\zeta}
\newcommand{\R}{\mathbb{R}}
\newcommand{\E}{\mathbb{E}}

\newcommand{\NAG}{\textsc{ag}\xspace}
\newcommand{\NASG}{\textsc{asg}\xspace}
\newcommand{\QHM}{\textsc{qhm}\xspace}
\newcommand{\GD}{\textsc{gd}\xspace}
\newcommand{\SGD}{\textsc{sgd}\xspace}
\newcommand{\cond}{Q}
\newcommand{\eg}{e.g.}
\newcommand{\ie}{i.e.}
\newcommand{\norm}[1]{\left\lVert#1\right\rVert}
\newcommand{\abs}[1]{\left\lvert#1\right\rvert}

\newcommand{\xstar}{x^\star}

\newcommand{\argmax}{\text{argmax}}
\newcommand{\std}{\sigma}

\newcommand{\defeq}{\coloneqq}
\renewcommand{\Pr}{\mathbb{P}}
\newcommand{\T}{\top}
\newcommand{\order}[1]{\mathcal{O}\left( #1 \right)}
\newcommand{\diag}[1]{\operatorname{diag}( #1 )}

\newcommand{\radius}[1]{\rho\left( #1 \right)}

\newtheorem{theorem}{Theorem}
\newtheorem{corollary}{Corollary}[theorem]
\newtheorem{lemma}{Lemma}

\begin{document}

\twocolumn[
\icmltitle{On the Convergence of Nesterov's Accelerated Gradient Method \\in Stochastic Settings}

\begin{icmlauthorlist}
\icmlauthor{Mahmoud Assran}{mcgill,fb,mila}
\icmlauthor{Michael Rabbat}{fb,mila}
\end{icmlauthorlist}

\icmlaffiliation{mcgill}{Department of Electrical \& Computer Engineering, McGill University, Montreal, QC, Canada}
\icmlaffiliation{fb}{Facebook AI Research, Montreal, QC, Canada}
\icmlaffiliation{mila}{Mila -- Quebec Artificial Intelligence Institute, Montreal, QC, Canada}

\icmlcorrespondingauthor{Mahmoud Assran}{massran@fb.com}
\icmlcorrespondingauthor{Michael Rabbat}{mikerabbat@fb.com}

\icmlsetsymbol{equal}{*}

\icmlkeywords{Nesterov, Momentum, Accelerated Gradient Method, Stochastic, Stochastic Approximation, Finite-Sum, Optimization, Machine Learning, ICML}
\vskip 0.3in
]
\printAffiliationsAndNotice{}  % leave blank if no need to mention equal contribution

\begin{abstract}
We study Nesterov's accelerated gradient method with constant step-size and momentum parameters in the stochastic approximation setting (unbiased gradients with bounded variance) and the finite-sum setting (where randomness is due to sampling mini-batches). To build better insight into the behavior of Nesterov's method in stochastic settings, we focus throughout on objectives that are smooth, strongly-convex, and twice continuously differentiable. In the stochastic approximation setting, Nesterov's method converges to a neighborhood of the optimal point at the same accelerated rate as in the deterministic setting. Perhaps surprisingly, in the finite-sum setting, we prove that Nesterov's method may diverge with the usual choice of step-size and momentum, unless additional conditions on the problem related to conditioning and data coherence are satisfied. Our results shed light as to why Nesterov's method may fail to converge or achieve acceleration in the finite-sum setting. 
\end{abstract}

\section{Introduction}

First-order stochastic methods have become the workhorse of machine learning, where many tasks can be cast as optimization problems,
\begin{equation} \label{eq:obj}
    \mathop{\operatorname{minimize}}_{x \in \R^d} f(x).
\end{equation}
Methods incorporating momentum and acceleration play an important role in the current practice of machine learning~\cite{sutskever2013importance,bottou2018optimization}, where they are commonly used in conjunction with stochastic gradients. However, the theoretical understanding of accelerated methods remains limited when used with stochastic gradients.

This paper studies the \emph{accelerated gradient} (\NAG) method of \citet{nesterov1983method} with constant step-size and momentum parameters. Given an initial point $x_0$, and with $x_{-1} = x_0$, the \NAG method repeats, for $k \ge 0$,
\begin{align}
    \itr{y}{k+1} &= \itr{x}{k} + \beta (\itr{x}{k} - \itr{x}{k-1}) \label{eq:nasg-y} \\
    \itr{x}{k+1} &= \itr{y}{k+1} - \alpha \itr{g}{k+1}, \label{eq:nasg-x}
\end{align}
where $\alpha$ and $\beta$ are the step-size and momentum parameters, respectively, and in the deterministic setting, $\itr{g}{k+1} = \nabla f(\itr{y}{k+1})$. When the momentum parameter $\beta$ is $0$, \NAG simplifies to standard \emph{gradient descent} (\GD).
When $\beta > 0$ it is possible to achieve accelerated rates of convergence for certain combinations of $\alpha$ and $\beta$ in the deterministic setting.

\subsection{Previous Work with Deterministic Gradients}

Suppose that the objective function in~\eqref{eq:obj} is $L$-smooth and $\mu$-strongly-convex. Then $f$ is minimized at a unique point $x^\star$, and we denote its minimum by $f^\star = f(x^\star)$. Let $\cond \defeq \nicefrac{L}{\mu}$ denote the condition number of $f$. In the deterministic setting, where $\itr{g}{k} = \nabla f(\itr{y}{k})$ for all $k$, \GD with constant step-size $\alpha = \nicefrac{2}{(L + \mu)}$ converges at the rate \citep{polyakIntro}
\begin{equation}
    \label{eq:gd-og}
    f(\itr{x}{k}) - f^\star \leq \frac{L}{2} \left(\frac{\cond - 1}{\cond + 1}\right)^{2k} \norm{\itr{x}{0} - \xstar}^2.
\end{equation}
The \NAG method with constant step-size $\alpha=\nicefrac{1}{L}$ and momentum parameter $\beta=\frac{\sqrt{\cond} - 1}{\sqrt{\cond} + 1}$ converges at the rate \citep{nesterov2004introductory}
\begin{equation}
    \label{eq:nag-og}
    f(\itr{x}{k}) - f^\star \leq L \left(\frac{\sqrt{\cond} - 1}{\sqrt{\cond}}\right)^k \norm{\itr{x}{0} - \xstar}^2.
\end{equation}
The rate in \eqref{eq:nag-og} matches (up to constants) the tightest-known worst-case lower bound achievable by any first-order black-box method for $\mu$-strongly-convex and $L$-smooth objectives:
\begin{equation}
    \label{eq:lower-bound}
    f(\itr{x}{k}) - f^\star
    \geq
    \frac{\mu}{2} \left( \frac{\sqrt{\cond} - 1}{\sqrt{\cond}+1} \right)^{2k}
    \norm{\itr{x}{0} - \xstar}^2.
\end{equation}
The lower bound~\eqref{eq:lower-bound} is proved in~\citet{nesterov2004introductory} in the infinite dimensional setting under the assumption $\cond > 1$.
Accordingly, Nesterov's Accelerated Gradient method is considered optimal in the sense that the convergence rate in~\eqref{eq:nag-og} depends on $\sqrt{\cond}$ rather than $Q$.

The proof of \eqref{eq:nag-og} presented in \citet{nesterov2004introductory} uses the method of estimate sequences. Several works have set out to develop better intuition for how the \NAG method achieves acceleration though other analysis techniques.

One line of work considers the limit of infinitessimally small step-sizes, obtaining ordinary differential equations (ODEs) that model the trajectory of the \NAG method \cite{su2014differential,defazio2019curved,laborde2019lyapunov}. \citet{allen2014linear} view the \NAG method as an alternating iteration between mirror descent and gradient descent and show sublinear convergence of the \NAG method for smooth convex objectives. 

\citet{lessard2016analysis} and \citet{hu2017dissipativity} frame the \NAG method and other popular first-order optimization methods as linear dynamical systems with feedback and characterize their convergence rate using a control-theoretic stability framework. The framework leads to closed-form rates of convergence for strongly-convex quadratic functions with deterministic gradients. For more general (non-quadratic) deterministic problems, the framework provides a means to numerically certify rates of convergence.

\subsection{Previous Work with Stochastic Gradients}

When Nesterov's method is run with stochastic gradients $\itr{g}{k+1}$, typically satisfying $\E[\itr{g}{k+1}] = \nabla f(\itr{y}{k+1})$, we refer to it as the \emph{accelerated stochastic gradient} (\NASG) method. In this setting, if $\beta = 0$ then \NASG is equivalent to \emph{stochastic gradient descent} (\SGD).

Despite the widespread interest in, and use of, the \NASG method, there are no definitive theoretical convergence guarantees. \citet{wiegerinck1994stochastic} study the \NASG method in an online learning setting and show that optimization can be modelled as a Markov process but do not provide convergence rates.
\citet{yang2016unified} study the \NASG method in the smooth strongly-convex setting, and show an $\mathcal{O}(1/\sqrt{k})$ convergence rate when employed with a diminishing step-size and bounded gradient assumption, but the rates obtained are slower than those for \SGD.

Recent work establishes convergence guarantees for the \NASG method in certain restricted settings. \citet{aybat2019robust} and \citet{kulunchakov2019estimate} consider smooth strongly-convex functions in a stochastic approximation model with gradients that are unbiased and have bounded variance, and they show convergence to a neighborhood when running the method with constant step size and momentum. \citet{can2019accelerated} further establish convergence in Wasserstein distribution under a stochastic approximation model. \citet{laborde2019lyapunov} study a perturbed ODE and show convergence for diminishing step-size. \citet{vaswani2019fast} study the \NASG method with constant step-size and diminishing momentum, and show linear convergence under a strong-growth condition, where the gradient variance vanishes at a stationary point.

Some results are available for other momentum schemes. \citet{loizou2017momentum} study Polyak's heavy-ball momentum method with stochastic gradients for randomized linear problems and show that it converges linearly under an exactness assumption. \citet{gitman2019understanding} characterize the stationary distribution of the Quasi-Hyperbolic Momentum (\QHM) method~\citep{ma2018quasihyperbolic} around the minimizer for strongly-convex quadratic functions with bounded gradients and bounded gradient noise variance.

The lack of general convergence guarantees for existing momentum schemes, such as Polyak's and Nesterov's, have led many authors to develop alternative accelerated methods specifically for use with stochastic gradients~\cite{lan2012optimal,ghadimi2012optimal,ghadimi2013optimal,allen2017katyusha,kidambi2018insufficiency,cohen2018acceleration,kulunchakov2019generic,liu2020accelerating}.

Accelerated first-order methods are also known to be sensitive to inexact gradients when the gradient errors are deterministic (possibly adversarial) and bounded \citep{daspremont2008smooth,devolder2014first}.

\subsection{Contributions}

We provide additional insights into the behavior of Nesterov's accelerated gradient method when run with stochastic gradients by considering two different settings. We first consider the stochastic approximation setting, where the gradients used by the method are unbiased, conditionally independent from iteration to iteration, and have bounded variance. We show that Nesterov's method converges at an accelerated linear rate to a region of the optimal solution for smooth strongly-convex quadratic problems.

Next, we consider the finite-sum setting, where $f(x) = \frac{1}{n}\sum_{i=1}^n f_i(x)$, under the assumption that each term $f_i$ is smooth and strongly-convex, and the only randomness is due to sampling one or a mini-batch of terms at each iteration.
In this setting we prove that, even when all functions $f_i$ are quadratic, Nesterov's \NASG method with the usual choice of step-size and momentum cannot be guaranteed to converge without making additional assumptions on the condition number and data distribution.
When coupled with convergence guarantees in the stochastic approximation setting, this impossibility result illuminates the dichotomy between our understanding of momentum-based methods in the stochastic approximation setting, and practical implementations of these methods in a finite-sum framework.

Our results also shed light as to why Nesterov's method may fail to converge or achieve acceleration in the finite-sum setting, providing further insight into what has previously been reported based on empirical observations. In particular, the bounded-variance assumption does not apply in the finite-sum setting with quadratic objectives.

We also suggest choices of the step-size and momentum parameters under which the \NASG method is guaranteed to converge for any smooth strongly-convex finite-sum, but where accelerated rates of convergence are no longer guaranteed.
Our analysis approach leads to new bounds on the convergence rate of \SGD in the finite-sum setting, under the assumption that each term $f_i$ is smooth, strongly-convex, and twice continuously differentiable.

\section{Preliminaries and Analysis Framework}
\label{sec:preliminaries}

In this section we establish a basic framework for analyzing the \NAG method. Then we specialize it to the stochastic approximation and finite-sum setting settings, respectively, in Sections~\ref{sec:stochastic-approximation} and~\ref{sec:finite-sum}.

Throughout this paper we assume that $f$ is twice-continuously differentiable, $L$-smooth, and $\mu$-strongly convex, with $0 < \mu \le L$; see \eg, \citet{nesterov2004introductory,bubeck2015convex}. 
Examples of typical tasks satisfying these assumptions are $\ell_2$-regularized logistic regression and $\ell_2$-regularized least-squares regression (\ie, ridge regression).
Taken together, these properties imply that the Hessian $\nabla^2 f(x)$ exists, and for all $x \in \R^d$ the eigenvalues of $\nabla^2 f(x)$ lie in the interval $[\mu, L]$. Also, recall that $x^\star$ denotes the unique minimizer of $f$ and $f^\star = f(x^\star)$.

In contrast to all previous work we are aware of, our analysis focuses on the sequence $(y_k)_{k \ge 0}$ generated by the method \eqref{eq:nasg-y}--\eqref{eq:nasg-x}. Let $r_k \defeq y_k - x^\star$ denote the suboptimality of the current iterate, and let $v_k \defeq x_k - x_{k-1}$ denote the velocity.

Substituting the definition of $y_{k+1}$ from \eqref{eq:nasg-y} into \eqref{eq:nasg-x} and rearranging, we obtain
\begin{equation} \label{eq:nasg-v}
v_{k+1} = \beta v_k - \alpha g_{k+1}.
\end{equation}
By using the definition of $v_k$, substituting \eqref{eq:nasg-v} and \eqref{eq:nasg-x} into \eqref{eq:nasg-y}, and rearranging, we also obtain that
\begin{equation} \label{eq:nasg-r}
r_{k+1} = r_k + \beta^2 v_{k-1} - \alpha (1 + \beta) g_k.
\end{equation}
Combining \eqref{eq:nasg-v} and \eqref{eq:nasg-r}, we get the recursion
\begin{equation} \label{eq:recursion-with-g}
\begin{bmatrix} r_{k+1} \\ v_k \end{bmatrix} = \begin{bmatrix} I & \beta^2 I \\ 0 & \beta I \end{bmatrix} \begin{bmatrix} r_k \\ v_{k-1} \end{bmatrix} - \alpha \begin{bmatrix} (1 + \beta) I \\ I \end{bmatrix} g_k.
\end{equation}
Note that $r_1 = x_0 - x^\star$ and $v_0 = 0$ based on the common convention that $x_{-1} = x_0$.

Our analysis below will build on the recursion \eqref{eq:recursion-with-g} and will also make use of the basic fact that if $f : \R^d \rightarrow R$ is twice continuously differentiable then for all $x, y \in \R^d$
\begin{equation} \label{eq:taylor}
    \nabla f(y) = \nabla f(x) + \int_0^1 \nabla^2 f(x + t(y - x)) \mathrm{d}t \; (y - x).
\end{equation}

\section{The Stochastic Approximation Setting}
\label{sec:stochastic-approximation}

Now consider the stochastic approximation setting. We assume, for all $k$, that $g_k$ is a random vector satisfying
\[
\E[g_k] = \nabla f(y_k)
\]
and that there is a finite constant $\sigma^2$ such that
\[
\E\left[ \norm{g_k - \nabla f(y_k)}^2 \right] \le \sigma^2.
\]
Let $\zeta_k = g_k - \nabla f(y_k)$ denote the gradient noise at iteration $k$, and suppose that these gradient noise terms are mutually independent.
Applying \eqref{eq:taylor} with $y = y_k$ and $x = x^\star$, we get that
\[
g_k = H_k r_k + \zeta_k, \quad \text{where} \quad 
H_k = \int_0^1 \nabla^2 f(x^\star + t r_k) \mathrm{d}t.
\]
Using this in \eqref{eq:recursion-with-g}, we find that $r_k$ and $v_k$ evolve according to
\begin{equation} \label{eq:stochapprox-recursion}
    \begin{bmatrix} r_{k+1} \\ v_k \end{bmatrix} = A_k \begin{bmatrix} r_k \\ v_{k-1} \end{bmatrix} - \alpha \begin{bmatrix} (1 + \beta) I \\ I \end{bmatrix} \zeta_k,
\end{equation}
where
\begin{equation} \label{eq:stochapprox-Ak}
    A_k = \begin{bmatrix}
        I - \alpha (1 + \beta) H_k & \beta^2 I \\
        - \alpha H_k & \beta I
    \end{bmatrix}.
\end{equation}
Unrolling the recursion \eqref{eq:stochapprox-recursion}, we get that
\begin{align} 
    \begin{bmatrix} r_{k+1} \\ v_k \end{bmatrix} &= (A_k \cdots A_1) \begin{bmatrix} x_0 - x^\star \\ 0\end{bmatrix} - \alpha \begin{bmatrix}
        (1 + \beta) I \\ I
    \end{bmatrix} \zeta_k \notag \\
    &\quad - \alpha \sum_{j=1}^{k-1} (A_k \cdots A_{j+1}) \begin{bmatrix}
        (1 + \beta)I \\ I
    \end{bmatrix} \zeta_j, \label{eq:stochapprox-unrolled}
\end{align}
from which it is clear that we may expect convergence properties to depend on the matrix products $A_k \cdots A_j$.

\subsection{The quadratic case}

We can explicitly bound the matrix product in the specific case where $f(x) = \frac{1}{2} x^\T H x - b^\T x + c$, for a symmetric matrix $H \in \R^{d \times d}$, and with $b \in \R^d$ and $c \in \R$. In this case, \eqref{eq:stochapprox-unrolled} simplifies to
\begin{equation} \label{eq:stochapprox-quadratic-unrolled}
\begin{bmatrix} r_{k+1} \\ v_k \end{bmatrix} = A^k \begin{bmatrix} x_0 - x^\star \\ 0 \end{bmatrix} - \alpha \sum_{j=1}^k A^{k - j} \begin{bmatrix} (1 + \beta)I \\ I\end{bmatrix} \zeta_j,
\end{equation}
where
\begin{equation} \label{eq:stochapprox-quadratic-A}
    A = \begin{bmatrix}
        I - \alpha (1 + \beta) H & \beta^2 I \\
        - \alpha H & \beta I
    \end{bmatrix}.
\end{equation}
We obtain an error bound by ensuring that the spectral radius $\rho(A)$ of $A$ is less than $1$. In this case we recover the well-known rate for \NAG in the deterministic setting. Let $\Delta_\lambda = (1 + \beta)^2 (1 - \alpha \lambda)^2 - 4 \beta (1 - \alpha \lambda)$ and define
\[
\rho_\lambda(\alpha, \beta) = \begin{cases} \tfrac{1}{2} \abs{(1 + \beta)(1 - \alpha \lambda)} + \tfrac{1}{2} \sqrt{\Delta_\lambda} & \text{ if } \Delta_\lambda \ge 0, \\
\sqrt{\beta(1 - \alpha \lambda)}  & \text{ otherwise.} \end{cases}
\]

\begin{theorem} \label{thm:stochapprox-quadratic}
Let $\rho(\alpha, \beta) = \max \{ \rho_{\mu}(\alpha, \beta), \rho_{L}(\alpha, \beta) \}$. If $\alpha$ and $\beta$ are chosen so that $\rho(\alpha, \beta) < 1$, then for any $\epsilon > 0$, there exists a constant $C_\epsilon$ such that, for all $k$,
\begin{align*}
    \mathbb{E}\left[ \norm{y_{k+1} - x^\star}^2 \right] \le C_\epsilon \bigg(& (\rho(\alpha, \beta) + \epsilon)^{2k} \norm{x_0 - x^\star}^2 \\
    &+ \frac{\alpha^2 ((1 + \beta)^2 + 1)}{1 - \rho(\alpha, \beta)^2} \sigma^2 \bigg) .
\end{align*}
\end{theorem}
Theorem~\ref{thm:stochapprox-quadratic} holds with respect to all norms; the constant $C_\epsilon$ depends on $\epsilon$ and the choice of norm.
Theorem~\ref{thm:stochapprox-quadratic} shows that \NASG converges at a linear rate to a neighborhood of the minimizer of $f$ that is proportional to $\sigma^2$. The proof is given in Appendix~\ref{sec:stochapprox-quadratic-proof} of the supplementary material, and we provide numerical experiments in Section~\ref{sub-sec:stochastic-approximation-numerical} to analyze the tightness of the convergence rate and coefficient multiplying $\sigma^2$ in Theorem~\ref{thm:stochapprox-quadratic}. Comparing to \citet{aybat2019robust}, we recover the same rate, despite taking a different approach, and the coefficient multiplying $\sigma^2$ in Theorem~\ref{thm:stochapprox-quadratic} is smaller.
\begin{corollary} \label{cor:stochapprox-nasg}
Suppose that $\alpha = \nicefrac{1}{L}$ and $\beta = \frac{\sqrt{\cond} - 1}{\sqrt{\cond} + 1}$.
Then for and all $k$,
\begin{align*}
\E[f(\itr{y}{k+1})] - f^\star \leq& \frac{L}{2} \left( \frac{\sqrt{\cond} - 1}{\sqrt{\cond}} + \itr{\epsilon}{k} \right)^{2k} \norm{x_0 - x^\star}^2 \\
&+ C_\epsilon \frac{5 \cond^2 + 2 \cond^{3/2} + Q}{2 L (2 \sqrt{\cond} - 1)(\sqrt{\cond} + 1)^2} \std^2,
\end{align*}
where $\itr{\epsilon}{k} \sim (\sqrt[k]{k} - 1)$.
\end{corollary}

\begin{theorem} \label{thm:stochapprox-sgd}
Let $f$ be $L$-smooth, $\mu$-strongly-convex, and twice continuously-differentiable (not necessarily quadratic). 
Suppose that $\alpha = \nicefrac{2}{\mu + L}$ and $\beta = 0$. Then for all $k$,
\begin{align*}
\E[f(\itr{y}{k+1})] - f^\star \le \frac{L}{2} \left(\frac{Q - 1}{Q + 1}\right)^{2k} \norm{x_0 - x^\star}^2 + \frac{Q \sigma^2}{2 L}.
\end{align*}
\end{theorem}

Corollary~\ref{cor:stochapprox-nasg} confirms that, with the standard choice of parameters, \NASG converges at an accelerated rate to a region of the optimizer. Comparing with Theorem~\ref{thm:stochapprox-sgd}, which is proved in Appendix~\ref{sec:proofs-stochapprox-cors}, we see that in the stochastic approximation setting, with bounded variance, \NASG not only converges at a faster rate than \SGD, the factor multiplying $\sigma^2$ also scales more favorably, $\order{\sqrt{Q} \sigma^2}$ for \NASG vs.~$\order{Q \sigma^2}$ for \SGD.

\begin{figure*}[t]
\centering
\subfloat[.5\textwidth][$(\cond = 2)$: Coefficient multiplying $\sigma^2$]{
\includegraphics[width=.25\textwidth]{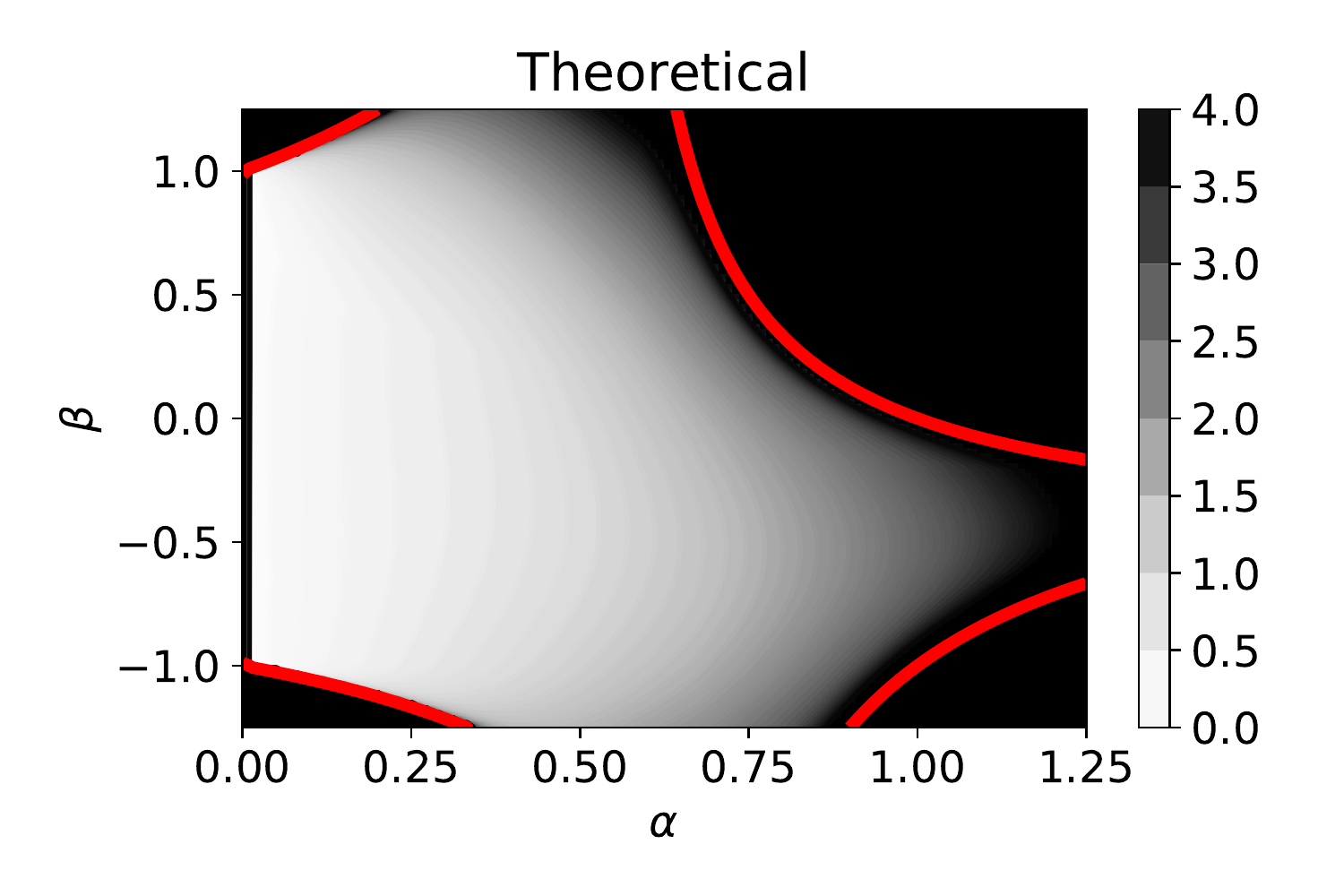}
\includegraphics[width=.25\textwidth]{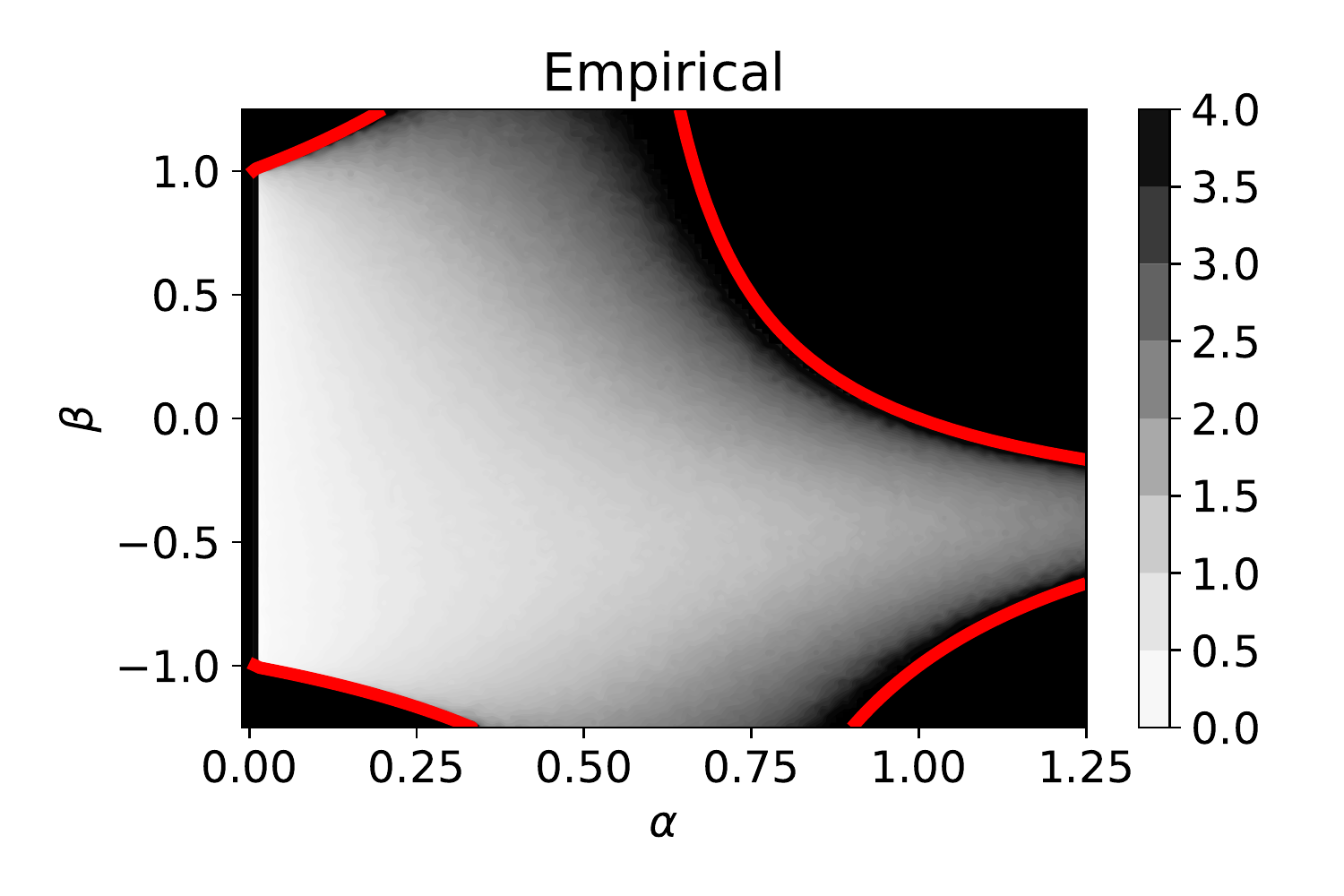}
\label{sub-fig:sa-1}
}
\subfloat[.5\textwidth][$(\cond = 2)$: Convergence rate $\rho(\alpha, \beta)$]{
\includegraphics[width=.25\textwidth]{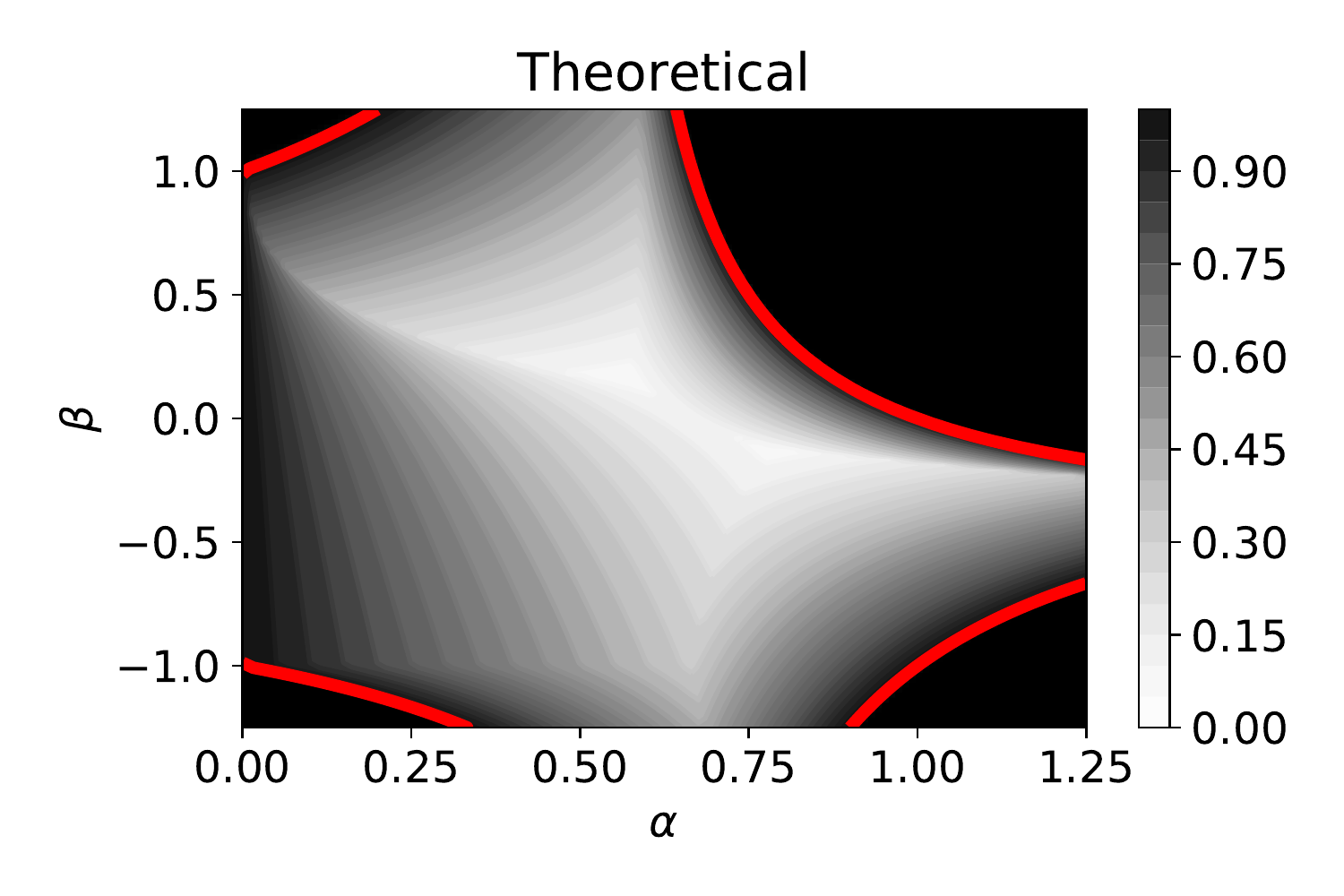}
\includegraphics[width=.25\textwidth]{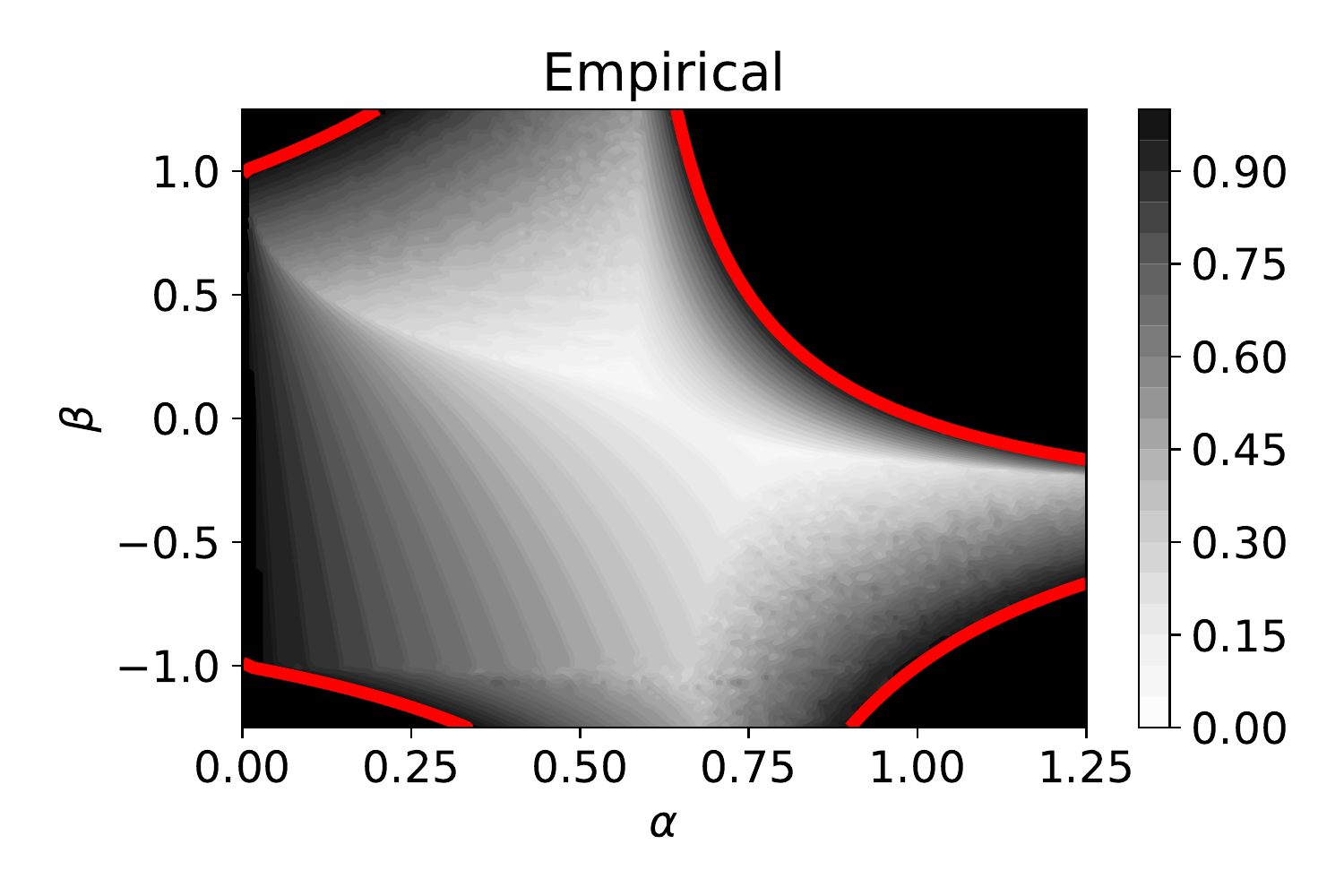}
\label{sub-fig:sa-2}
} \\
\subfloat[.5\textwidth][$(\cond = 8)$: Coefficient multiplying $\sigma^2$]{
\includegraphics[width=.25\textwidth]{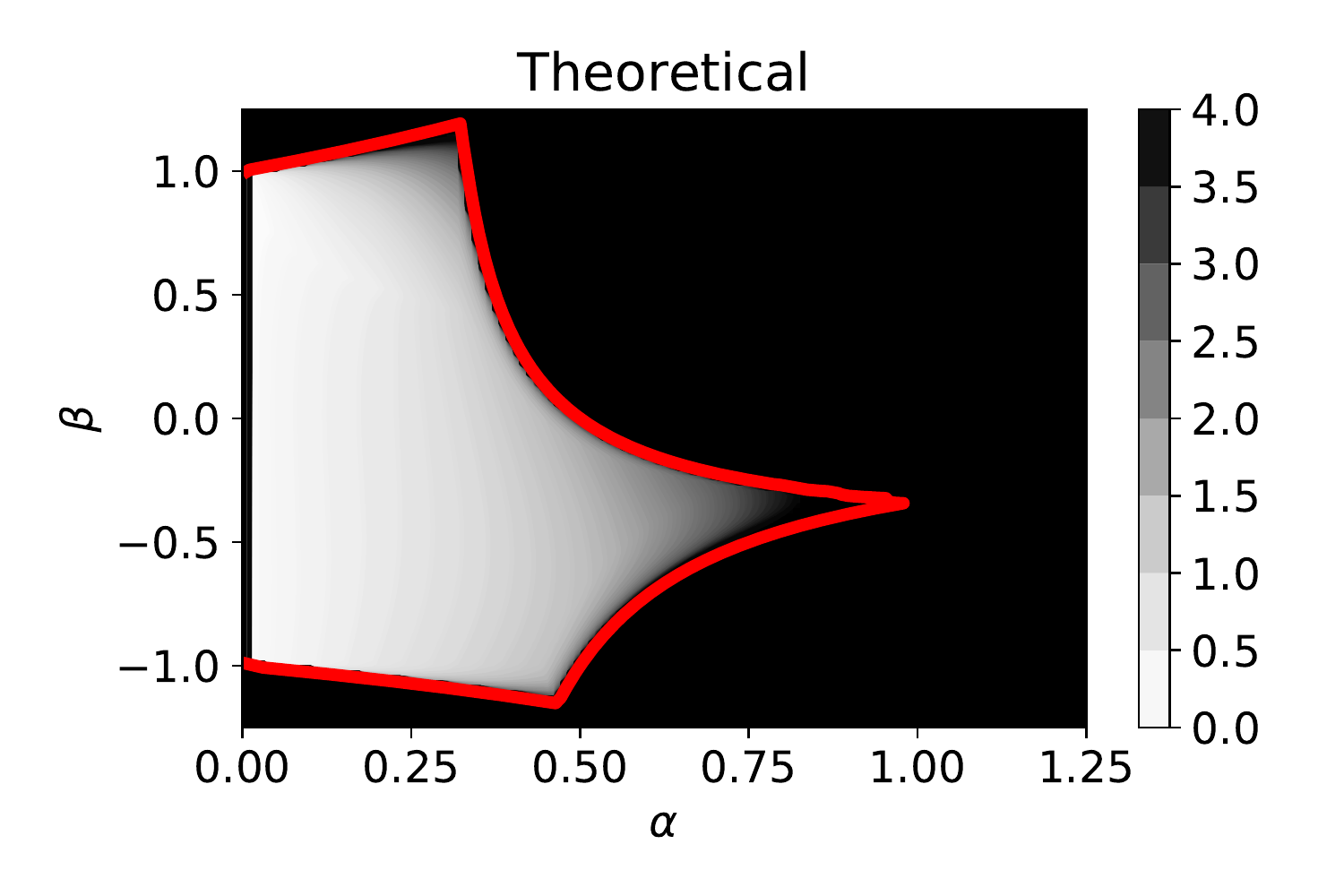}
\includegraphics[width=.25\textwidth]{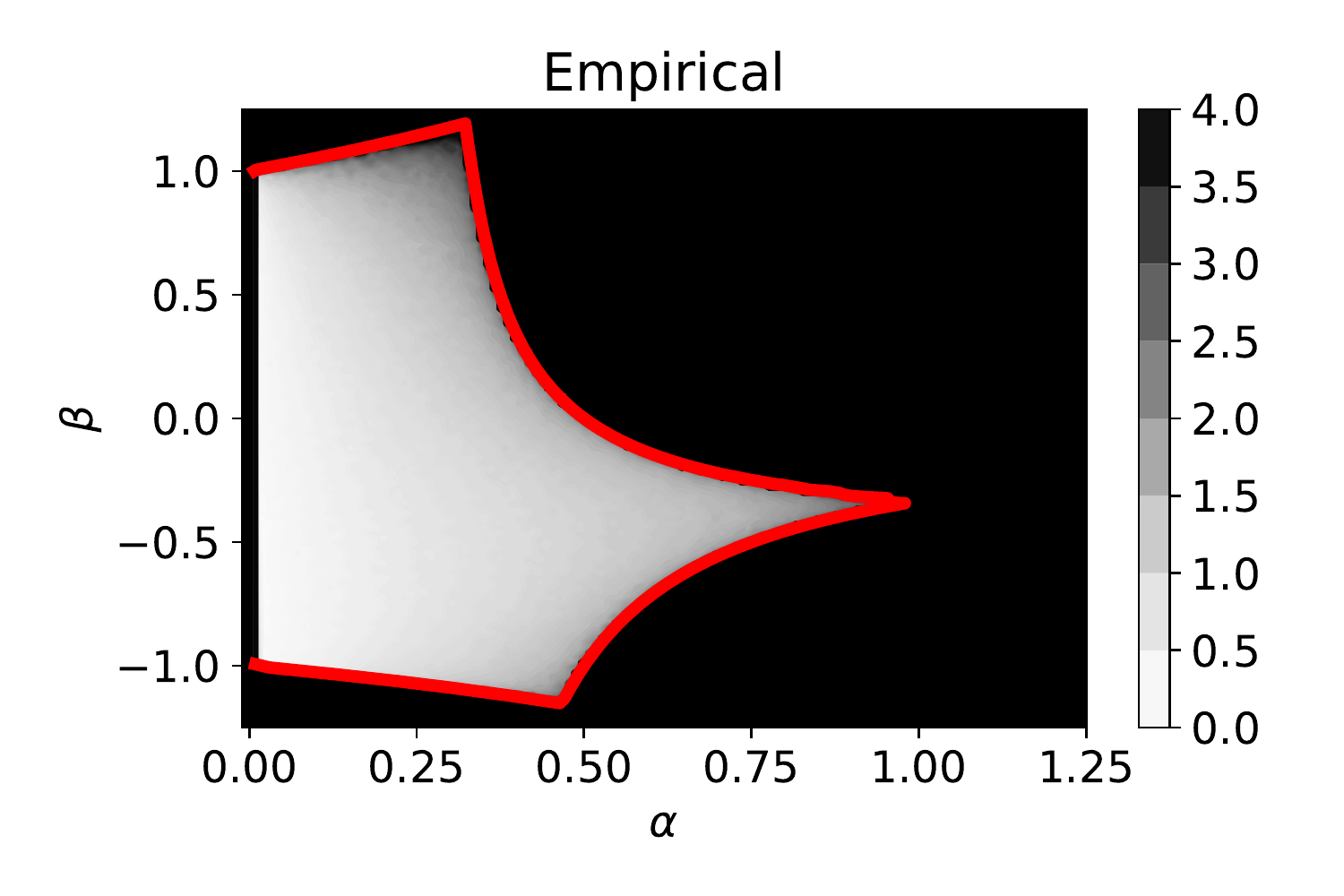}
\label{sub-fig:sa-3}
}
\subfloat[.5\textwidth][$(\cond = 8)$: Convergence rate $\rho(\alpha, \beta)$]{
\includegraphics[width=.25\textwidth]{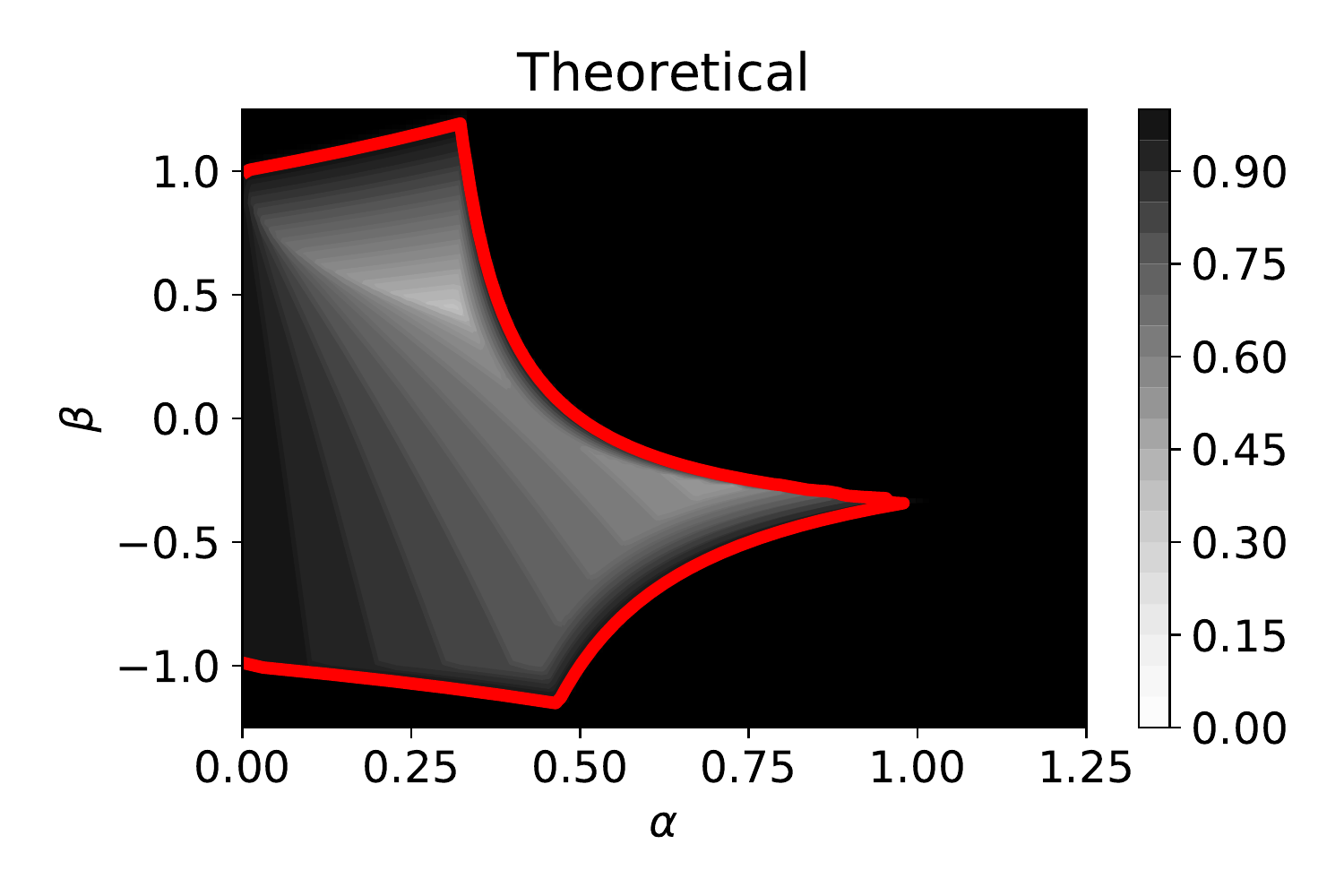}
\includegraphics[width=.25\textwidth]{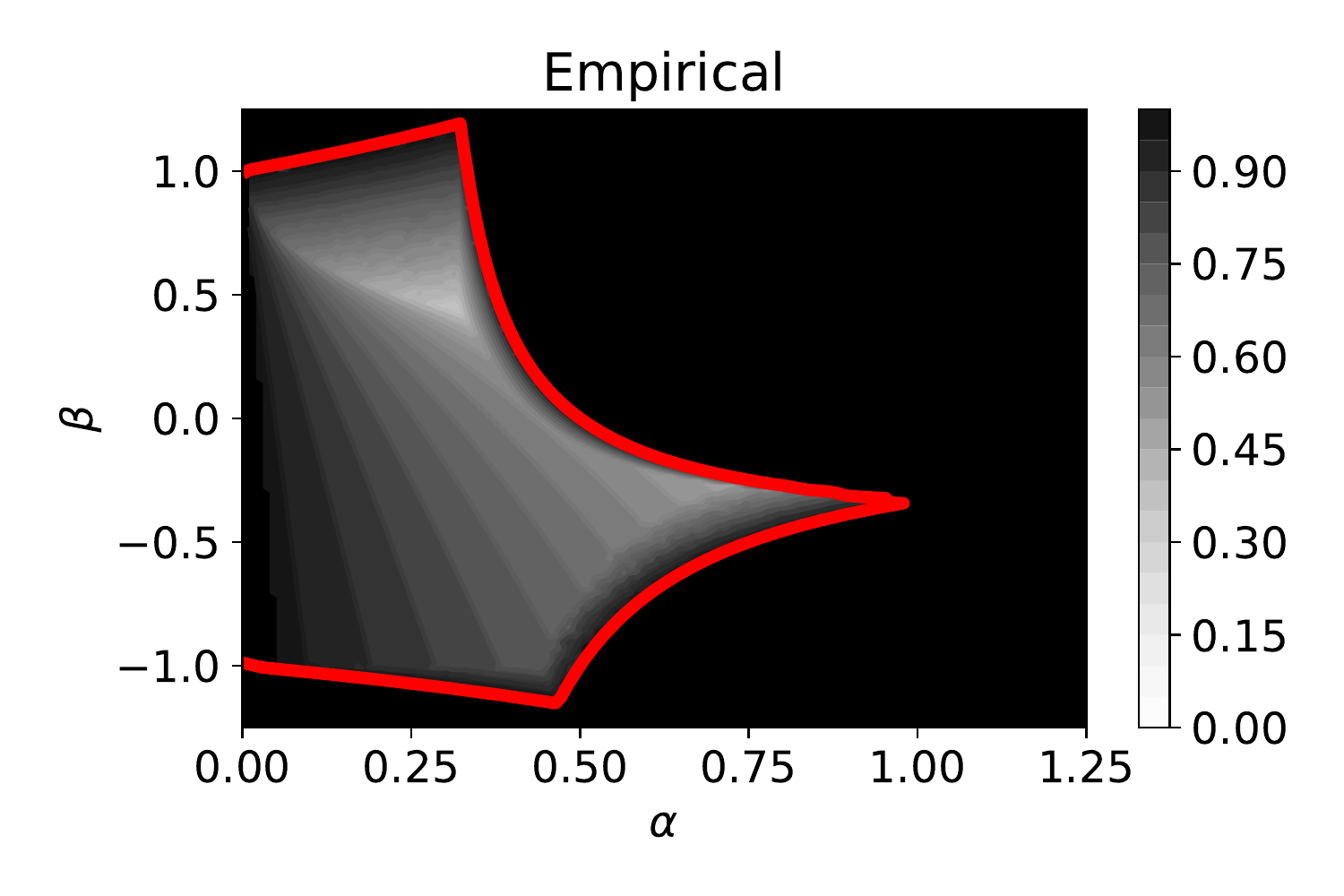}
\label{sub-fig:sa-4}
} \\
\subfloat[.5\textwidth][$(\cond = 32)$: Coefficient multiplying $\sigma^2$]{
\includegraphics[width=.25\textwidth]{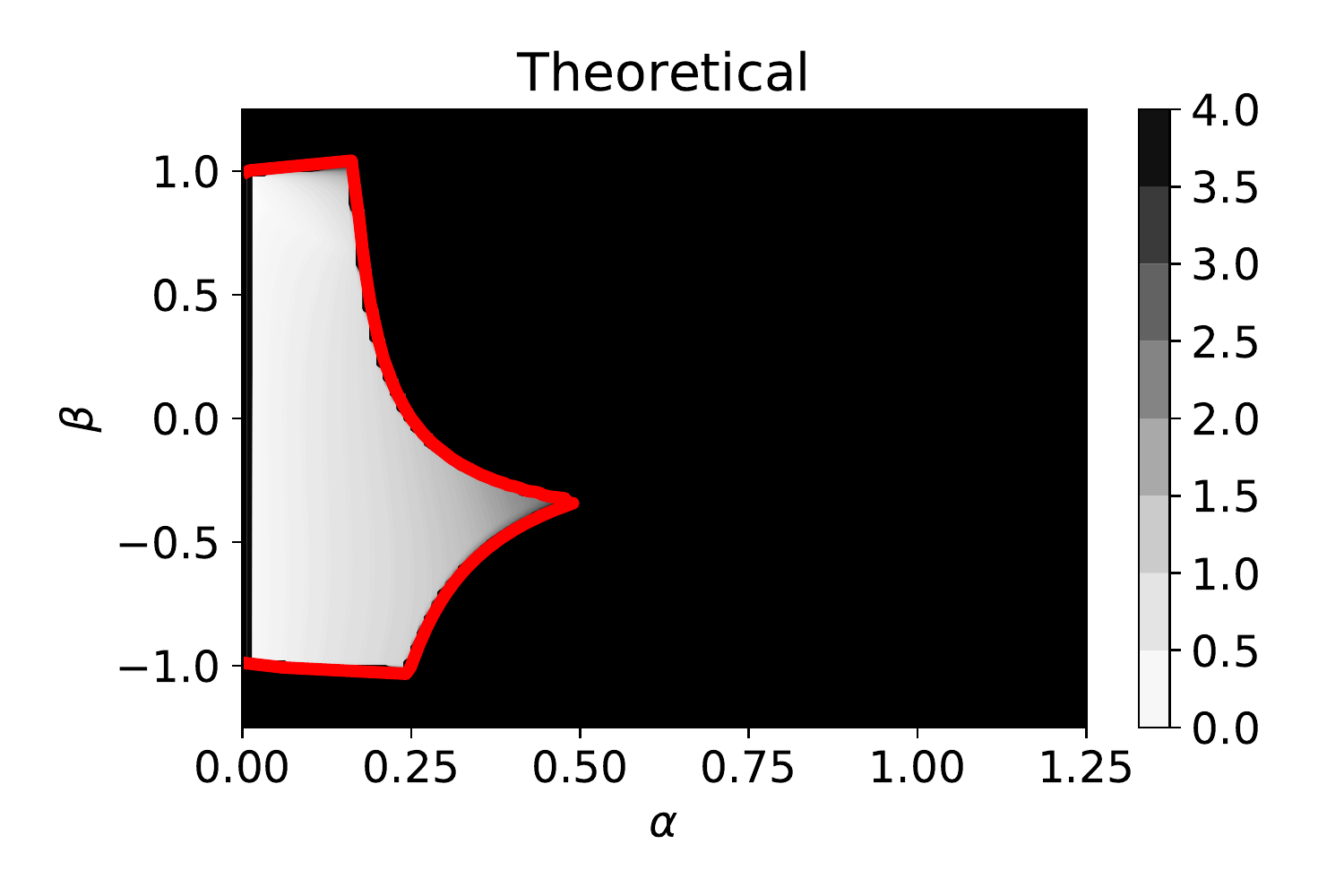}
\includegraphics[width=.25\textwidth]{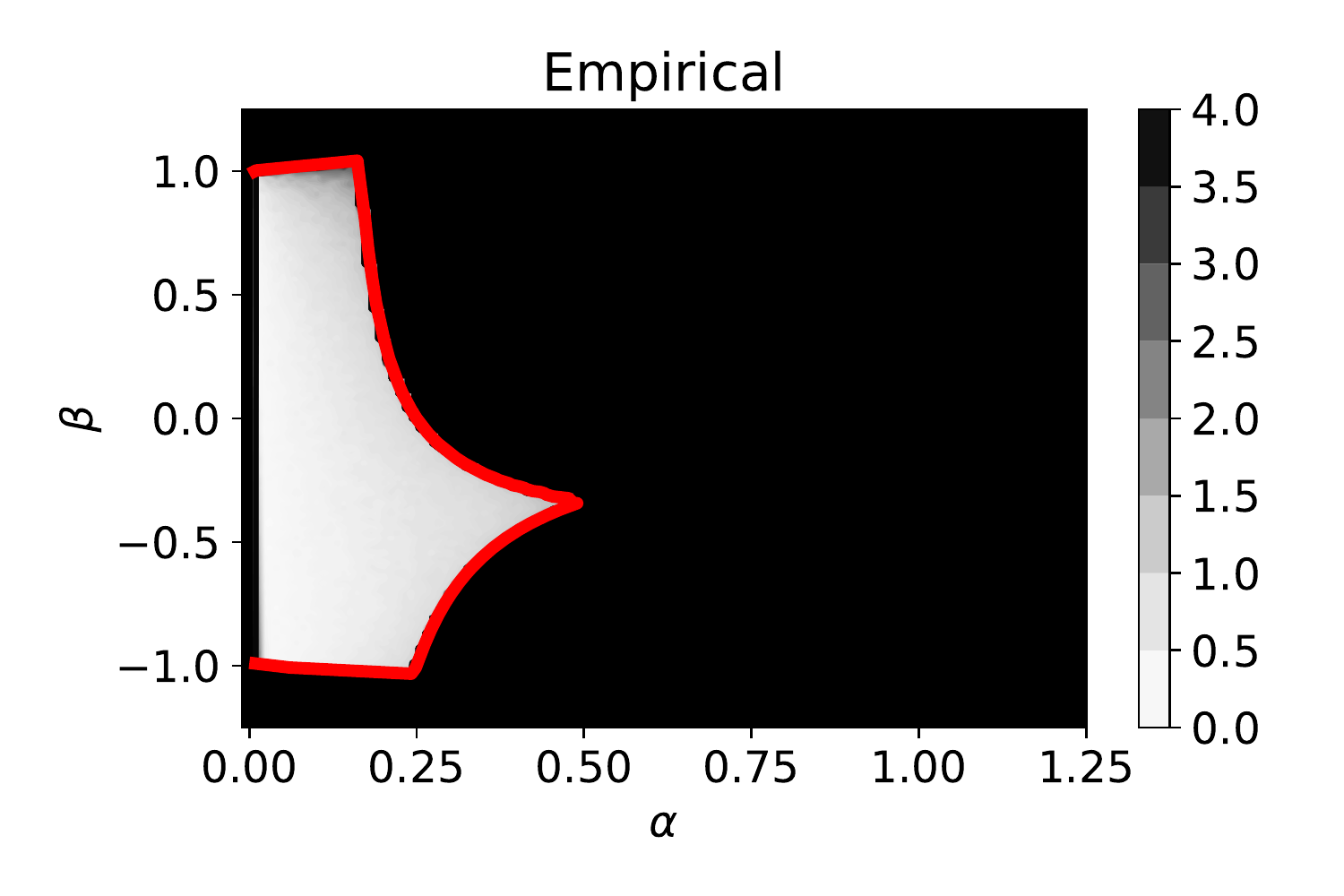}
\label{sub-fig:sa-5}
}
\subfloat[.5\textwidth][$(\cond = 32)$ Convergence rate $\rho(\alpha, \beta)$]{
\includegraphics[width=.25\textwidth]{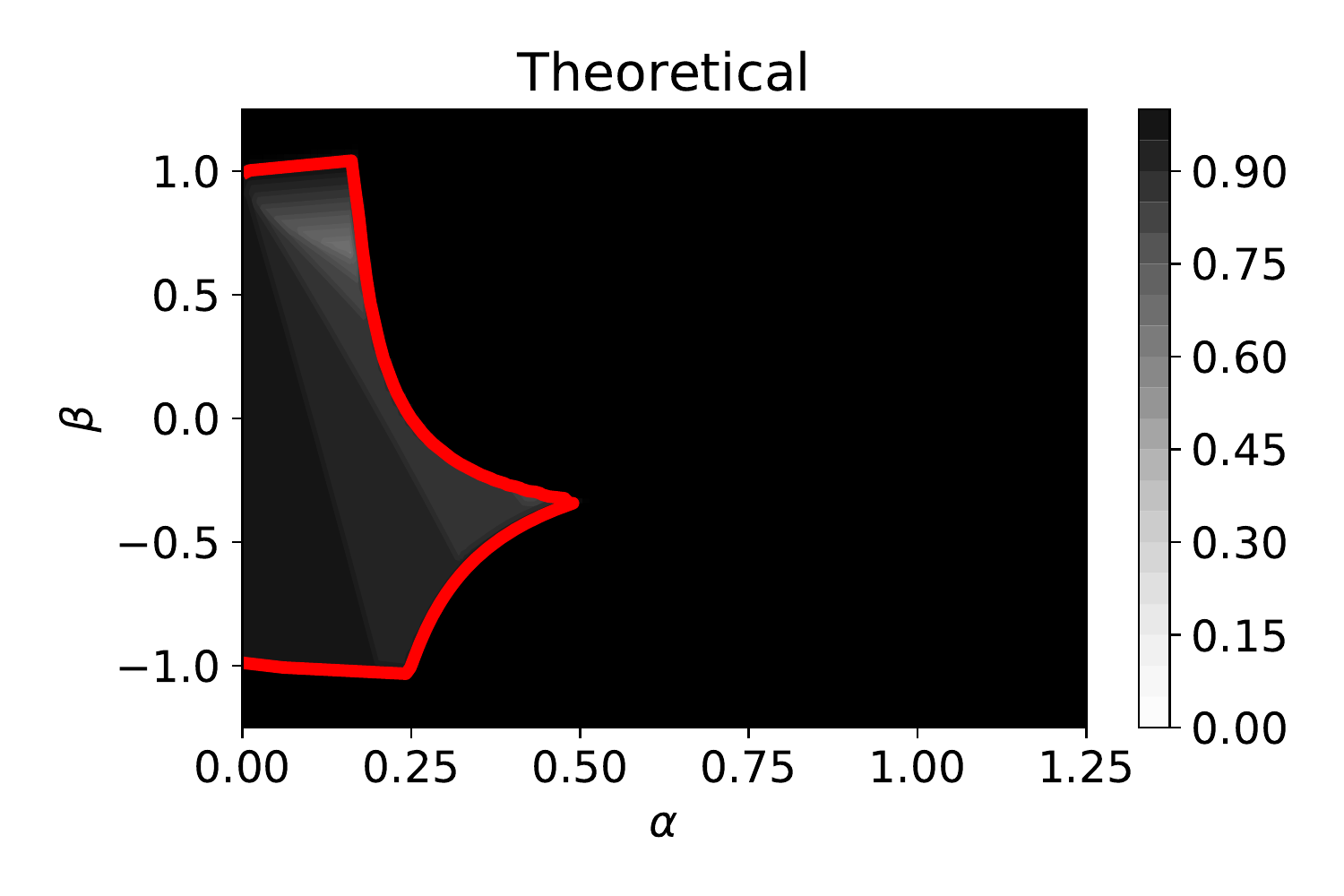}
\includegraphics[width=.25\textwidth]{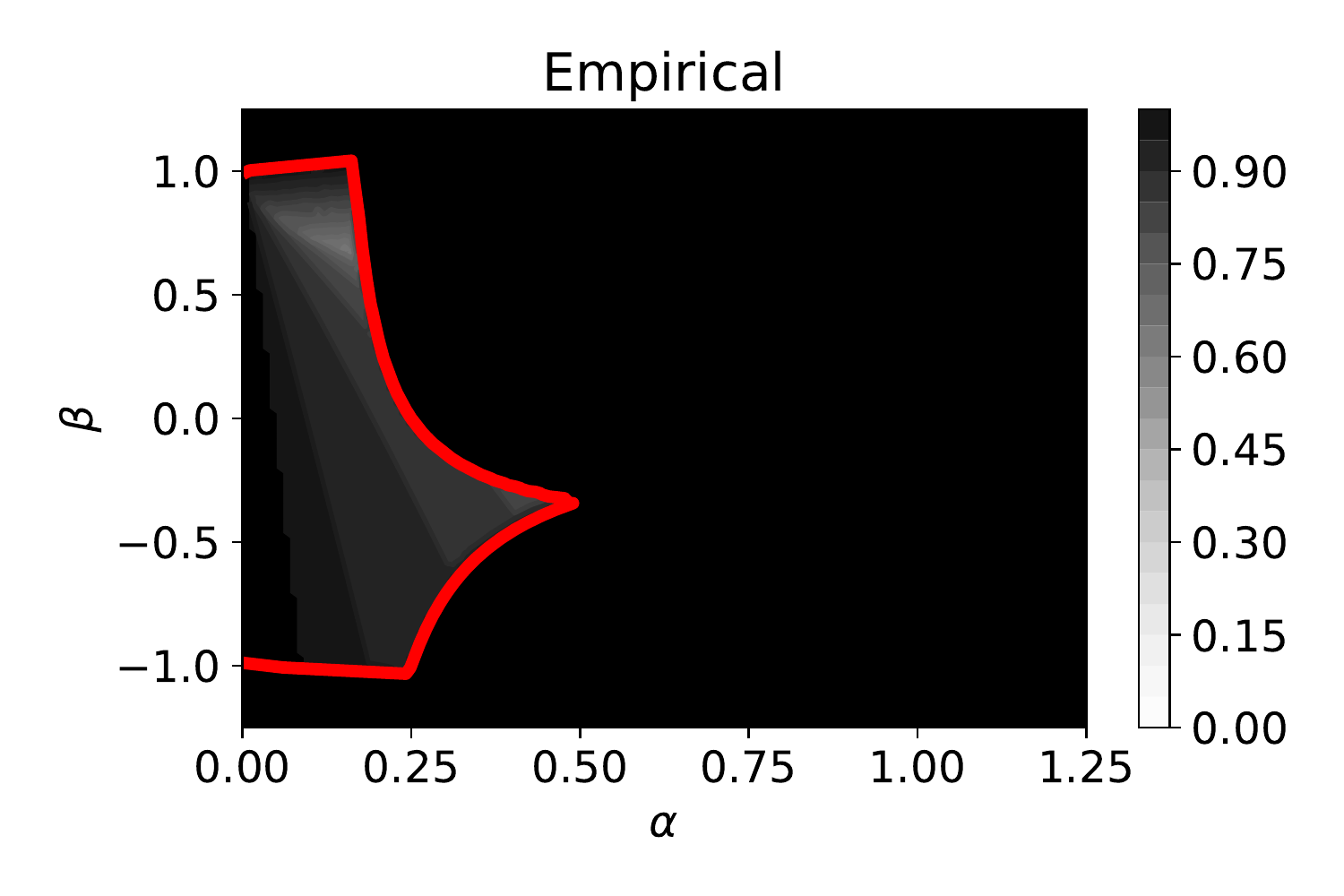}
\label{sub-fig:sa-6}
}
\caption{
Visualizing the accuracy with which the theory predicts the coefficient of the variance term and the convergence rate for different choices of constant step-size and momentum parameters, and various objective condition numbers $\cond$.
Plots labeled ``Theoretical'' depict theoretical results from Theorem~\ref{thm:stochapprox-quadratic}.
Plots labeled ``Empirical'' depict empirical results when using the \NASG method to solve a least-squares regression problem with additive Gaussian noise; each pixel corresponds to an independent run of the \NASG method for a specific choice of constant step-size and momentum parameters.
In all figures, the area enclosed by the red contour depicts the theoretical stability region from Theorem~\ref{thm:stochapprox-quadratic} for which $\rho(\alpha, \beta) < 1$.
Fig.~\ref{sub-fig:sa-1}/\ref{sub-fig:sa-3}/\ref{sub-fig:sa-5}:
Pixel intensities correspond to the coefficient of the variance term in Theorem~\ref{thm:stochapprox-quadratic} ($\lim_{k \to \infty} \frac{1}{\std} \E \norm{\itr{y}{k}-\xstar}_\infty$), which provides a good characterization of the magnitude of the neighbourhood of convergence, even without explicit knowledge of the constant $C_{\epsilon}$.
Brighter regions correspond to smaller coefficients, while darker regions correspond to larger coefficients.
Fig.~\ref{sub-fig:sa-2}/\ref{sub-fig:sa-4}/\ref{sub-fig:sa-6}: Pixel intensities correspond to the theoretical convergence rates in Theorem~\ref{thm:stochapprox-quadratic}, which provides a good characterization of the empirical convergence rates.
Brighter regions correspond to faster rates, and darker regions correspond to slower rates.
The theoretical conditions for convergence in Theorem~\ref{thm:stochapprox-quadratic} depicted by the red-contour are tight.
}
\label{fig:stochastic-approximation}
\end{figure*}

\subsection{Numerical Experiments}
\label{sub-sec:stochastic-approximation-numerical}
In Figure~\ref{fig:stochastic-approximation} we visualize runs of the \NASG method on a least-squares regression problem for different problem condition numbers $\cond$. The objective $f$ corresponds to the worst-case quadratic function used to construct the lower bound \eqref{eq:lower-bound} \cite{nesterov2004introductory}, for dimension $d = 100$.
Stochastic gradients are sampled by adding zero-mean Gaussian noise with variance $\sigma^2 = 0.0025$ to the true gradient.
The left plots in each sub-figure depict theoretical predictions from Theorem~\ref{thm:stochapprox-quadratic}, while the right plots in each sub-figure depict empirical results.
Each pixel corresponds to an independent run of the \NASG method for a specific choice of constant step-size and momentum parameters.
In all figures, the area enclosed by the red contour depicts the theoretical stability region from Theorem~\ref{thm:stochapprox-quadratic} for which $\rho(\alpha, \beta) < 1$.

Figures~\ref{sub-fig:sa-1}/\ref{sub-fig:sa-3}/\ref{sub-fig:sa-5} showcase the coefficient multiplying the variance term, which is taken to be $\frac{\alpha^2((1 + \beta)^2 + 1)}{1 - \rho(\alpha, \beta)^2}$ in theory.
Brighter regions correspond to smaller coefficients, while darker regions correspond to larger coefficients.
All sets of figures (theoretical and empirical) use the same color scale.
We can see that the coefficient of the variance term in Theorem~\ref{thm:stochapprox-quadratic} provides a good characterization of the magnitude of the neighbourhood of convergence.
The constant $C_\epsilon$ is approximated as $1 + (1-\rho(\alpha, \beta)^2)(\norm{A}^2 - \rho(\alpha, \beta)^2)$, where $\norm{A}$ denotes the largest singular value of $A$ in~\eqref{eq:stochapprox-quadratic-A}, and $\rho(\alpha,\beta)$ is the largest eigenvalue of $A$.
More detail on this simple approximation is provided in Appendix~\ref{sub-sec:estimating-cepsilon} of the supplementary material.

Figures.~\ref{sub-fig:sa-2}/\ref{sub-fig:sa-4}/\ref{sub-fig:sa-6} showcase the linear convergence rate in theory and in practice.
Brighter regions correspond to faster rates, and darker regions correspond to slower rates.
Again, all figures (theoretical and empirical) use the same color scale.
We can see that the theoretical linear convergence rates in Theorem~\ref{thm:stochapprox-quadratic} provide a good characterization of the empirical convergence rates.
Moreover, the theoretical conditions for convergence in Theorem~\ref{thm:stochapprox-quadratic} depicted by the red-contour appear to be tight.

In short, the theory developed in this section appears to provide an accurate characterization of the \NASG method in the stochastic-approximation setting. As we will see in the subsequent section, this theoretical characterization does not reflect its behavior in the finite-sum setting, which is typically closer to practical machine-learning setups, where randomness is due to mini-batching.

\section{The Finite-Sum Setting}
\label{sec:finite-sum}

Now consider the finite-sum setting, with
\begin{equation}
    \label{eq:finite-sum}
    f(x) = \frac{1}{n} \sum^n_{i=1} f_i(x),
\end{equation}
where each function $f_i$ is $\mu$-strongly convex, $L$-smooth, and twice continuously differentiable.
In this setting, stochastic gradients $\itr{g}{k}$ are obtained by sampling a subset of terms.
This can be seen as approximating the gradient $\nabla f(\itr{y}{k})$ with a mini-batch gradient
\begin{equation}
    \label{eq:finite-sum-gradient}
    \itr{g}{k} = \sum_{i=1}^n \nu_{k,i} \nabla f_{i}(\itr{y}{k}),
\end{equation}
where $\nu_{k} \in \R^n$ is a sampling vector with components $\nu_{k,i}$ satisfying $\E[\nu_{k,i}] = \frac{1}{n}$ \cite{gower2019sgd}. To simplify the discussion, let us assume that the mini-batch sampled at every iteration $k$ has the same size, and all elements are given the same weight, so $\sum_{i=1}^n \nu_{k,i} = 1$, those indices $i$ which are sampled have $\nu_{k,i} = \frac{1}{m}$ where $m$ is the mini-batch size ($1 \le m \le n$), and $\nu_{k,i} = 0$ for all other indices.

\subsection{An Impossibility Result}

Next we show that even when each function $f_i$ is well-behaved, the \NASG method may diverge when using the standard choice of step-size and momentum.
Instability of Nesterov's method for convex (but not strongly convex) functions with unbounded eigenvalues is shown in~\citet{liu2020accelerating}.
This section employs a different proof technique to strengthen this result to the case where each function $f_i$ is $\mu$-strongly-convex and $L$-smooth (all eigenvalues bounded between $\mu$ and $L$).

Let us assume that we do not see the same mini-batch twice consecutively; \ie,
\begin{equation} \label{eq:shuffle}
    \Pr(\norm{\nu_{k+1} - \nu_k} > 0) = 1 \quad \text{ for all } k.
\end{equation}
It is typical in practice to perform training in epochs over the data set, and to randomly permute the data set at the beginning of each epoch, so it is unlikely to see the same mini-batch twice in a row. Note we have not assumed that the sample vectors $\nu_k$ are independent. We do assume that $\E_k[\nu_{k,i}] = \frac{1}{n}$, where $\E_k$ denotes expectation with respect to the marginal distribution of $\nu_k$.

The \emph{interpolation condition} is said to hold if the minimizer $x^\star$ of $f$ also minimizes each $f_i$; \ie, if $\nabla f_i(x^\star) = 0$ for all $i=1,\dots,n$. It has been observed in some settings that stronger convergence guarantees can also be obtained when interpolation or a related assumption holds; \eg, \cite{schmidt2013fast,loizou2017momentum,ma2018power,vaswani2019fast}.

\begin{theorem} \label{thm:counter-example}
Suppose we run the \NASG method \eqref{eq:nasg-y}--\eqref{eq:nasg-x} in a finite-sum setting where $n \ge 3$ and the sampling vectors $\nu_k$ satisfy the condition \eqref{eq:shuffle}. For any initial point $x_0 \in \R^d$, there exist $L$-smooth, $\mu$-strongly convex quadratic functions $f_1, \dots, f_n$ such that $f$ is also $L$-smooth and $\mu$-strongly convex, and if we run the \NASG method with $\alpha = 1/L$ and $\beta = \frac{\sqrt{\cond} - 1}{\sqrt{\cond} + 1}$, then
\[
\lim_{k \rightarrow \infty} \E[\norm{y_k - x^\star}] = \infty.
\]
This is true even if the functions $f_1,\dots, f_n$ are required to satisfy the interpolation condition.
\end{theorem}

\begin{proof}
We will prove this claim constructively. Given the initial vector $x_0$, choose $x^\star \in \R^d$ to be any vector $x^\star \ne x_0$. 

Let $U$ be an orthogonal matrix. Let the Hessian matrices $H_i$, $i=1,\dots,n$, be chosen so that they are all diagonalized by $U$, and let $\Lambda_i$ denote the diagonal matrix of eigenvalues of $H_i$; \ie, $H_i = U \Lambda_i U^\T$. Denote by $\Lambda_{\nu_k}$ the matrix
\begin{equation}
    \Lambda_{\nu_k} = \sum_{i=1}^n \nu_{k,i} \Lambda_i.
\end{equation}
It follows that $\Lambda_{\nu_k} \in \R^{d \times d}$ is also diagonal, and all of its diagonal entries are in $[\mu, L]$. 

Recall that we have assumed that the functions $f_i$ are quadratic: $f_i(x) = \frac{1}{2} x^\T H_i x - b_i^\T x + c_i$. Let us assume that $b_i \in \R^d$ and $c_i \in R$ are chosen so that all functions $f_i$ are minimized at the same point $x^\star$, satisfying the interpolation condition. Then from \eqref{eq:taylor}, we have
\begin{equation} \label{eq:finitesum-impossibility-g-interpolation}
g_k = U \Lambda_{\nu_k} U^\T r_k.
\end{equation}
Using this in \eqref{eq:recursion-with-g} and unrolling, we obtain that
\begin{equation} \label{eq:finitesum-recursion}
\begin{bmatrix}
    r_{k+1} \\ v_k
\end{bmatrix}
= A_k A_{k-1} \dots A_1 \begin{bmatrix}
    r_{1} \\ v_{0}
\end{bmatrix},
\end{equation}
where 
\begin{equation}
    A_j = \begin{bmatrix}
        I - \alpha (1 + \beta) U \Lambda_{\nu_j} U^\T & \beta^2 I \\
        - \alpha U \Lambda_{\nu_j} U^\T & \beta I
    \end{bmatrix}.
\end{equation}
For fixed $n$ and $m$, there are a finite number of sampling vectors $\nu_k$ (precisely $\binom{n}{m}$), and therefore the matrices $A_j$ belong to a bounded set $\mathcal{A}$. It follows that the trajectory $([r_{k+1}, v_k]^\T)_{k\ge0}$ is stable if the joint spectral radius of the set of matrices $\mathcal{A}$ is less than one \cite{rota1960note}. Conversely, if $\E[\rho(A_k \cdots A_1)^{1/k}] > 1$ for all $k$ sufficiently large, then $\lim_{k\rightarrow \infty} \norm{y_k - x^\star} = \infty$.

Based on the construction above, the norm of the matrix product $A_k \dots A_1$ in \eqref{eq:finitesum-recursion} can be characterized by studying products of smaller $2 \times 2$ matrices of the form
\begin{equation} \label{eq:finitesum-B}
    B(\lambda_{k,j}) = \begin{bmatrix}
        1 - \alpha (1 + \beta) \lambda_{k,j} & \beta^2 \\
        - \alpha \lambda_{k,j} & \beta
    \end{bmatrix},
\end{equation}
where $\lambda_{k,j}$ is a diagonal entry of $\Lambda_{v_k}$. To see this, observe that there is a permutation matrix $P \in \{0,1\}^{2d \times 2d}$ such that (see Appendix~\ref{sec:permutation})
\begin{align*}
P 
\begin{bmatrix}
    U^\T & 0 \\
    0 & U^\T
\end{bmatrix} &
A_j
\begin{bmatrix}
    U & 0 \\
    0 & U
\end{bmatrix}
P^\T \\
&= 
\begin{bmatrix}
    B(\lambda_{j,1}) & 0 & \cdots & 0 \\
    0 & B(\lambda_{j,2}) & \cdots & 0 \\
    \vdots & \vdots & \ddots & \vdots \\
    0 & 0 & \cdots & B(\lambda_{j,d})
\end{bmatrix},
\end{align*}
where $\lambda_{j,i}$ is the $i$th diagonal entry of $\Lambda_{v_j}$.
\begin{figure*}[!t]
\centering
\subfloat[.33\textwidth][$n=50$]{
    \includegraphics[width=.33\textwidth]{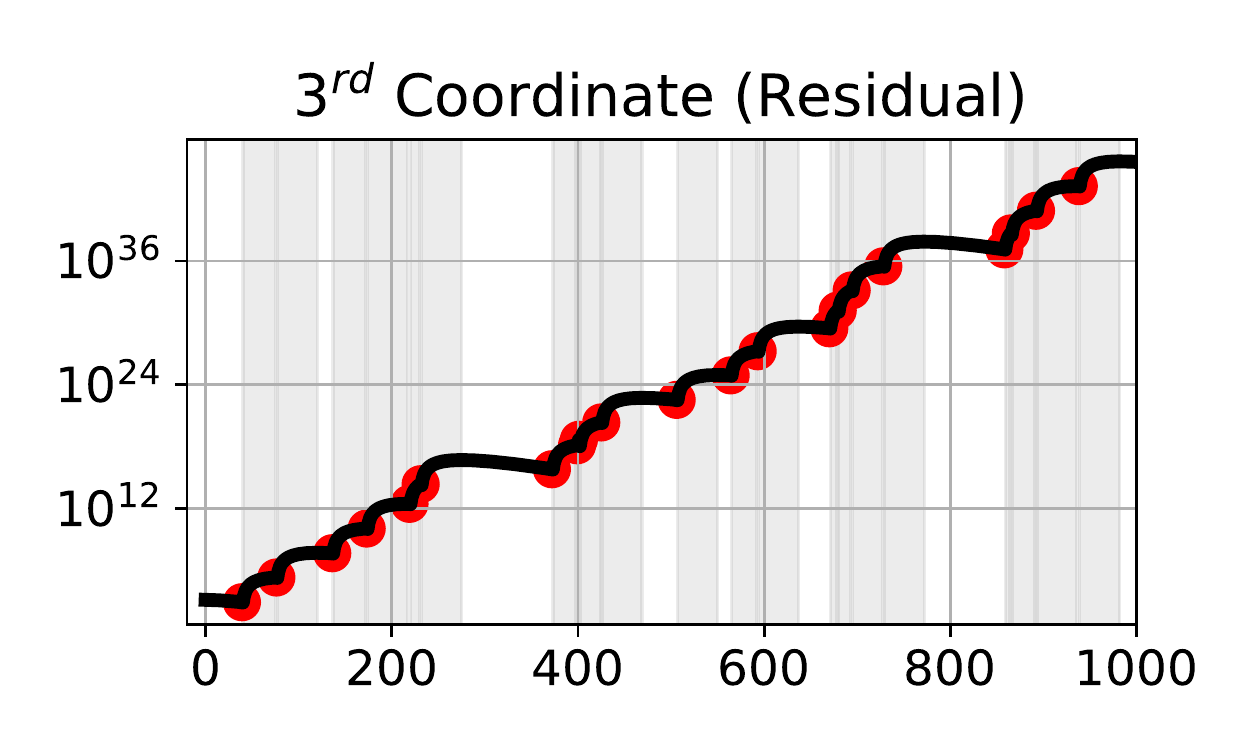}
    \label{sub-fig:finite-sum-1}}
\subfloat[.33\textwidth][$n=250$]{
    \includegraphics[width=.33\textwidth]{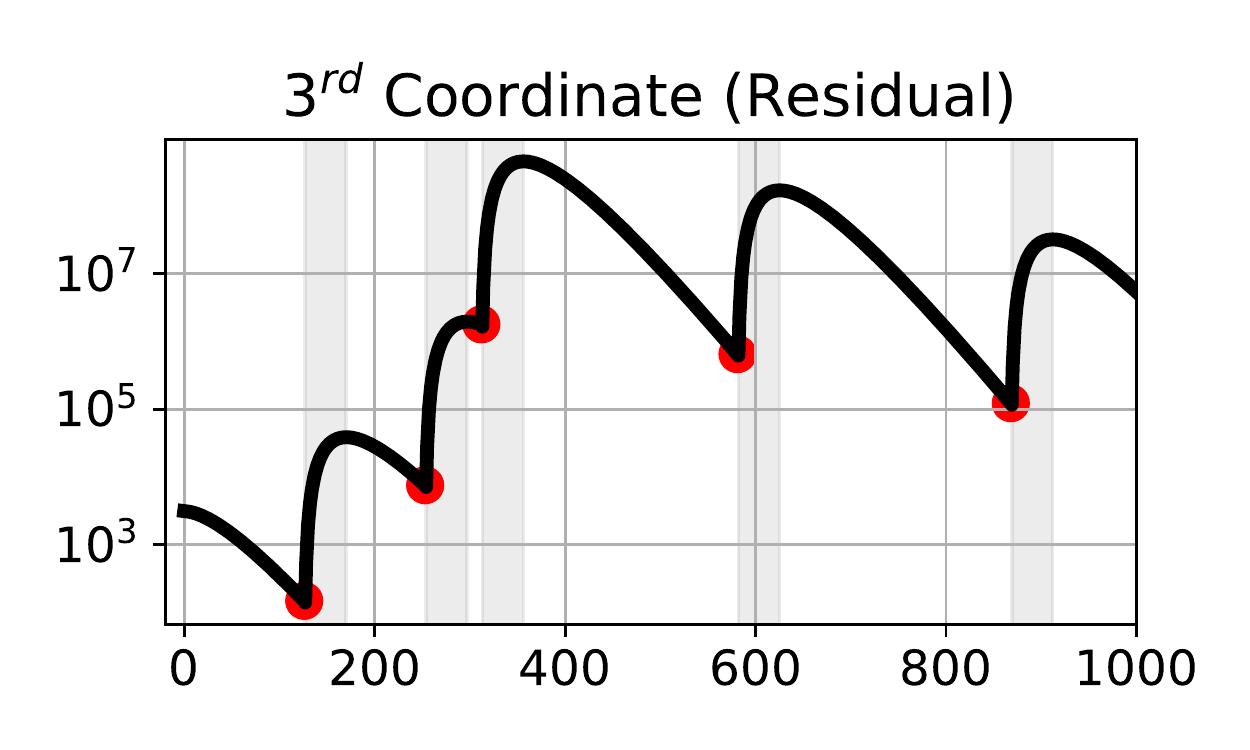}
    \label{sub-fig:finite-sum-2}}
\subfloat[.33\textwidth][$n=1000$]{
    \includegraphics[width=.33\textwidth]{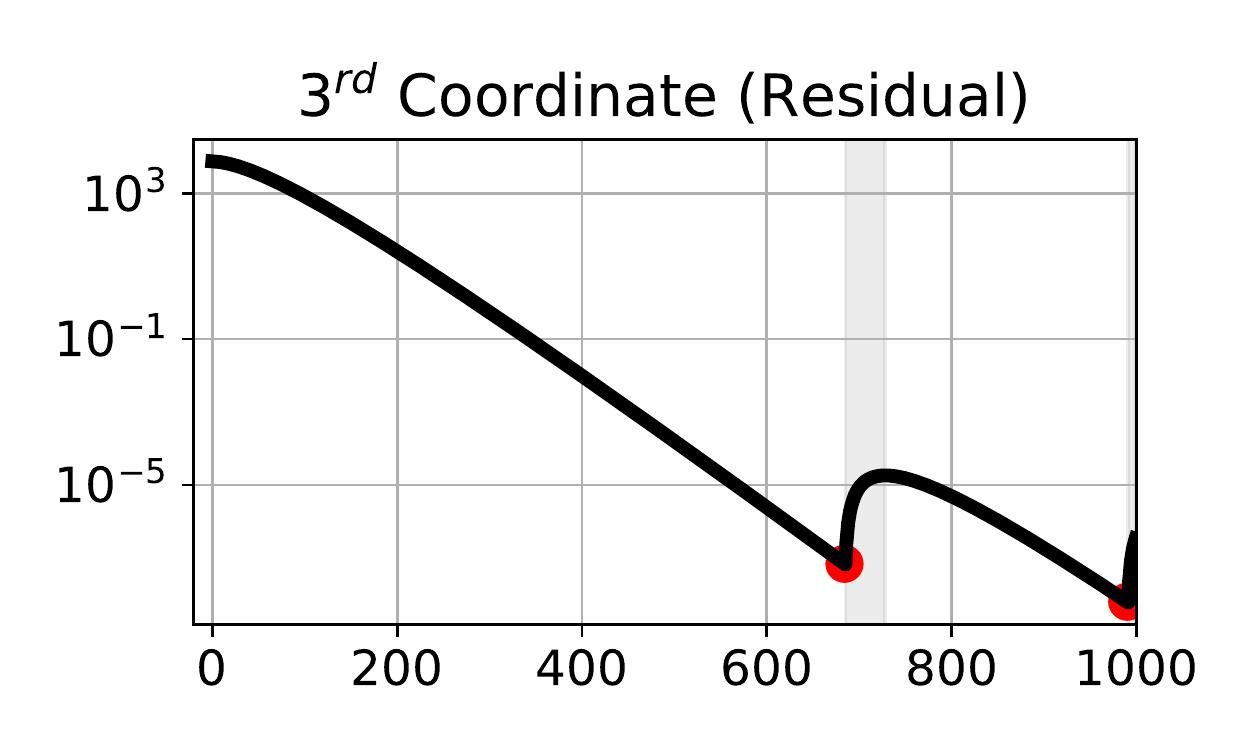}
    \label{sub-fig:finite-sum-3}}
\caption{
Visualizing the convergence of Nesterov's \NASG method ${\itr{y}{k} - \xstar}$ in $\R^3$ along a single coordinate direction in the $L$-smooth $\mu$-strongly-convex finite-sum setting with the usual choice of parameters ($\alpha=\nicefrac{1}{L}$ and $\beta=\nicefrac{(\sqrt{\cond}-1)}{(\sqrt{\cond}+1)}$).
The $L$-smoothness parameter is $100$ and the modulus of strong-convexity $\mu$ is $0.05$.
There are $n$ functions $f_1, \ldots, f_n$ in the finite-sum, each with the same minimizer $\xstar$.
All the functions have the eigenvalue $L$ along the first coordinate basis vector, and the eigenvalue $\mu$ along the second coordinate basis vector.
Along the third coordinate basis vector, the functions $f_1, \ldots, f_{n-1}$ have eigenvalue $\mu$, while only a single function, $f_n$, has eigenvalue $L$.
At each iteration, the \NASG method obtains a stochastic gradient by sampling one function from the finite-sum.
Red points indicate iterations at which the mini-batch corresponding to the function $f_n$ was sampled.
Gray shading indicates iterations at which the momentum and gradient vector point in opposite directions along the given coordinate axis. The inconsistent mini-batch leads to the divergence of the \NASG method, but becomes less destabilizing as the number of terms in the finite-sum $n$ grows.
}
\label{fig:finite-sum}
\end{figure*}

Furthermore, since all matrices $H_i$ have the same eigenvectors, we have that
\begin{align*}
P 
\begin{bmatrix}
    U^\T & 0 \\
    0 & U^\T
\end{bmatrix} &
A_k A_{k-1} \cdots A_1
\begin{bmatrix}
    U & 0 \\
    0 & U
\end{bmatrix}
P^\T \\
&= 
\begin{bmatrix}
    T_{k,1} & 0 & \cdots & 0 \\
    0 & T_{k,2} & \cdots & 0 \\
    \vdots & \vdots & \ddots & \vdots \\
    0 & 0 & \cdots & T_{k,d}
\end{bmatrix},
\end{align*}
where $T_{k,j} = B(\lambda_{k,j}) \cdots B(\lambda_{1,j})$. Hence, the spectral radius of the product $A_k \cdots A_1$ corresponds to the maximum spectral radius of any of the $2 \times 2$ matrices $T_{k,j}$, $j=1,\dots,d$. 

Let $j$ index a subspace such that $u_j^\T r_1 \ne 0$, where $u_j$ is the $j$th column of $U$. To simplify the discussion, suppose that all mini-batches are of size $m=1$, and assume $n > 1$. Since we can define the Hessians of the functions $f_i$ such that the eigenvalues pair together arbitrarily, consider matrix products of the form
\begin{equation}
    T_{k,j} = B(L) B(\mu)^{k_1} B(L) B(\mu)^{k_2} \cdots B(L) B(\mu)^{k_s},
\end{equation}
where $k = k_1 + \dots + k_s + s$. That is, all but one of the functions $f_i$ have the eigenvalue $\mu$ in this subspace, and the remaining one has eigenvalue $L$ in this subspace. Hence, most of the time we sample mini-batches corresponding to $B(\mu)$, and once in a while we sample mini-batches with $B(L)$. Moreover, since we do not sample the same mini-batch twice consecutively, we never see back-to-back $B(L)$'s. For this case, and with the standard choice of step-size and momentum parameters, we can precisely characterize the spectral radius of $T_{k,j}$.

\begin{lemma} \label{lem:spectral-radius-T_j}
If $\alpha = 1/L$ and $\beta = \frac{\sqrt{\cond}-1}{\sqrt{\cond}+1}$, then 
\[
\rho(T_j) = \left( \frac{\sqrt{\cond} - 1}{\sqrt{\cond}}\right)^k \times k_1 \cdots k_s.
\]
\end{lemma}

The proof of Lemma~\ref{lem:spectral-radius-T_j} is given in Appendix~\ref{sec:spectral-radius-T_j-proof}. Since we do not sample the same mini-batch twice in a row, it follows that $k_j \ge 1$ for all $j=1,\dots,s$.
Based on the assumption that $\E_k[\nu_{k,i}] = \frac{1}{n}$, we have $\E[s] = \frac{k}{n}$.
Moreover, since $\E_k[\nu_{k,i}] = \frac{1}{n}$, for large $k$ ($\gg n$) we have $\E[(k_1 \dots k_s)^{\frac{1}{k}}] \approx (n-1)^{1/n}$. Thus, for sufficiently large $Q$ and sufficiently large $k$,
\[
\E[\rho(A_k \cdots A_1)^{1/k}] > 1.
\]
Therefore, $\lim_{k \rightarrow \infty} \E \norm{y_k - x^\star} = \infty$.

Recall that we assumed the interpolation condition holds in order to get $g_k$ of the form \eqref{eq:finitesum-impossibility-g-interpolation}. If we relax this and do not require interpolation, then $g_k$ will have an additional term involving $\nabla f_i(x^\star)$, and the expression \eqref{eq:finitesum-recursion} will also have an additional terms, akin to the $\zeta_k$ terms in \eqref{eq:stochapprox-unrolled}. The same arguments still apply, leading to the same conclusion.
\end{proof}

\subsection{Example}
The divergence result in Theorem~\ref{thm:counter-example} stems from the fact that the algorithm acquires momentum along a low-curvature direction, and then, suddenly, a high-curvature mini-batch is sampled that overshoots along the current trajectory.
Momentum prevents the iterates from immediately adapting to the overshoot, and propels the iterates away from the minimizer for several consecutive iterations.

To illustrate this effect, consider the following example finite-sum problem with $d=3$, where each function $f_i$ is a strongly-convex quadratic with gradient
\[
    \nabla f_i(x) = U \Lambda_i U^T (x - \xstar).
\]
For simplicity, take $U = I$, and let
\[
    \Lambda_i =
    \begin{bmatrix}
    L & 0 & 0 \\
    0 & \mu & 0 \\
    0 & 0 & \lambda_i
    \end{bmatrix}.
\]
The scalar $\lambda_i$ is equal to $\mu$ for all $i \neq n$, and is equal to $L$ for $i = n$.
Therefore, each function $f_i$ is $\mu$-strongly convex, $L$-smooth, and minimized at $\xstar$, and the global objective $f$ is also $\mu$-strongly convex, $L$-smooth, and minimized at $\xstar$.
Moreover, the functions $f_i$ are nearly all identical, except for $f_n$, which we refer to as the inconsistent mini-batch.

From the proof of Theorem~\ref{thm:counter-example}, the growth rate of the iterates along the third coordinate direction, with the usual choice of parameters ($\alpha=\nicefrac{1}{L}$, $\beta=\frac{\sqrt{\cond}-1}{\sqrt{\cond}+1}$), is
\begin{align*}
    \E[\radius{\itr{A}{k}\cdots\itr{A}{1}}^{\nicefrac{1}{k}}]
    \sim 
    \left(\frac{\sqrt{\cond} - 1}{\sqrt{\cond}} \right) (n-1)^{\frac{1}{n}}.
\end{align*}
Notice that the term $(n-1)^{\frac{1}{n}}$ goes to $1$ as $n$ grows to infinity.
Hence, for a fixed condition number $\cond$, the \NASG method exhibits an increased probability of convergence as $n$ becomes large.
The intuition for this is that we sample the inconsistent mini-batch less frequently, and thereby decrease the likelihood of derailing convergence.

Figure~\ref{fig:finite-sum} illustrates the convergence of the \NASG method in this setting with the usual choice of parameters ($\alpha=\nicefrac{1}{L}$, $\beta=\frac{\sqrt{\cond}-1}{\sqrt{\cond}+1}$), for various $n$ (number of terms in the finite-sum).
At each iteration, the \NASG method obtains a stochastic gradient by sampling a mini-batch from the finite-sum.
Components of iterates along the first coordinate direction converge in a finite number of steps, and components of iterates along the second coordinate direction converge at Nesterov's rate $\nicefrac{(\sqrt{\cond}-1)}{\sqrt{\cond}}$.
Meanwhile, components of iterates along the third coordinate direction diverge.

Annotated red points indicate iterations at which the mini-batch corresponding to the function $f_n$ was sampled. The shaded windows illustrate that immediately after the inconsistent mini-batch is sampled, the gradient and momentum buffer have opposite signs for several consecutive iterations.

\subsection{Convergent Parameters}

Next we turn our attention to finding alternative settings for the parameters $\alpha$ and $\beta$ in the \NASG method which guarantee convergence in the finite-sum setting. \citet{vaswani2019fast} obtain linear convergence under a strong growth condition using an alternative formulation of \NASG which has multiple momentum parameters by keeping the step-size constant and having the momentum parameters vary. Here we focus on constant step-size and momentum and make no assumptions about growth.

Our approach is to bound the spectral norm of the products $\norm{A_k \cdots A_j}$ using submultiplicativity of matrix norms. This recovers linear convergence to a neighborhood of the minimizer, but the rate is no longer accelerated.

Define the quantities
\begin{align*}
    C_\lambda(\alpha, \beta) &= (1 - \alpha (1 + \beta) \lambda)^2 + \alpha^2 \lambda^2 + \beta^2 (\beta^2 + 1)\\
    \tilde{\Delta}_\lambda(\alpha, \beta) &= C_\lambda(\alpha, \beta)^2 - 4 \beta^2 (1 - \alpha \lambda)^2 \\
    R_\lambda(\alpha, \beta) &= \frac{1}{\sqrt{2}} \left( C_\lambda(\alpha, \beta) + \sqrt{ \tilde{\Delta}_{\lambda}(\alpha, \beta)} \right)^{1/2}
\end{align*}
and let $R(\alpha, \beta) = \max_{\lambda \in [\mu, L]} R_\lambda(\alpha, \beta)$.

\begin{theorem}
\label{th:finite-sum}
Let $\alpha$ and $\beta$ be chosen so that $R(\alpha, \beta) < 1$. Then for all $k \ge 0$,
\begin{align*}
    \E \norm{\itr{y}{k+1}-\xstar} \le& R(\alpha,\beta)^{k} \norm{\itr{y}{1}-\xstar} \\
    &+ \frac{\alpha \sqrt{(1+\beta)^2 + 1}}{1 - R(\alpha,\beta)} \std,
\end{align*}
where
\[
    \std = \frac{1}{n}\sum^n_{i=1} \norm{\nabla f_i(\xstar)}.
\]
\end{theorem}
Theorem~\ref{th:finite-sum} is proved in Appendix~\ref{sec:proof-thm-finite-sum} for general $L$-smooth $\mu$-strongly-convex functions. Note that if an interpolation condition holds (a weaker assumption than the strong growth condition), then $\std = 0$.

Theorem~\ref{th:finite-sum} shows that the \NASG method can be made to converge in the finite-sum setting for $L$-smooth $\mu$-strongly convex objective functions when run with constant step-sizes.
In particular, the algorithm converges at a linear rate to a neighborhood of the minimizer that is proportional to the variance of the noise terms.
Note that this theorem also allows for negative momentum parameters.
Using the spectral norm to guarantee stability is restrictive, in that it is sufficient but not necessary. There may be values of $\alpha$ and $\beta$ for which $R(\alpha, \beta) \ge 1$ and the algorithm still converges. Having $R(\alpha, \beta) < 1$ ensures that $\norm{r_k} + \norm{v_k}$ decreases at every iteration.

\begin{corollary}
\label{cor:finite-sum-general}
Suppose that $\alpha < \frac{2}{L}$ and $\beta = 0$. Then for all $k \ge 0$
\begin{align*}
    \E \norm{\itr{y}{k} - \xstar}
    \leq&\
        \varrho(\alpha)^{k}
        \norm{\itr{y}{0} - \xstar} + \frac{\alpha}{1-\varrho(\alpha)} \std,
\end{align*}
where $\varrho(\alpha) \defeq \max_{\lambda \in \{\mu, L\}} \abs{1-\alpha\lambda}$.
\end{corollary}
Corollary~\ref{cor:finite-sum-general} is proved in Appendix~\ref{sec:proof-of-cor-finite-sum} for smooth strongly-convex functions, and shows the convergence of \SGD in the finite-sum setting without making any assumptions on the noise distribution.
\begin{corollary}
\label{cor:finite-sum}
Suppose that $\alpha = \frac{2}{\mu + L}$ and $\beta = 0$. Then for all $k \ge 0$
\begin{align*}
    \E \norm{\itr{y}{k} - \xstar}
    \leq&\
        \left(\frac{\cond - 1}{\cond + 1} \right)^{k}
        \norm{\itr{y}{0} - \xstar} + \frac{1}{\mu} \std.
\end{align*}
\end{corollary}
Corollary~\ref{cor:finite-sum}, is proved in Appendix~\ref{sec:proof-of-cor-finite-sum} for smooth strongly-convex functions, and shows that \SGD converges to a neighborhood of $x^\star$ at the same linear rate as \GD, viz.~\eqref{eq:gd-og}, in the finite-sum setting, without making any assumptions on the noise distribution, such as the strong-growth condition; a novel result to the best of our knowledge. Moreover, when the interpolation condition holds, we have that $\std = 0$.

\begin{figure*}[!t]
\centering
\subfloat[.33\textwidth][$Q=16$]{
    \includegraphics[width=.33\textwidth]{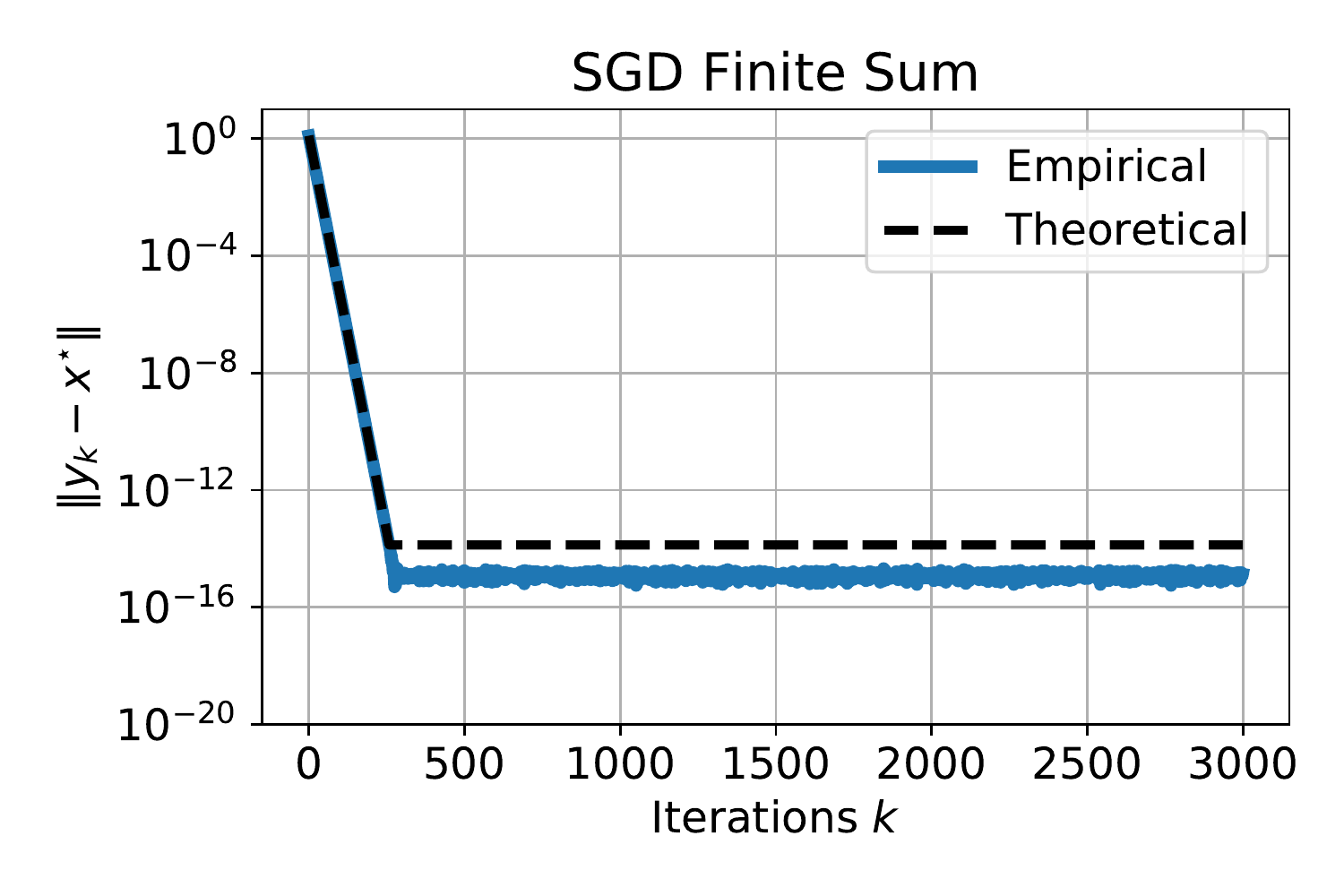}
    \label{sub-fig:sgd-finite-sum-1}}
\subfloat[.33\textwidth][$Q=32$]{
    \includegraphics[width=.33\textwidth]{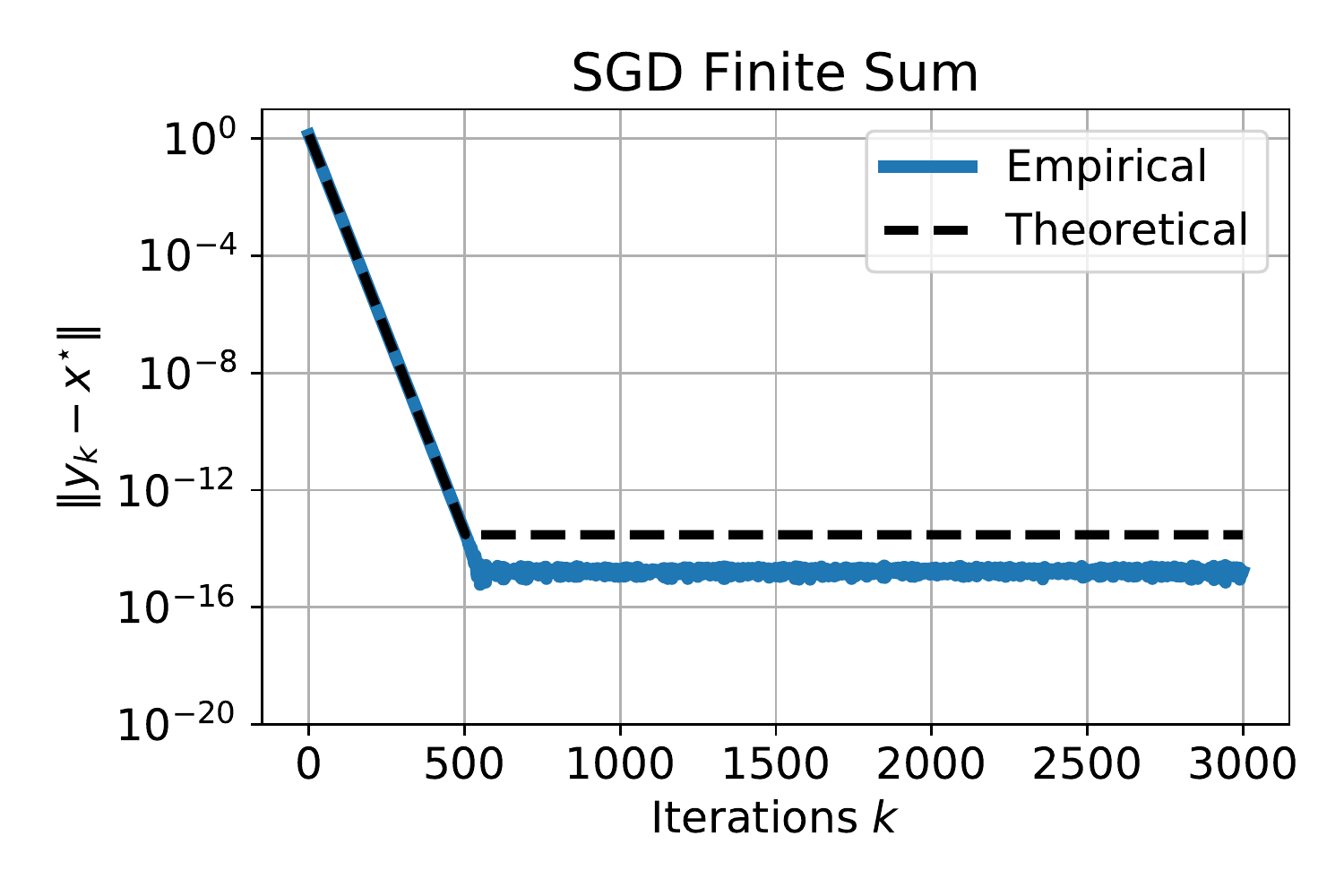}
    \label{sub-fig:sgd-finite-sum-2}}
\subfloat[.33\textwidth][$Q=64$]{
    \includegraphics[width=.33\textwidth]{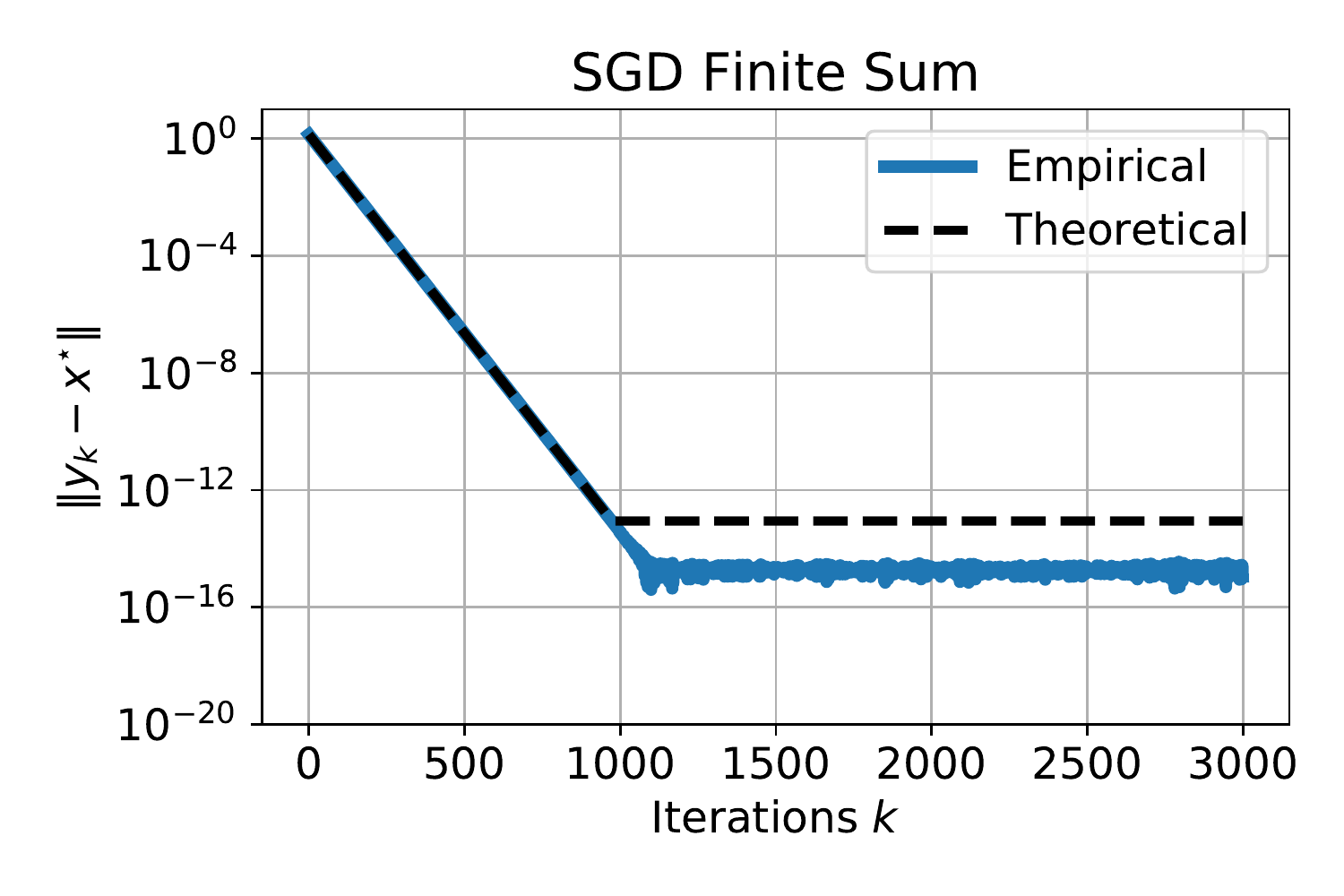}
    \label{sub-fig:sgd-finite-sum-3}}
\caption{Visualizing the accuracy with which Corollary~\ref{cor:finite-sum} predicts the theoretical convergence of \SGD with step-size $\alpha=\nicefrac{2}{(\mu+L)}$ in the finite-sum setting, when minimizing randomly generated least-squares problems with various condition numbers $Q$. The finite-sum problem consists of 25000 data samples, with 2 features each, partitioned into 50 mini-batches, each with condition number $Q$. At each iteration one of the 50 mini-batches is sampled to compute a stochastic gradient step. Dashed lines indicate the theoretical convergence rate and variance bound from Corollary~\ref{cor:finite-sum}. Solid lines indicate the empirical convergence observed in practice. The convergence rate and variance bound in Corollary~\ref{cor:finite-sum} is tight.}
\label{fig:sgd-finite-sum}
\end{figure*}

Figure~\ref{fig:sgd-finite-sum} illustrates the tightness of the convergence rate and variance bound in Corollary~\ref{cor:finite-sum} when minimizing randomly generated least-squares problems with various condition numbers.
The finite-sum least-squares problem consists of 25000 data samples, with 2 features each, partitioned into 50 mini-batches, each with condition number $Q$. At each iteration, one of the 50 mini-batches is sampled to compute a stochastic gradient step.
Dashed lines indicate the theoretical convergence rate and variance bound from Corollary~\ref{cor:finite-sum}. Solid lines indicate the empirical convergence observed in practice. The convergence rate and variance bound in Corollary~\ref{cor:finite-sum} provide a tight characterization of the \SGD convergence observed in practice.

\section{Conclusions}

This paper contributes to a broader understanding of the \NASG method in stochastic settings. Although the method behaves well in the stochastic approximation setting, it may diverge in the finite-sum setting when using the usual step-size and momentum. This emphasizes the important role the bounded variance assumption plays in the stochastic approximation setting, since a similar condition does not necessarily hold in the finite-sum setting. Forsaking acceleration guarantees, we provide conditions under which the \NASG method is guaranteed to converge in the smooth strongly-convex finite-sum setting with constant step-size and momentum, without assuming any growth or interpolation condition.

We believe there is scope to obtain tighter convergence bounds for the \NASG method with constant step-size and momentum in the finite-sum setting.
Convergence guarantees using the joint spectral radius are likely to provide the tightest and most intuitive bounds, but are also difficult to obtain.
To date, Lyapunov-based proof techniques have been the most fruitful in the literature.

We also believe that there is scope to improve the robustness of Nesterov's method to inconsistent mini-batches in the finite-sum setting.
For example, adaptive restarts, which have been show to improve the convergence rate of Nesterov's method~\cite{o2015adaptive} with deterministic gradients, may also be able to mitigate the divergence behaviour identified in this paper.

We also believe that future work understanding the role that negative momentum parameters play in practice may lead to improved optimization of machine learning models.
All convergence guarantees and variance bounds in this paper hold for both positive and negative momentum parameters.
Our variance bounds and theoretical rates support the observation that negative momentum parameters may slow-down convergence, but can also lead to non-trivial variance reduction.
Previous work has found negative momentum to be useful in asynchronous distributed optimization \cite{mitliagkas2016asynchrony} and for stabilizing adversarial training \cite{gidel2018negative}.
Although it is almost certainly not possible (in general) to obtain zero variance solutions by only using negative momentum parameters, for Deep Learning practitioners that already use the \NASG method to train their models, perhaps momentum schedules incorporating negative values towards the end of training can improve performance.

\section*{Acknowledgements}
We thank Leon Bottou, Aaron Defazio, Alexandre Defossez, Tom Goldstein, and Mark Tygert for feedback and conversations about earlier versions of this work.

% Path specified relative to the icml-template file
\bibliographystyle{icml/icml2020}
\bibliography{refs}

%%%%% Beginning of Supplementary Material
\appendix
\counterwithin{figure}{section}

\onecolumn
\icmltitle{Supplementary Material for\\ On the Convergence of Nesterov's Accelerated Gradient Method \\ in Stochastic Settings}

\section{Proof of Theorem~\ref{thm:stochapprox-quadratic}}
\label{sec:stochapprox-quadratic-proof}

We begin from \eqref{eq:stochapprox-quadratic-unrolled}. By taking the squared norm on both sides, and recalling that the random vectors $\itr{\noise}{k}$ have zero mean and are mutually independent, we have
\begin{align}
	\mathbb{E} \norm{y_{k+1} - x^\star}^2 &\le \mathbb{E} \norm{ \begin{bmatrix} r_{k+1} \\ v_k \end{bmatrix} }^2 \notag \\
	&= \mathbb{E}_{\itr{\noise}{k},\ldots,\itr{\noise}{1}} \left[ \norm{ A^k \begin{bmatrix} x_1 - x^\star \\ 0 \end{bmatrix} - \alpha \sum_{j=1}^k A^{k - j} \begin{bmatrix} (1 + \beta)I \\ I\end{bmatrix} \zeta_j }^2\right] \notag \\
	&= \mathbb{E}_{\itr{\noise}{k}} \left[ \cdots \mathbb{E}_{\itr{\noise}{1}} \left[ \norm{ A^k \begin{bmatrix} x_1 - x^\star \\ 0 \end{bmatrix} - \alpha \sum_{j=1}^k A^{k - j} \begin{bmatrix} (1 + \beta)I \\ I\end{bmatrix} \zeta_j }^2\right] \cdots \right]\notag \\
	&= \norm{ A^k \begin{bmatrix} x_1 - x^\star \\ 0 \end{bmatrix} }^{2} + \alpha^{2} \sum_{j=1}^k \E_{\itr{\noise}{j}} \norm{ A^{k - j} \begin{bmatrix} (1 + \beta)I \\ I\end{bmatrix} \zeta_j }^2 \notag \\
	&\le \norm{A^k}^2 \norm{x_1 - x^\star}^2 + \alpha^2 ((1 + \beta)^2 + 1) \sigma^2 \sum_{j=1}^k \norm{A^{k-j}}^2. \label{eq:stochapprox-proof-1}
\end{align}
Recall that the spectral radius of a square matrix $A \in \R^{2d \times 2d}$ is defined as $\max_{i=1, \dots, 2d} \abs{\lambda_i(A)}$, where $\lambda_i(A)$ is the $i$th eigenvalue of $A$. The spectral radius satisfies \citep{hornJohnson}
\[
\rho(A)^k \le \norm{A^k} \quad \text{ for all } k,
\]
and (Gelfand's theorem)
\[
\lim_{k \rightarrow \infty} \norm{A^k}^{1/k} = \rho(A).
\]
Hence, for any $\epsilon > 0$, there exists a $K_\epsilon$ such that $\norm{A^k}^{1/k} \le (\rho(A) + \epsilon)$ for all $k \ge K_\epsilon$. Let
\begin{equation}
\label{eq:C-val-thm1}
    C_{\epsilon} = \max_{k < K_\epsilon} \max\left\{ 1, \frac{\norm{A^k}}{(\rho(A) + \epsilon)^k} \right\}.
\end{equation}
Then $\norm{A^k} \le C_\epsilon (\rho(A) + \epsilon)^k$ for all $k$. Moreover, if $\norm{A^k}^{1/k}$ converges monotonically to $\rho(A)$, then $C_\epsilon \le \norm{A} / \rho(A)$.

Now, recall that we have assumed $f(x) = \frac{1}{2} x^\T H x - b^\T x + c$ where $H \in \R^{d \times d}$ is symmetric, and we have also assumed that $f$ is $L$-smooth and $\mu$-strongly convex. Thus all eigenvalues of $H$ satisfy $\mu \le \lambda_i(H) \le L$.

\begin{lemma} \label{lem:spectral-radius-A}
For $A$ as defined in \eqref{eq:stochapprox-quadratic-A}, we have $\rho(A) = \max\{ \rho_{\mu}(\alpha, \beta), \rho_L(\alpha, \beta)\}$ where
\[
\rho_\lambda(\alpha, \beta) = \begin{cases} \tfrac{1}{2} \abs{(1 + \beta)(1 - \alpha \lambda)} + \tfrac{1}{2} \sqrt{\Delta_\lambda} & \text{ if } \Delta_\lambda \ge 0, \\
\sqrt{\beta(1 - \alpha \lambda)}  & \text{ otherwise,} \end{cases}
\]
and $\Delta_\lambda = (1 + \beta)^2 (1 - \alpha \lambda)^2 - 4 \beta (1 - \alpha \lambda)$.
\end{lemma}

\begin{proof}
Since $H$ is real and symmetric, it has a real eigenvalue decomposition $H = U \Lambda_H U^\T$, where $U \in \R^{d \times d}$ is an orthogonal matrix and $\Lambda_H$ is the diagonal matrix of eigenvalues of $H$. Observe that $A$ can be viewed as a $2 \times 2$ block matrix with $d \times d$ blocks that all commute with each other, since each block is an affine matrix function of $H$. Thus, by \citet[Lemma~5]{polyak1964some}, $\xi$ is an eigenvalue of $A$ if and only if there is an eigenvalue $\lambda$ of $H$, such that $\xi$ is an eigenvalue of the $2 \times 2$ matrix
\begin{equation}
\label{eq:b-lambda}
    B(\lambda) \defeq \begin{bmatrix}
        1 - \alpha (1 + \beta) \lambda & \beta^2 \\
        - \alpha \lambda & \beta
    \end{bmatrix}.
\end{equation}
The characteristic polynomial of $B(\lambda)$ is 
\begin{align*}
\xi^2 - (1 + \beta)(1 - \alpha \lambda) \xi + \beta (1 - \alpha \lambda) = 0,
\end{align*}
from which it follows that eigenvalues of $B(\lambda)$ are given by $\rho_\lambda(\alpha, \beta)$; see, \eg., \citet[Appendix~A]{lessard2016analysis}. Note that the characteristic polynomial of $B(\lambda)$ is the same as the characteristic polynomial of a different matrix appearing in \citet{lessard2016analysis}, that arises from a different analysis of the \NAG method. Finally, as discussed in \citet{lessard2016analysis}, for any fixed values of $\alpha$ and $\beta$, the function $\rho_\lambda(\alpha, \beta)$ is quasi-convex in $\lambda$, and hence the maximum over all eigenvalues of $A$ is achieved at one of the extremes $\lambda = \mu$ or $\lambda = L$.
\end{proof}

To complete the proof of Theorem~\ref{thm:stochapprox-quadratic}, use Lemma~\ref{lem:spectral-radius-A} with \eqref{eq:stochapprox-proof-1} to obtain that, for any $\epsilon > 0$, there is a positive constant $C_\epsilon$ such that
\begin{align*}
\mathbb{E}[\norm{y_{k+1} - x^\star}^2] &\le C_\epsilon \left( (\rho(A) + \epsilon)^{2k} \norm{x_0 - x^\star}^2 + \alpha^2 ((1 + \beta)^2 + 1) \sigma^2 \sum_{j=1}^k (\rho(A) + \epsilon)^{2(k-j)} \right) \\
&\le C_\epsilon \left( (\rho(A) + \epsilon)^{2k} \norm{x_0 - x^\star}^2 + \frac{\alpha^2 ((1 + \beta)^2 + 1)}{1 - (\rho(A) + \epsilon)^2} \sigma^2 \right).
\end{align*}

\subsection{Estimating the constant $C_\epsilon$}
\label{sub-sec:estimating-cepsilon}
For the theoretical plots in the numerical experiments in Section~\ref{sub-sec:stochastic-approximation-numerical} and in Appendix~\ref{sec:additional_experiments} below, we estimate the constant $C_\epsilon$ by taking $K_\epsilon \approx 2$ in~\eqref{eq:C-val-thm1}.
That is, for arbitrarily small $\epsilon$ and all $k \geq 2$, we approximate $\norm{A^k}^{1/k}$ by $(\radius{A} + \epsilon)$.
Therefore, the summation term in~\eqref{eq:stochapprox-proof-1} is approximated as
\begin{equation}
\label{eq:C-approx}
    \alpha^2((1+\beta)^2 + 1)\left(\frac{1}{1-\rho(\alpha, \beta)^2} + (\norm{A}^2 - \rho(\alpha, \beta)^2)\right),
\end{equation}
where $\norm{A}$ denotes the largest singular value of $A$ in~\eqref{eq:stochapprox-quadratic-A}, and $\rho(\alpha,\beta)$ is the largest eigenvalue of $A$.
The first term in~\eqref{eq:C-approx} corresponds to the geometric limit of the summation term in~\eqref{eq:stochapprox-proof-1} after taking matrix norms and approximating the norms of matrix products by powers of the spectral radius for all products $k \geq 2$.
The difference term in~\eqref{eq:C-approx} is simply used to correct for the case $k = 1$.
Setting $C_\epsilon \frac{\alpha^2((1+\beta)^2 + 1)}{1-\rho(\alpha,\beta)^2}$ equal to~\eqref{eq:C-approx} and solving for $C_\epsilon$ gives us the approximate expression for $C_\epsilon$ used in the theoretical plots in Section~\ref{sub-sec:stochastic-approximation-numerical}.

\section{Proofs of Corollary~\ref{cor:stochapprox-nasg} and Theorem~\ref{thm:stochapprox-sgd}}
\label{sec:proofs-stochapprox-cors}

Taking $\alpha = 1/L$ and $\beta = \frac{\sqrt{\cond}-1}{\sqrt{\cond} + 1}$, we find that $\rho(\alpha, \beta) = \frac{\sqrt{\cond} - 1}{\sqrt{\cond}}$. 
Since $f(x) = \frac{1}{2}x^T Hx - b^T x + c$ is an $L$-smooth $\mu$-strongly convex quadratic, all eigenvalues of $H$ are bounded between $\mu$ and $L$.
Therefore, from~\citet[Lemma~5]{polyak1964some}, we have that $\norm{A^k}_2 \leq \max_{\lambda \in [\mu, L]} \norm{B(\lambda)^k}_2 \leq \max_{\lambda \in [\mu, L]} \sqrt{d} \norm{B(\lambda)^k}_\infty $, where $B(\lambda)$ is as defined in~\eqref{eq:b-lambda}.
The eigenvalues of $B(\lambda)^k$ are maximized at $\lambda = \mu$ for $k > 1$, therefore, for large $k$, $\norm{B(\lambda)^k}_\infty$ is maximized at $\lambda = \mu$.

Note that the Jordan form of $B(\mu)$ is given by $V J V^{-1}$, where
\begin{align*}
V =
    \begin{bmatrix}
        \frac{\sqrt{\cond}(\sqrt{\cond}-1)}{\sqrt{\cond}+1} & \cond \\
        -1 & 0
    \end{bmatrix} \qquad \text{ and } \qquad
J =
    \begin{bmatrix}
        \frac{\sqrt{\cond}-1}{\sqrt{\cond}} & 1 \\
        0 & \frac{\sqrt{\cond}-1}{\sqrt{\cond}}
    \end{bmatrix}.
\end{align*}
Using the Jordan form, we determine that $B(\mu)^{k}$ is
\begin{align*}
B(\mu)^k = 
\begin{bmatrix}
    \left(1 + \frac{k}{\sqrt{\cond}+1}\right) \left(\frac{\sqrt{\cond}-1}{\sqrt{\cond}}\right)^{k} & k \left(\frac{\sqrt{\cond}-1}{\sqrt{\cond}+1}\right)^2  \left(\frac{\sqrt{\cond}-1}{\sqrt{\cond}}\right)^{k - 1} \\[1em]
    -\frac{k}{\cond} \left(\frac{\sqrt{\cond}-1}{\sqrt{\cond}}\right)^{k-1} & \left(1 - \frac{k}{\sqrt{\cond} + 1}\right) \left(\frac{\sqrt{\cond}-1}{\sqrt{\cond}}\right)^{k}
\end{bmatrix}.
\end{align*}
Therefore, we have that
\begin{align}
\label{eq:b-lambda-infty-norm}
    \norm{B(\mu)^k}_\infty
    \leq&
    \left(1 + \frac{k}{\sqrt{\cond}+1}\right) \left(\frac{\sqrt{\cond}-1}{\sqrt{\cond}}\right)^{k} + k 
    \max\left\{
        \frac{1}{\cond},\
        \left(\frac{\sqrt{\cond}-1}{\sqrt{\cond}+1}\right)^2
    \right\}
    \left(\frac{\sqrt{\cond}-1}{\sqrt{\cond}}\right)^{k - 1}.
\end{align}
Therefore for large $k$
\[
    \norm{A^k}^2_2 \leq d \norm{B(\mu)^k}^2_\infty = \left(\frac{\sqrt{\cond}-1}{\sqrt{\cond}} + \itr{\epsilon}{k}\right)^{2k},
\]
where $\itr{\epsilon}{k} \sim (\sqrt[k]{k} - 1)$.
Also observe that 
\begin{align*}
\frac{\alpha^2 ((1 + \beta)^2 + 1)}{1 - \rho(\alpha, \beta)^2} &= \frac{1}{L^2} \frac{\left(\frac{2 \sqrt{\cond}}{\sqrt{\cond} + 1}\right)^2 + 1}{1 - \left( \frac{\sqrt{\cond} - 1}{\sqrt{\cond}} \right)^2} \\
&= \frac{1}{L^2} \frac{5 \cond^2 + 2 \cond^{3/2} + \cond}{(\sqrt{\cond} + 1)^2 (2 \sqrt{\cond} - 1)}.
\end{align*}
Since $f$ is $L$-smooth,
\[
f(y_{k+1}) - f^\star \le \frac{L}{2} \norm{y_{k+1} - x^\star}^2.
\]
Thus, by Theorem~\ref{thm:stochapprox-quadratic} we have
\begin{align*}
\E[f(y_{k+1})] - f^\star &\le \frac{L}{2} \left(\frac{\sqrt{\cond} - 1}{\sqrt{\cond}} + \itr{\epsilon}{k} \right)^{2k} \norm{x_0 - x^\star}^2 + C_\epsilon \frac{5 Q^2 + 2 Q^{3/2} + Q}{2 L (2 \sqrt{Q} - 1) (\sqrt{Q} + 1)^2} \std^2,
\end{align*}
which completes the proof of Corollary~\ref{cor:stochapprox-nasg}.

To prove Theorem~\ref{thm:stochapprox-sgd}, first observe that when $\beta = 0$, the recursion simplifies significantly. Specifically, then $y_{k+1} = x_k$, $v_k = - \alpha g_k$, and we have (using similar notation as in the proof of Theorem~\ref{thm:stochapprox-quadratic})
\begin{align*}
r_{k+1} &= (I - \alpha H_k) r_k - \alpha \zeta_k \\
&= \prod_{j=1}^k (I - \alpha H_j) r_1 - \alpha \zeta_k - \alpha \sum_{j=1}^{k-1} \prod_{l=j+1}^k (I - \alpha H_l) \zeta_j,
\end{align*}
where
\[
H_j = \int_0^1 \nabla f^2(x^\star - t \  r_j) \mathrm{d}t.
\]

Of course, since $f$ is $L$-smooth and $\mu$-strongly convex, all eigenvalues of $H_j$ lie in the interval $[\mu, L]$ for all $j \ge 0$. 

Now, taking the squared norm on both sides, and recalling that the random vectors $\itr{\noise}{k}$ have zero mean and are mutually independent, we have
\begin{align*}
    \E \norm{y_{k+1} - x^\star}^2 &= \E \norm{r_{k+1}}^2 \\
    &= \E_{\itr{\noise}{k},\ldots,\itr{\noise}{1}} \norm{ \prod_{j=1}^k (I - \alpha H_j) r_1 - \alpha \zeta_k - \alpha \sum_{j=1}^{k-1} \prod_{l=j+1}^k (I - \alpha H_l) \zeta_j }^2 \\
    &= \E_{\itr{\noise}{k}} \left[ \cdots \E_{\itr{\noise}{1}}\left[ \norm{ \prod_{j=1}^k (I - \alpha H_j) r_1 - \alpha \zeta_k - \alpha \sum_{j=1}^{k-1} \prod_{l=j+1}^k (I - \alpha H_l) \zeta_j }^2\right] \cdots \right] \\
    &= \left(\prod_{j=1}^k \norm{I - \alpha H_j}^2\right) \norm{x_1 - x^\star}^2 + \alpha^2 \E_{\itr{\noise}{k}}\norm{\itr{\noise}{k}}^2 + \alpha^2 \sum_{j=1}^{k-1} \E_{\itr{\noise}{j}} \norm{\left(\prod_{l=j+1}^{k} I - \alpha H_l \right) \itr{\noise}{j}}^2 \\
    &\le \left(\prod_{j=1}^k \norm{I - \alpha H_j}^2\right) \norm{x_1 - x^\star}^2 + \alpha^2 \sigma^2 + \alpha^2 \sigma^2 \sum_{j=1}^{k-1} \left(\prod_{l=j+1}^{k} \norm{I - \alpha H_l}^2\right).
\end{align*}

Now, since $I - \alpha H_j$ is symmetric, we have $\norm{(I - \alpha H_j)}^2 = \rho(I - \alpha H_j)^{2}$, where $\rho(I - \alpha H_j)$ denotes the spectral radius of $I - \alpha H_j$ (the largest magnitude of an eigenvalue of $I - \alpha H_j$). For $\alpha = \frac{2}{\mu + L}$, and since the eigenvalues of $H_j$ lie in the interval $[\mu, L]$, it is straightforward to show that $\rho(I - \alpha H_j) = \frac{\cond - 1}{\cond + 1}$.

Therefore we have
\begin{align*}
\E\left[ \norm{y_{k+1} - x^\star}^2 \right] &\le \left(\frac{\cond - 1}{\cond + 1}\right)^{2k} \norm{x_0 - x^\star}^2 + \alpha^2 \sigma^2 \sum_{j=1}^k \left(\frac{\cond - 1}{\cond + 1}\right)^{2(k-j)} \\
&\le \left(\frac{\cond - 1}{\cond + 1}\right)^{2k} \norm{x_0 - x^\star}^2 + \frac{\alpha^2 \sigma^2}{1 - (\frac{\cond - 1}{\cond + 1})^2} \\
&= \left(\frac{\cond - 1}{\cond + 1}\right)^{2k} \norm{x_0 - x^\star}^2 + \frac{Q}{2L} \sigma^2,
\end{align*}
which completes the proof of Theorem~\ref{thm:stochapprox-sgd}.

\section{Permutation Matrix Construction}
\label{sec:permutation}

For a vector $x \in \R^d$, let $\diag{x}$ denote a $d \times d$ diagonal matrix with its $i$th diagonal entry equal to $x_i$. 
Let $a,b,c,d \in \R^d$ and suppose $M \in \R^{2d \times 2d}$ is the matrix
\[
M = \begin{bmatrix}
    \diag{a} & \diag{b} \\ \diag{c} & \diag{d}
\end{bmatrix}.
\]
Let $P \in \{0,1\}^{2d \times 2d}$ be the permutation matrix with entries $P_{i,j}$ for $i,j = 1,\dots,2d$ given by
\[
P_{i,j} = \begin{cases} 1 & \text{ if } i \text{ is odd and } j = (i-1)/2 + 1 \\
1 & \text{ if } i \text{ is even and } j = d + \lfloor \frac{i-1}{2} \rfloor + 1 \\
0 & \text{ otherwise.}
\end{cases}
\]
Then one can verify that
\[
P M P^\T = \begin{bmatrix} 
    T_1 & 0 & \cdots & 0 \\
    0 & T_2 & \cdots & 0 \\
    \vdots & \vdots & \ddots & \vdots \\
    0 & 0 & \cdots & T_d
\end{bmatrix}
\]
where, for $j=1,\dots,d$, $T_j$ is the $2 \times 2$ matrix
\[
T_j = \begin{bmatrix}
   a_j & b_j \\
   c_j & d_j
\end{bmatrix}.
\]

\section{Proof of Lemma~\ref{lem:spectral-radius-T_j}}
\label{sec:spectral-radius-T_j-proof}

Recall that $\alpha = 1/L$ and $\beta = \frac{\sqrt{\cond} - 1}{\sqrt{\cond} + 1}$.
For matrices of the form
\[
T_{k} = B(L) B(\mu)^{k_1} B(L) B(\mu)^{k_2} \cdots B(L) B(\mu)^{k_s} B(L),
\]
where
\[
	B(\lambda) =
	\begin{bmatrix}
		1 - \alpha (1+ \beta) \lambda & \beta^2 \\
		- \alpha \lambda & \beta
	\end{bmatrix},
\]
we would like to show that the spectral radius $\rho(T_k)$ is equal to
\[
\rho(T_k) = \left( \frac{\sqrt{\cond} - 1}{\sqrt{\cond}}\right)^k \times k_1 k_2 \cdots k_s.
\]

To see this, first note that the Jordan form of $B(\mu)$ is given by $V J V^{-1}$, where
\begin{align*}
V =
    \begin{bmatrix}
        \frac{\sqrt{\cond}(\sqrt{\cond}-1)}{\sqrt{\cond}+1} & \cond \\
        -1 & 0
    \end{bmatrix} \qquad \text{ and } \qquad
% V^{-1} =
%     \begin{bmatrix}
%         0 & -1 \\
%         \frac{1}{\cond} & \frac{\sqrt{\cond}-1}{\sqrt{\cond}(\sqrt{\cond}+1)}
%     \end{bmatrix} \\
J =
    \begin{bmatrix}
        \frac{\sqrt{\cond}-1}{\sqrt{\cond}} & 1 \\
        0 & \frac{\sqrt{\cond}-1}{\sqrt{\cond}}
    \end{bmatrix}.
\end{align*}
Using the Jordan form, we determine that $B(\mu)^{\itr{k}{\ell}}$ is
\begin{align*}
B(\mu)^{\itr{k}{\ell}} = 
\begin{bmatrix}
    \left(1 + \frac{\itr{k}{\ell}}{\sqrt{\cond}+1}\right) \left(\frac{\sqrt{\cond}-1}{\sqrt{\cond}}\right)^{\itr{k}{\ell}} & \itr{k}{\ell} \left(\frac{\sqrt{\cond}-1}{\sqrt{\cond}+1}\right)^2  \left(\frac{\sqrt{\cond}-1}{\sqrt{\cond}}\right)^{\itr{k}{\ell} - 1} \\[1em]
    -\frac{\itr{k}{\ell}}{\cond} \left(\frac{\sqrt{\cond}-1}{\sqrt{\cond}}\right)^{\itr{k}{\ell}-1} & \left(1 - \frac{\itr{k}{\ell}}{\sqrt{\cond} + 1}\right) \left(\frac{\sqrt{\cond}-1}{\sqrt{\cond}}\right)^{\itr{k}{\ell}}
\end{bmatrix}.
\end{align*}
Through direct matrix multiplication
\begin{align*}
B(L) B(\mu)^{\itr{k}{\ell}} B(L) =
    - \left(\frac{\sqrt{\cond}-1}{\sqrt{\cond}}\right)^{\itr{k}{\ell}+1} \itr{k}{\ell} B(L).
\end{align*}
Therefore,
\begin{align*}
    T_j = (-1)^{s-1} \left(\frac{\sqrt{\cond}-1}{\sqrt{\cond}}\right)^{k-1-\itr{k}{s}} \itr{k}{1}\itr{k}{2}\cdots\itr{k}{s-1}B(L)B(\mu)^{\itr{k}{s}}.
\end{align*}
Finally, the spectral-radius of $B(L)B(\mu)^{\itr{k}{s}}$ is
\begin{align*}
    \radius{B(L)B(\mu)^{\itr{k}{s}}} = \left(\frac{\sqrt{\cond}-1}{\sqrt{\cond}}\right)^{\itr{k}{s} + 1} \itr{k}{s},
\end{align*}
and hence\begin{align*}
    \radius{T_k} = \left(\frac{\sqrt{\cond}-1}{\sqrt{\cond}}\right)^{k} \itr{k}{1}\itr{k}{2}\cdots\itr{k}{s}.
\end{align*}

\section{Proof of Theorem~\ref{th:finite-sum}}
\label{sec:proof-thm-finite-sum}

Since the functions $f_i$ are assumed to be twice continuously differentiable, 
by \eqref{eq:taylor} we can express the mini-batch gradients as
\begin{align} \label{eq:finite-sum-gradient-taylor}
    g_k = \tilde{H}_k r_k + z_k,
\end{align}
where
\[
\tilde{H}_k = \sum_{i=1}^n v_{k,i} \int_0^1 \nabla^2 f_i(x^\star + t r_k) \mathrm{d}t
\]
and
\[
z_k = \sum_{i=1}^n v_{k,i} \nabla f_i(x^\star).
\]

By convexity of norms,
\[
\norm{z_k} \le \sum_{i=1}^n v_{k,i} \norm{ \nabla f_k(x^\star) }.
\]
Hence, taking expectations gives
\begin{align*}
\E_k[\norm{z_k}] &\le \frac{1}{n} \sum_{i=1}^n \norm{\nabla f_i(x^\star) } \\
&= \sigma.
\end{align*}

Using \eqref{eq:finite-sum-gradient-taylor} in \eqref{eq:recursion-with-g} and unrolling, we obtain
\begin{equation} \label{eq:finite-sum-general-unrolled}
\begin{bmatrix}
r_{k+1} \\ v_k
\end{bmatrix}
= A_k \cdots A_1 \begin{bmatrix}
r_1 \\ v_0
\end{bmatrix}
- \alpha
\begin{bmatrix}
(1 + \beta) I \\ I
\end{bmatrix}
z_k - \alpha \sum_{j=1}^{k-1} (A_k \cdots A_{j+1}) \begin{bmatrix}
(1 + \beta) I \\ I
\end{bmatrix} z_j,
\end{equation}
where
\[
A_k = \begin{bmatrix}
I - \alpha (1 + \beta) \tilde{H}_k & \beta^2 I \\
- \alpha \tilde{H}_k & \beta I
\end{bmatrix}.
\]

By submultiplicativity of matrix norms, $\norm{A_k \cdots A_{j+1}} \le \prod_{l = j+1}^k \norm{A_l}$. Thus we turn our attention to bounding the spectral norm of $A_k$.

\begin{lemma}
\[
\norm{A_k} \leq \max_{\lambda \in [\mu, L]} \norm{B(\lambda)} = R(\alpha, \beta).
\]
\end{lemma}

\begin{proof}
For all $k \ge 0$, every eigenvalue of $\tilde{H}_k$ lies in the interval $[\mu, L]$, based on the assumption that each function $f_i$ is $L$-smooth and $\mu$-strongly convex. It follows from \citet[Lemma~5]{polyak1964some} that there exists an eigenvalue $\lambda$ of $\tilde{H}_k$ such that $\norm{A_k}$ is equal to the spectral norm of
\[
B(\lambda) = \begin{bmatrix}
1 - \alpha (1 + \beta) \lambda & \beta^2 \\
- \alpha \lambda & \beta
\end{bmatrix}.
\]

We next compute $\norm{B(\lambda)}$, which is equal to the square root of the largest eigenvalue of
\[
B(\lambda)^\T B(\lambda) = \begin{bmatrix}
(1 - \alpha (1 + \beta)\lambda)^2 + \alpha^2 \lambda^2 & \beta^2 (1 - \alpha (1 + \beta)\lambda) - \alpha \beta \lambda \\
\beta^2 (1 - \alpha (1 + \beta)\lambda) - \alpha \beta \lambda & \beta^2 (\beta^2 + 1)
\end{bmatrix}.
\]
The characteristic polynomial of $B(\lambda)^\T B(\lambda)$ is
\[
\xi^2 - C_\lambda(\alpha, \beta) \xi + \beta^2 (1 - \alpha \lambda)^2 = 0,
\]
where
\[
C_\lambda(\alpha, \beta) = (1 - \alpha (1 + \beta)\lambda)^2 + \alpha^2 \lambda^2 + \beta^2 (\beta^2 + 1).
\]
The largest root of the characteristic polynomial is equal to 
\[
R_\lambda(\alpha, \beta)^2 = \frac{1}{2} \left( C_\lambda(\alpha, \beta) + \sqrt{C_\lambda(\alpha, \beta)^2 - 4 \beta^2 (1 - \alpha \lambda)^2} \right)
\]
which is equal to $\norm{B(\lambda)}^2$. Therefore
\[
\norm{A_k} \le \max_{\lambda \in [\mu, L]} R_\lambda(\alpha, \beta).
\]
\end{proof}

Assume that $\alpha$ and $\beta$ have been chosen so that $R(\alpha, \beta) < 1$. Then for all $k$ and $j+1$, $\norm{A_k \cdots A_{j+1}} \le \prod_{l=j+1}^k \norm{A_l} \le R(\alpha, \beta)^{k - j}$.

Taking the norm on both sides of \eqref{eq:finite-sum-general-unrolled} and using the triangle inequality, we have
\begin{align}
\norm{\begin{bmatrix}
r_{k+1} \\ v_k
\end{bmatrix}}
&\le R(\alpha, \beta)^k \norm{\begin{bmatrix} r_1 \\ v_0 \end{bmatrix}} + \alpha \sqrt{(1 + \beta)^2 + 1} \sum_{j=1}^k R(\alpha, \beta)^{k-j} \norm{z_k}.
\end{align}
Taking the expectation gives
\begin{align}
\E_k \norm{y_{k+1} - x^\star} &\le \E_k \norm{\begin{bmatrix} r_{k+1} \\ v_k\end{bmatrix}} \\
&\le  R(\alpha, \beta)^k \norm{x_0 - x^\star} + \frac{\alpha \sqrt{(1 + \beta)^2 + 1}}{1 - R(\alpha, \beta)^2} \sigma.
\end{align}

\section{Proof of Corollaries~\ref{cor:finite-sum} and~\ref{cor:finite-sum-general}}
\label{sec:proof-of-cor-finite-sum}

When $\beta = 0$, we have $y_{k+1} = x_k$ and $v_k = -\alpha g_k$ for all $k$. In this case we have
\[
    r_{k+1} = r_k - \alpha g_k.
\]
Since the objectives $f_i$ are twice continuously differentiable, the mini-batch gradients can again be written as (using the same notation as in the proof of Theorem~\ref{th:finite-sum})
\[
    g_k = \tilde{H}_k r_k + z_k.
\]
Thus, with $A_k = I - \alpha \tilde{H}_k$, we have
\begin{align*}
r_{k+1} &= A_k r_k - \alpha z_k \\
&= A_k \cdots A_1 r_1 - \alpha z_k - \alpha \sum_{j=1}^{k-1} (A_k \cdots A_{j+1}) z_k.
\end{align*}
Since $\tilde{H}_k$ is symmetric, it follows that $A_k$ is also symmetric, and so $\norm{A_k}$ is equal to the largest magnitude of any eigenvalue of $A_k$. Recall that all eigenvalues of $\tilde{H}_k$ lie in the interval $[\mu, L]$.
Therefore, $\norm{A_k} \le \max_{\lambda \in [\mu, L]} \abs{1 - \alpha \lambda} = \max\{ \abs{1-\alpha\mu}, \abs{1-\alpha L}\}$.
Choosing $\alpha < \frac{2}{L}$ and taking the norm and expectation thus yields that
\begin{align}
\label{eq:proof-cor-finite-sum}
\begin{split}
\E_k \norm{x_k - x^\star} &= \E_k \norm{r_{k+1}} \\
&\le \abs{1 - \alpha \tilde\lambda}^k \norm{x_0 - x^\star} + \frac{\alpha}{1 - \abs{1 - \alpha \tilde\lambda}} \sigma,
\end{split}
\end{align}
where $\tilde \lambda \defeq \argmax_{\lambda \in \{\mu, L\}}\abs{1-\alpha \lambda}$.
When $\alpha = \frac{2}{\mu + L}$, we have that $\max_{\lambda \in [\mu, L]} \abs{1 - \alpha \lambda} = \frac{\cond - 1}{\cond + 1}$, and equation~\eqref{eq:proof-cor-finite-sum} simplifies as
\begin{align*}
\E_k \norm{x_k - x^\star} &= \E_k \norm{r_{k+1}} \\
&\le \left(\frac{\cond - 1}{\cond + 1}\right)^k \norm{x_0 - x^\star} + \frac{1}{\mu} \sigma.
\end{align*}

\section{Additional Experiments}
\label{sec:additional_experiments}

\subsection{Least Squares}
\begin{figure*}[!t]
\centering
\subfloat[.5\textwidth][$(\cond = 2)$: Coefficient multiplying $\sigma^2$]{
    \includegraphics[width=.25\textwidth]{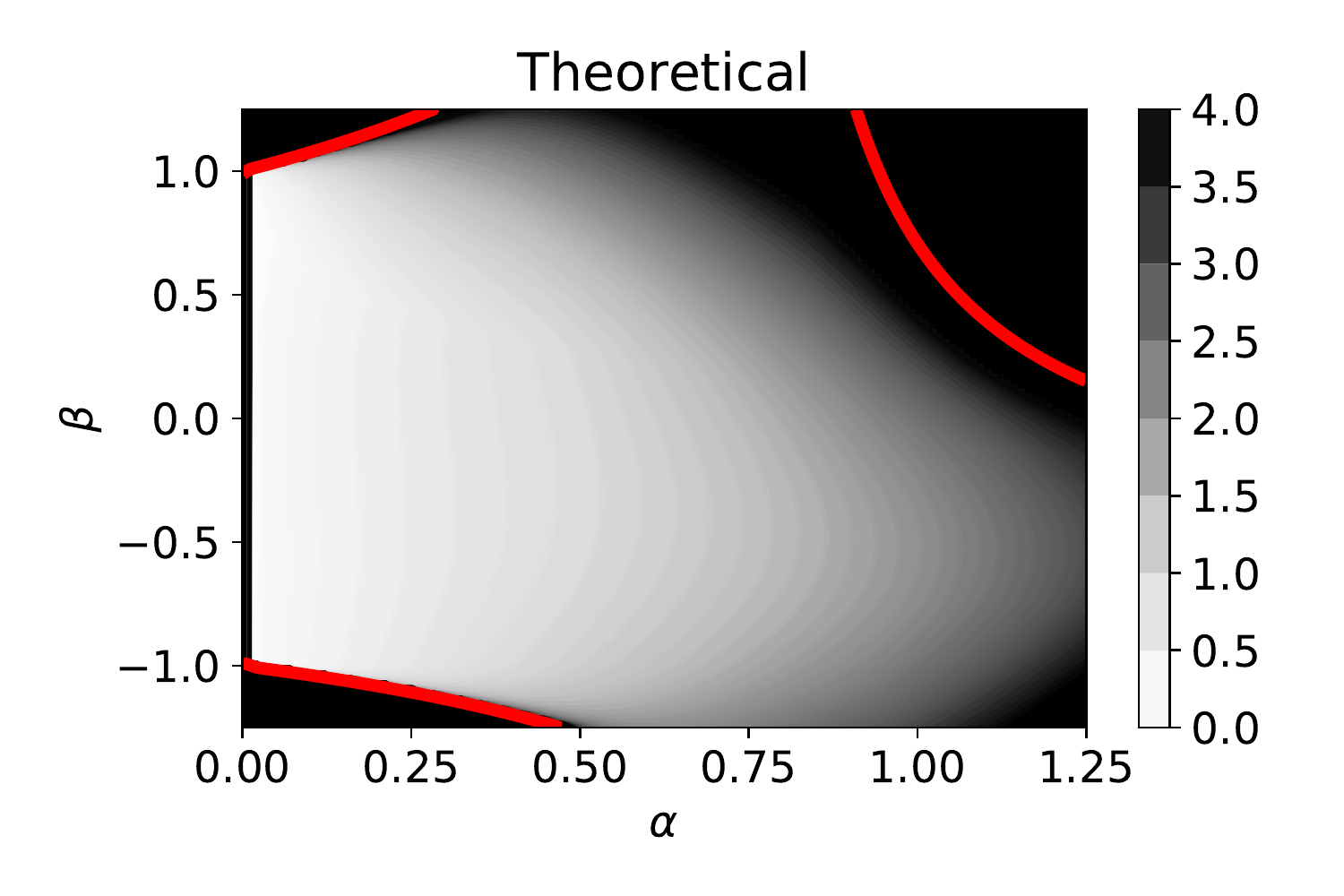}
    \includegraphics[width=.25\textwidth]{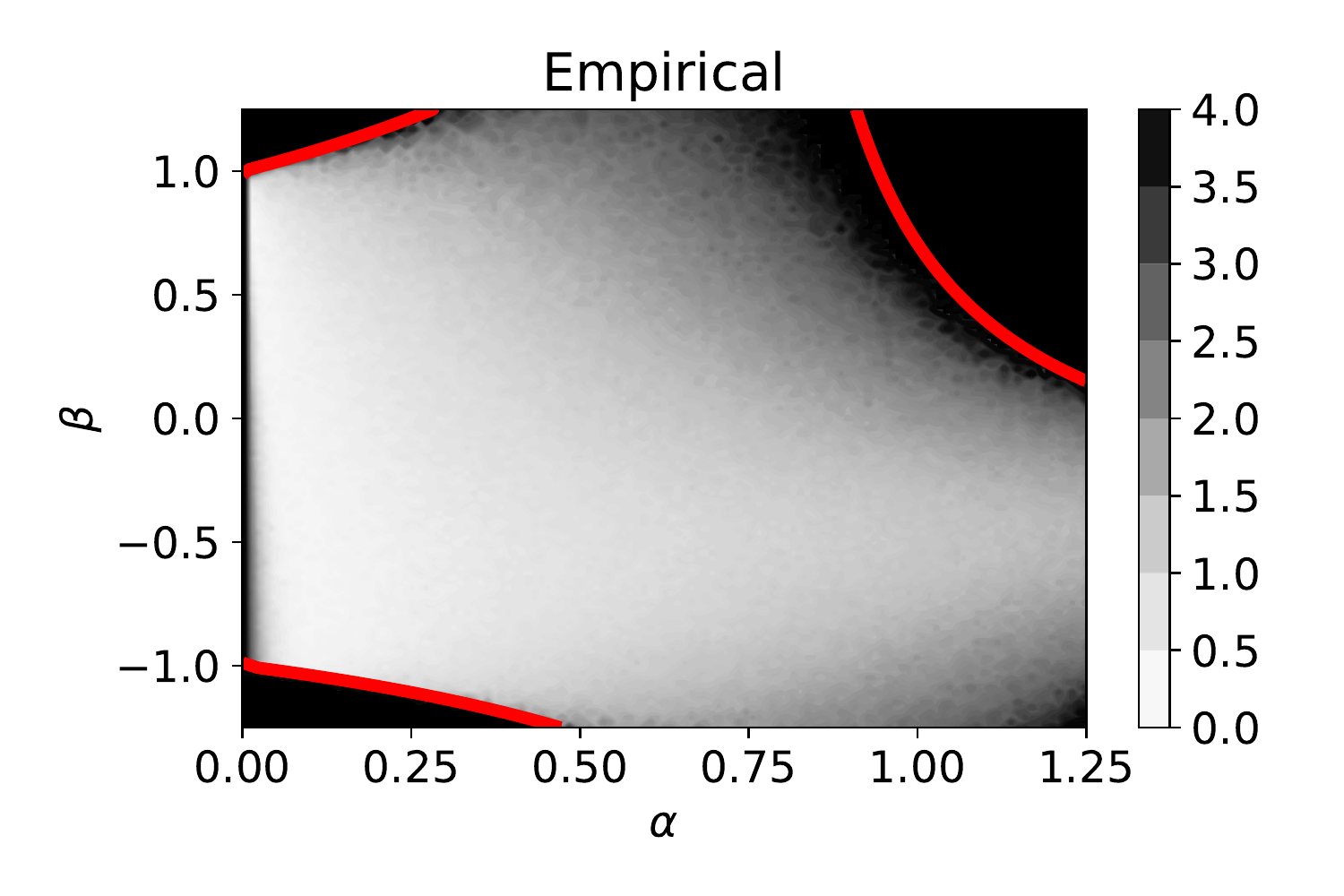}
    \label{sub-fig:least-squares-sa-1}
}
\subfloat[.5\textwidth][$(\cond = 2)$: Convergence rate $\rho(\alpha, \beta)$]{
    \includegraphics[width=.25\textwidth]{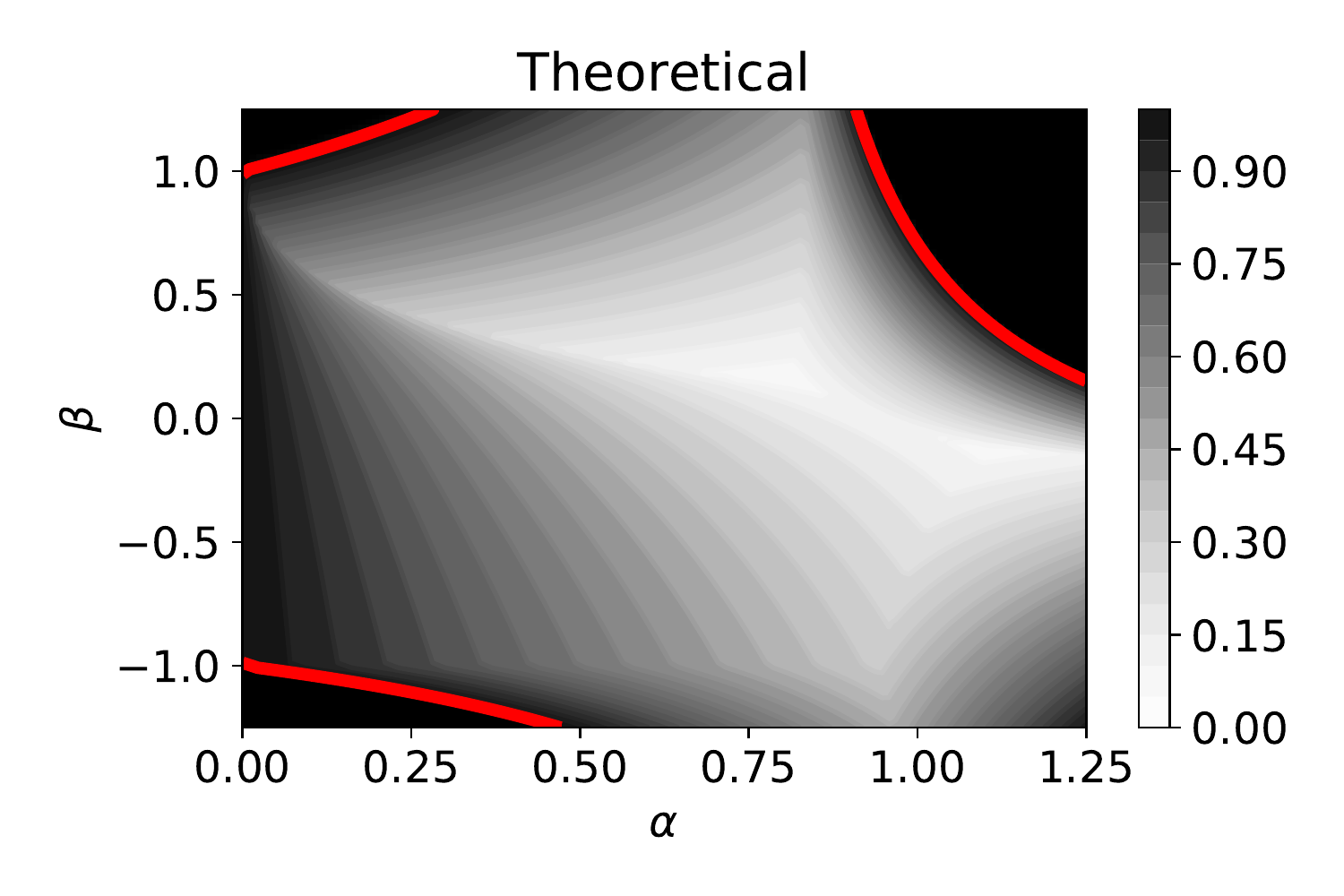}
    \includegraphics[width=.25\textwidth]{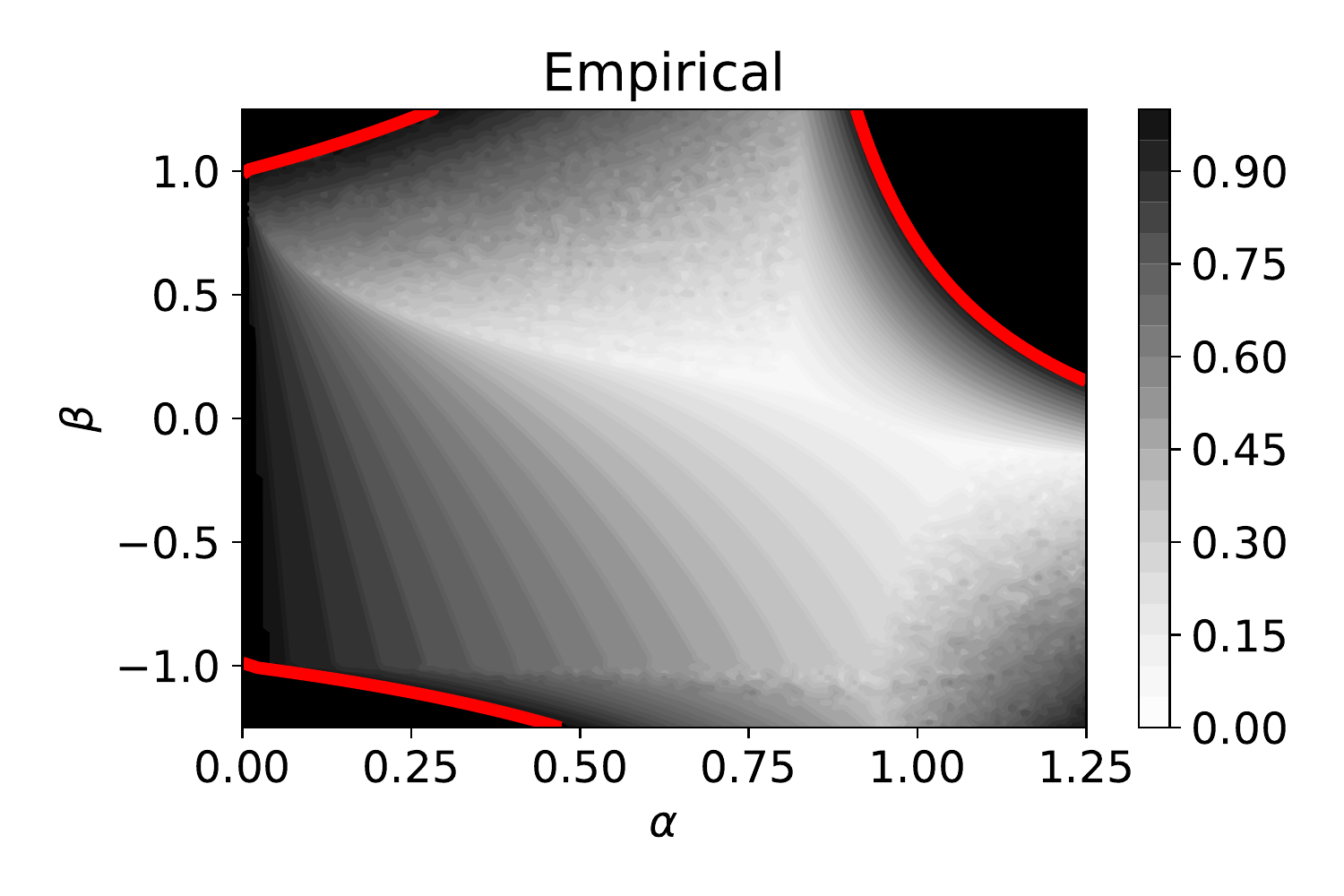}
    \label{sub-fig:least-squares-sa-2}
} \\
\subfloat[.5\textwidth][$(\cond = 8)$: Coefficient multiplying $\sigma^2$]{
    \includegraphics[width=.25\textwidth]{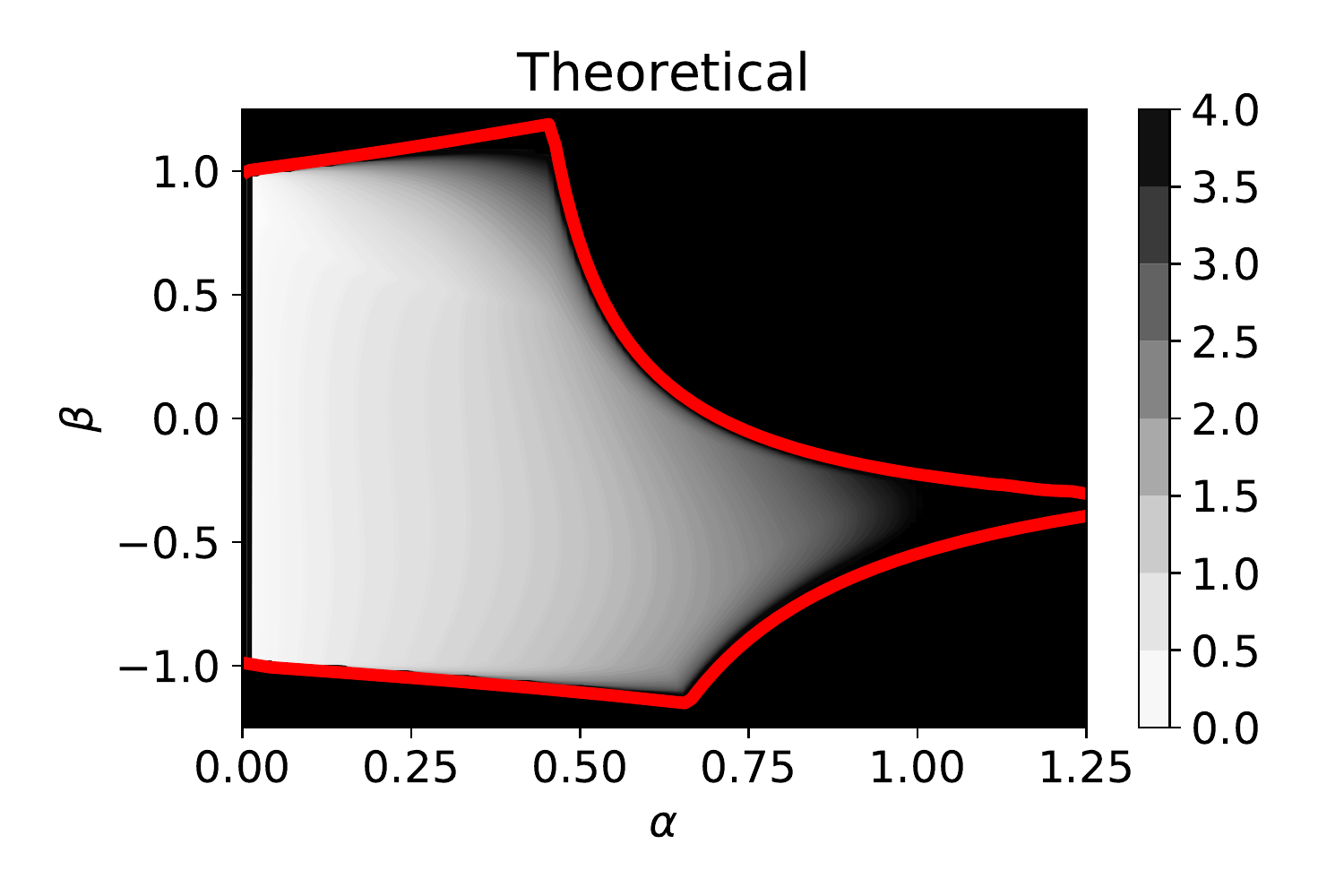}
    \includegraphics[width=.25\textwidth]{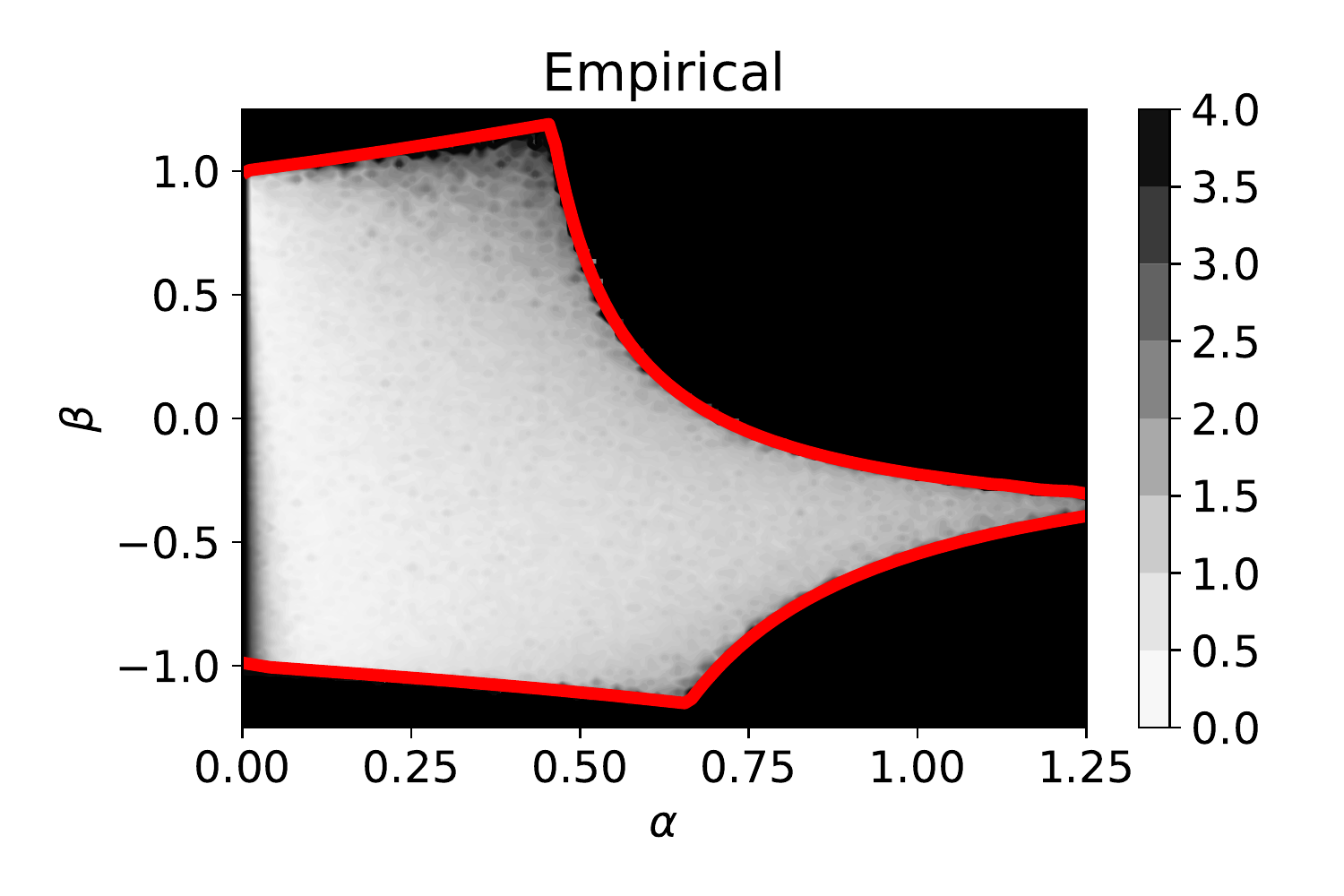}
    \label{sub-fig:least-squares-sa-3}
}
\subfloat[.5\textwidth][$(\cond = 8)$: Convergence rate $\rho(\alpha, \beta)$]{
    \includegraphics[width=.25\textwidth]{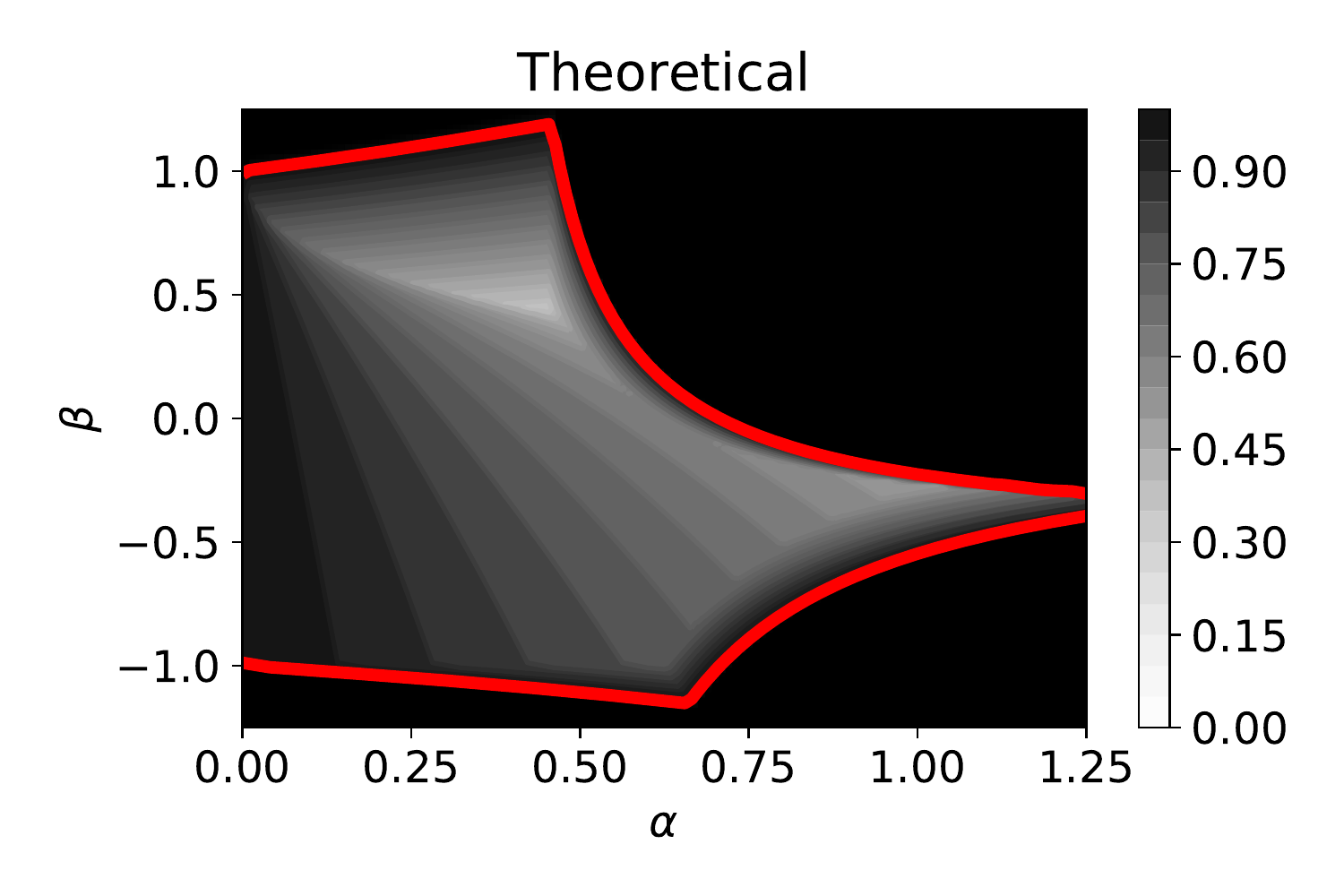}
    \includegraphics[width=.25\textwidth]{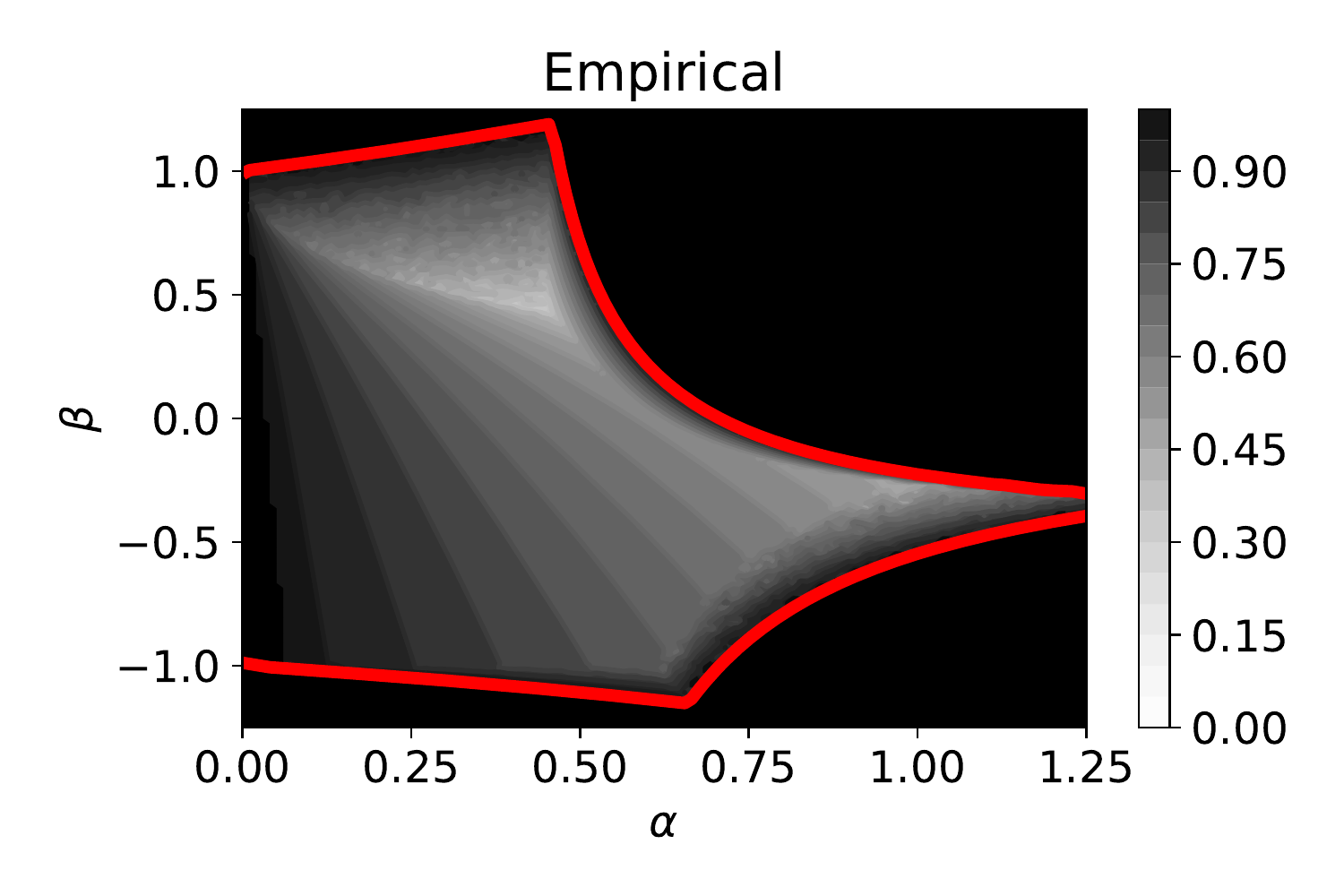}
    \label{sub-fig:least-squares-sa-4}
} \\
\subfloat[.5\textwidth][$(\cond = 32)$: Coefficient multiplying $\sigma^2$]{
    \includegraphics[width=.25\textwidth]{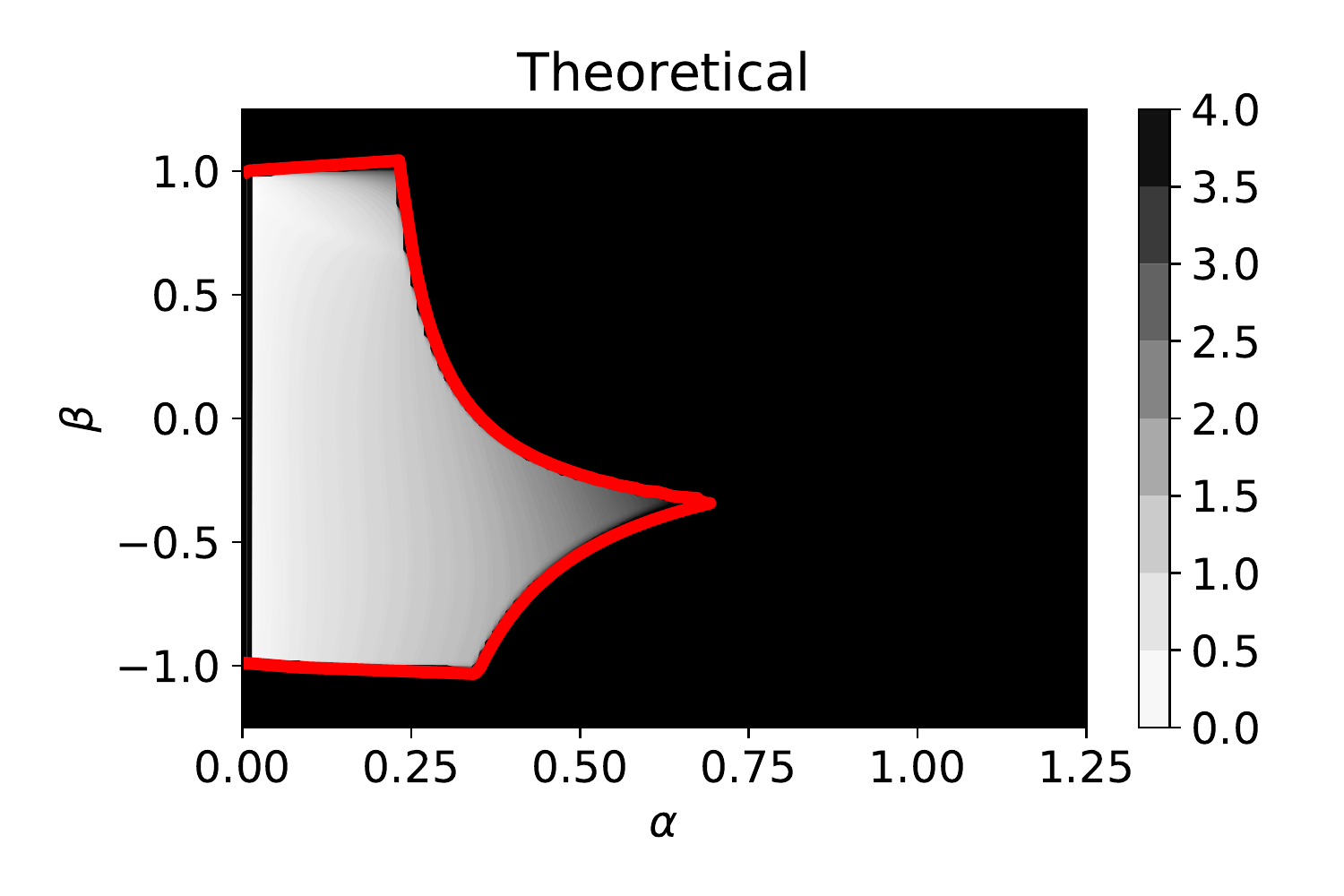}
    \includegraphics[width=.25\textwidth]{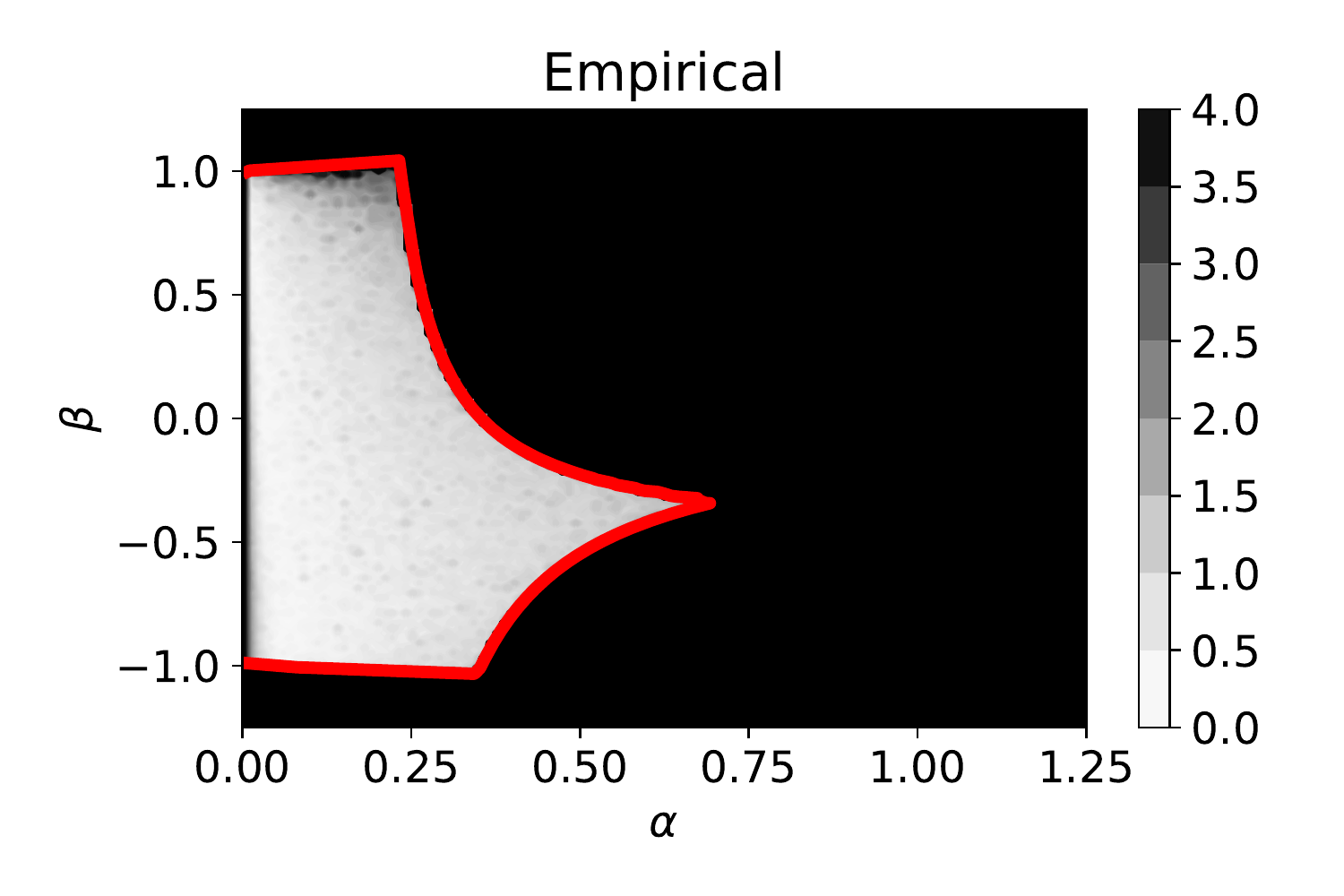}
    \label{sub-fig:least-squares-sa-5}
}
\subfloat[.5\textwidth][$(\cond = 32)$ Convergence rate $\rho(\alpha, \beta)$]{
    \includegraphics[width=.25\textwidth]{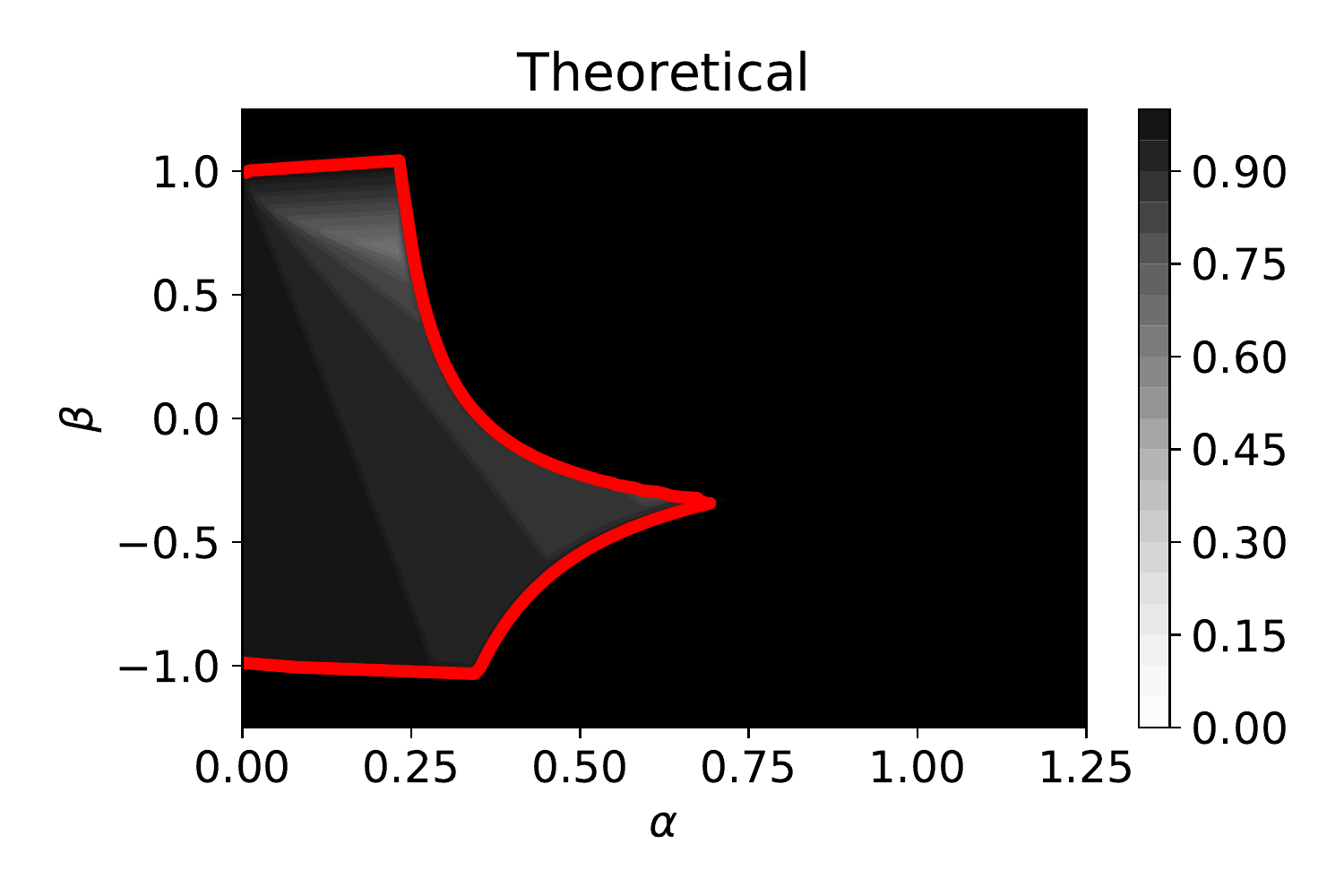}
    \includegraphics[width=.25\textwidth]{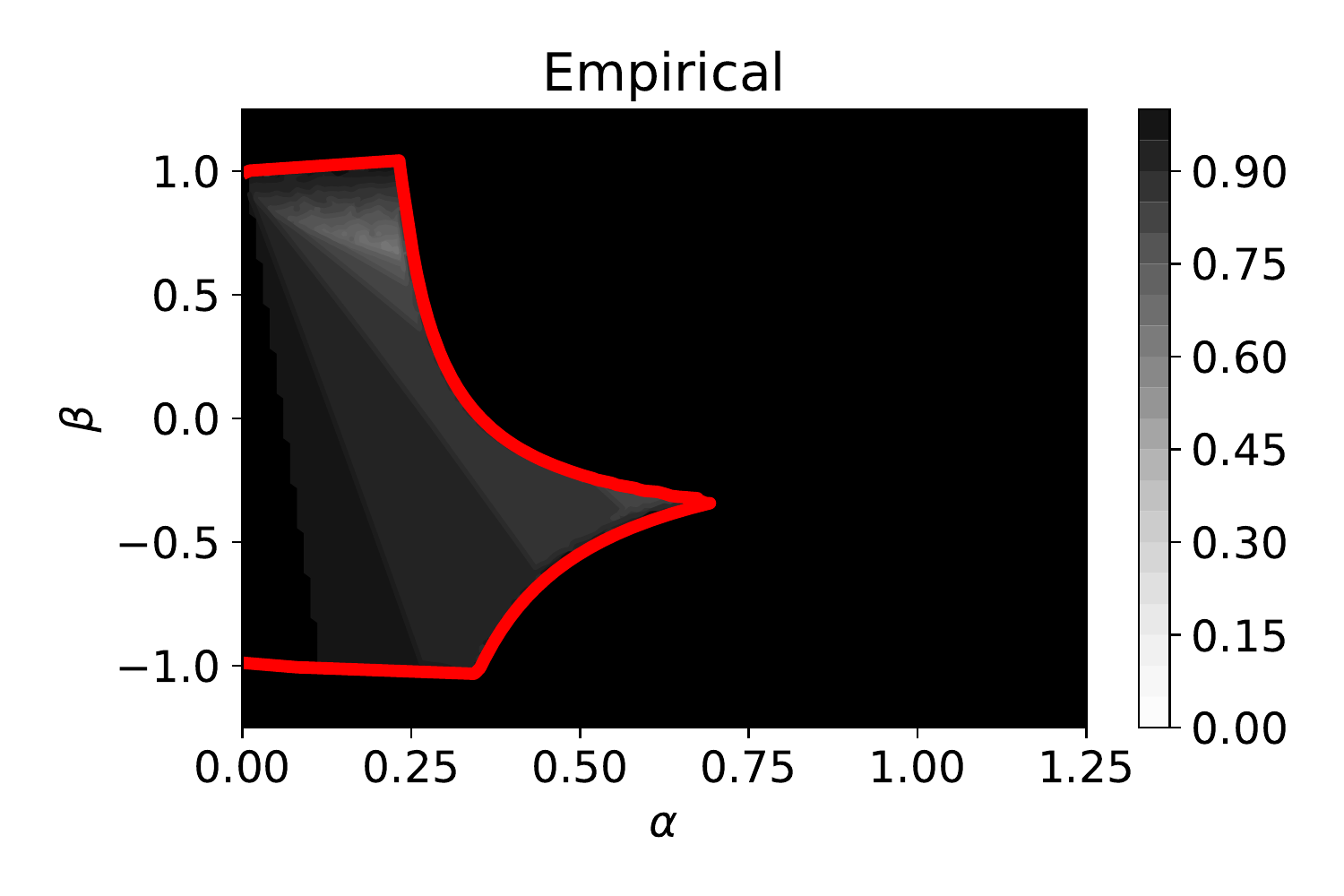}
    \label{sub-fig:least-squares-sa-6}
}
\caption{
    Visualizing the accuracy with which the theory predicts the coefficient of the variance term and the convergence rate for different choices of constant step-size and momentum parameters, and various objective condition numbers $\cond$.
    Plots labeled ``Theoretical'' depict theoretical results from Theorem~\ref{thm:stochapprox-quadratic}.
    Plots labeled ``Empirical'' depict empirical results when using the \NASG method to solve a least-squares regression problem with additive Gaussian noise; each pixel corresponds to an independent run of the \NASG method for a specific choice of constant step-size and momentum parameters.
    In all figures, the area enclosed by the red contour depicts the theoretical stability region from Theorem~\ref{thm:stochapprox-quadratic} for which $\rho(\alpha, \beta) < 1$.
    Fig.~\ref{sub-fig:least-squares-sa-1}/\ref{sub-fig:least-squares-sa-3}/\ref{sub-fig:least-squares-sa-5}:
    Pixel intensities correspond to the coefficient of the variance term in Theorem~\ref{thm:stochapprox-quadratic} ($\lim_{k \to \infty} \frac{1}{\std} \E \norm{\itr{y}{k}-\xstar}_\infty$), which provides a good characterization of the magnitude of the neighbourhood of convergence, even without explicit knowledge of the constant $C_\epsilon$.
    Fig.~\ref{sub-fig:least-squares-sa-2}/\ref{sub-fig:least-squares-sa-4}/\ref{sub-fig:least-squares-sa-6}: Pixel intensities correspond to the theoretical convergence rates in Theorem~\ref{thm:stochapprox-quadratic}, which provides a good characterization of the empirical convergence rates.
    Moreover, the theoretical conditions for convergence in Theorem~\ref{thm:stochapprox-quadratic} depicted by the red-contour are tight.
}
\label{fig:least-squares-stochastic-approximation}
\end{figure*}

To provide additional experiments illustrating the relationship between empirical observations and the theory developed in Section~\ref{sec:stochastic-approximation} for the stochastic approximation setting, we conduct additional experiments on randomly-generated least-squares problems.
We generate the least-squares problem using the approach described in~\citep{lenard1984randomly}.
Visualizations are shown in Figure~\ref{fig:least-squares-stochastic-approximation}.

We run the \NASG method on least-squares regression problems with various condition numbers $\cond$. The objectives $f$ correspond to randomly generated least squares problems, consisting of 500 data samples with 10 features each.
Stochastic gradients are sampled by adding zero-mean Gaussian noise, with standard-deviation $\sigma = 0.25$, to the true gradient.
The left plots in each sub-figure depict theoretical predictions from Theorem~\ref{thm:stochapprox-quadratic}, while the right plots in each sub-figure depict empirical results.
Each pixel corresponds to an independent run of the \NASG method for a specific choice of constant step-size and momentum parameters.
In all figures, the area enclosed by the red contour depicts the theoretical stability region from Theorem~\ref{thm:stochapprox-quadratic} for which $\rho(\alpha, \beta) < 1$.

Figures~\ref{sub-fig:least-squares-sa-1}/\ref{sub-fig:least-squares-sa-3}/\ref{sub-fig:least-squares-sa-5} showcase the coefficient multiplying the variance term, which is taken to be $\frac{\alpha^2((1 + \beta)^2 + 1)}{1 - \rho(\alpha, \beta)^2}$ in theory.
Brighter regions correspond to smaller coefficients, while darker regions correspond to larger coefficients.
All sets of figures (theoretical and empirical) use the same color scale.
We can see that the coefficient of the variance term in Theorem~\ref{thm:stochapprox-quadratic} provides a good characterization of the magnitude of the neighbourhood of convergence.
The constant $C_\epsilon$ is approximated as $1+ (1-\rho(\alpha, \beta)^2)(\varrho(\alpha, \beta)^2 - \rho(\alpha, \beta)^2)$, where $\varrho(\alpha, \beta)$ is defined as the largest singular value of $A$ in~\eqref{eq:stochapprox-quadratic-A}, and $\rho(\alpha,\beta)$ is the largest eigenvalue of $A$.

Figures.~\ref{sub-fig:least-squares-sa-2}/\ref{sub-fig:least-squares-sa-4}/\ref{sub-fig:least-squares-sa-6} showcase the linear convergence rate in theory and in practice.
Brighter regions correspond to faster rates, and darker regions correspond to slower rates.
Again, all figures (theoretical and empirical) use the same color scale.
We can see that the theoretical linear convergence rates in Theorem~\ref{thm:stochapprox-quadratic} provide a good characterization of the empirical convergence rates.
Moreover, the theoretical conditions for convergence in Theorem~\ref{thm:stochapprox-quadratic} depicted by the red-contour appear to be tight.

\subsection{Multinomial Logistic Regression}

Next we conduct experiments on $\ell_2$ regularized multinomial logistic regression problems with additive Gaussian noise, to examine whether the \NASG method still achieves acceleration over \SGD for these problems in the stochastic approximation setting, as is predicted by the theory in Section~\ref{sec:stochastic-approximation}.
These problems are smooth and strongly-convex, but non-quadratic.
Tight estimates of the smoothness constant $L$ and the modulus of strong-convexity $\mu$ cannot be computed definitively since the eigenvalues of the Hessian vary throughout the parameter space.

We randomly generate multi-class classification problems consisting of 5 classes and 100 data samples with 10 features each, only five of which are discriminative.
We create one data cluster per class, and vary the cluster separation and regularization parameter to vary the condition number $Q$.
For reporting purposes, we estimate the condition number $Q$ during training by evaluating the eigenvalues of the Hessian at each iteration.
The smoothness constant $L$ is taken to be the maximum eigenvalue seen during training, and the modulus of strong-convexity $\mu$ is taken to be the minimum eigenvalue seen during training.
We use the \texttt{make\_classification()} function in scikit-learn~\citep{scikit-learn} to generate random classification problem instances.
\begin{figure*}[!t]
\centering
\subfloat[.33\textwidth][$Q=30$]{
    \includegraphics[width=.33\textwidth]{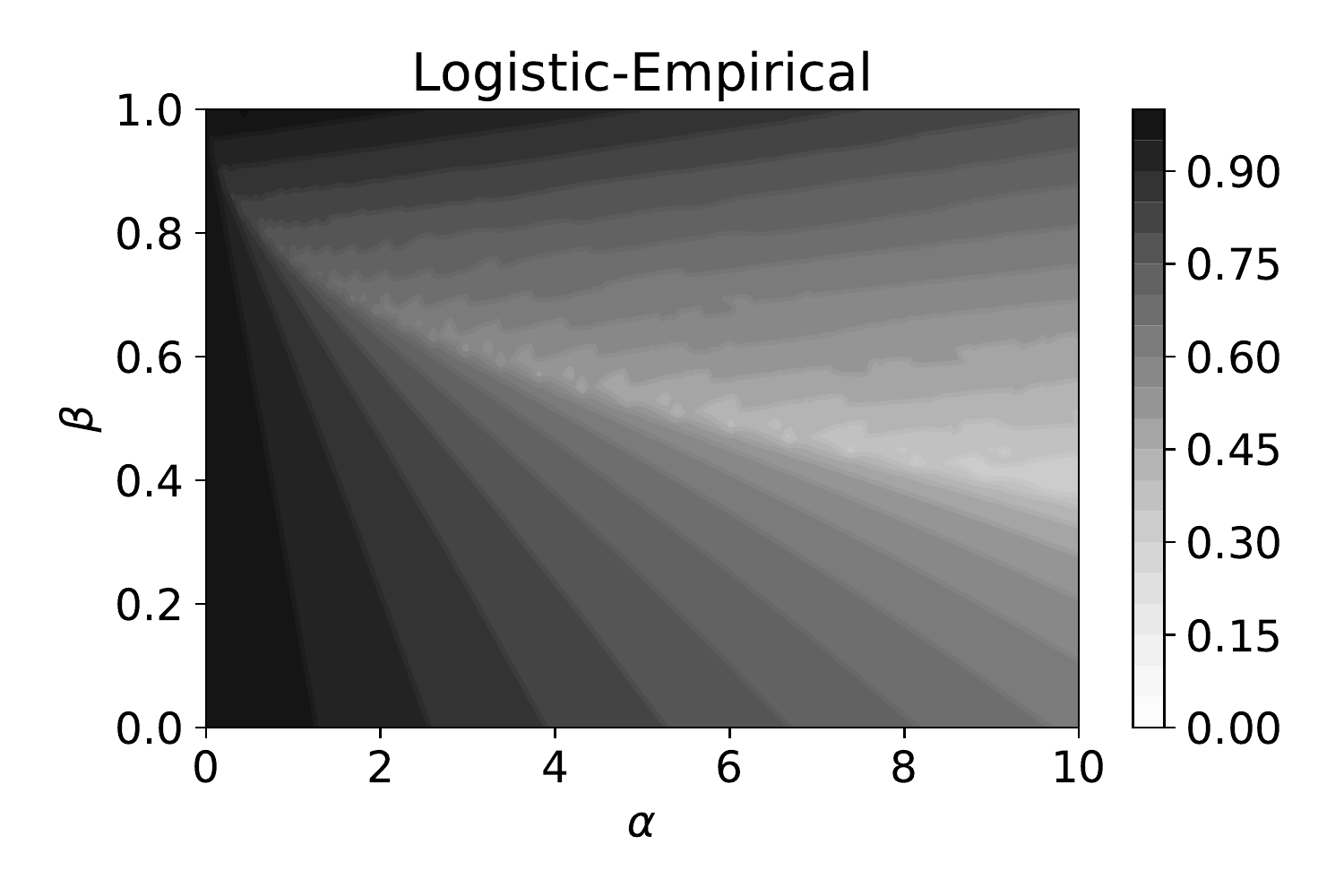}
    \label{sub-fig:logistic-sa-1}}
\subfloat[.33\textwidth][$Q=45$]{
    \includegraphics[width=.33\textwidth]{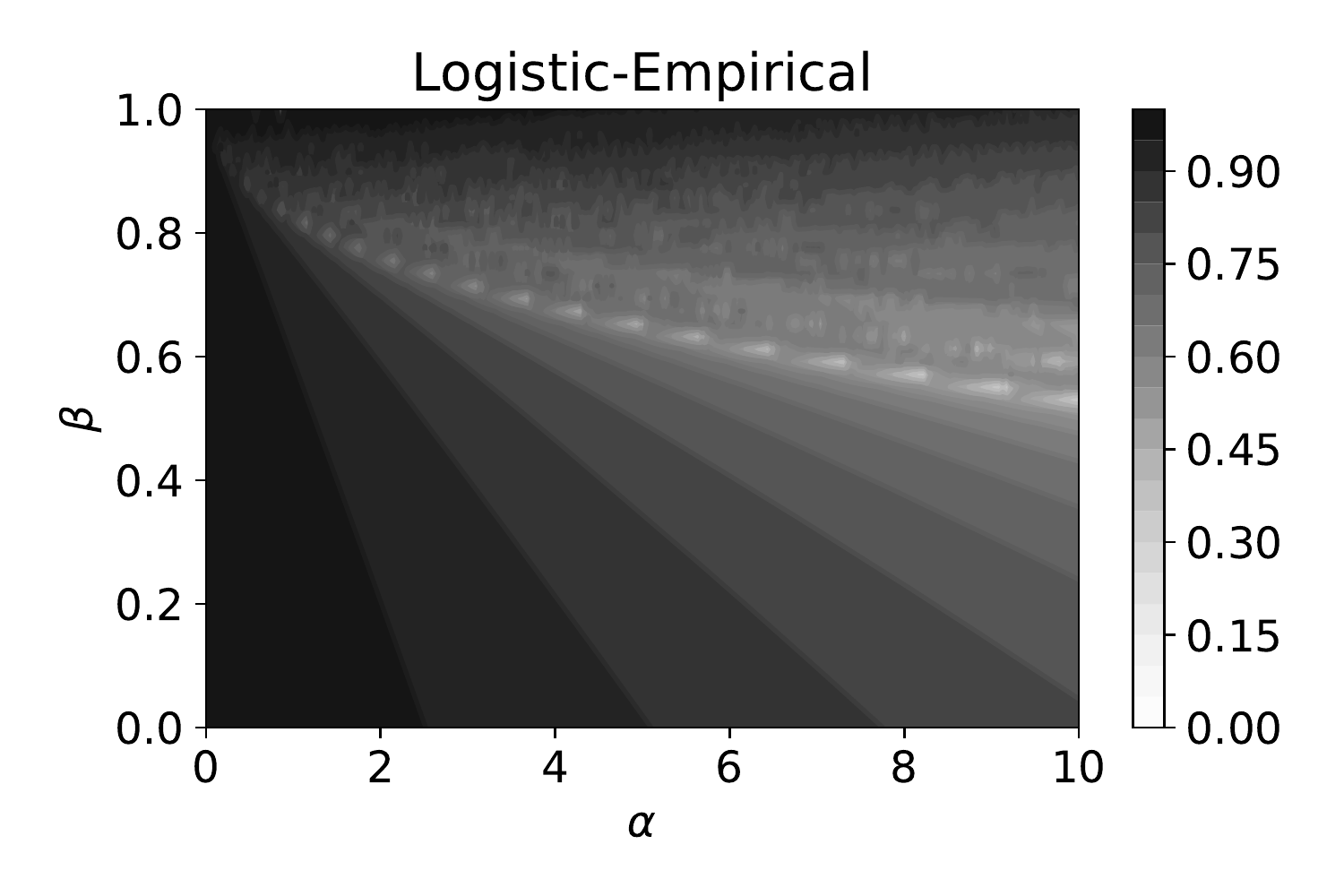}
    \label{sub-fig:logistic-sa-2}}
\subfloat[.33\textwidth][$Q=60$]{
    \includegraphics[width=.33\textwidth]{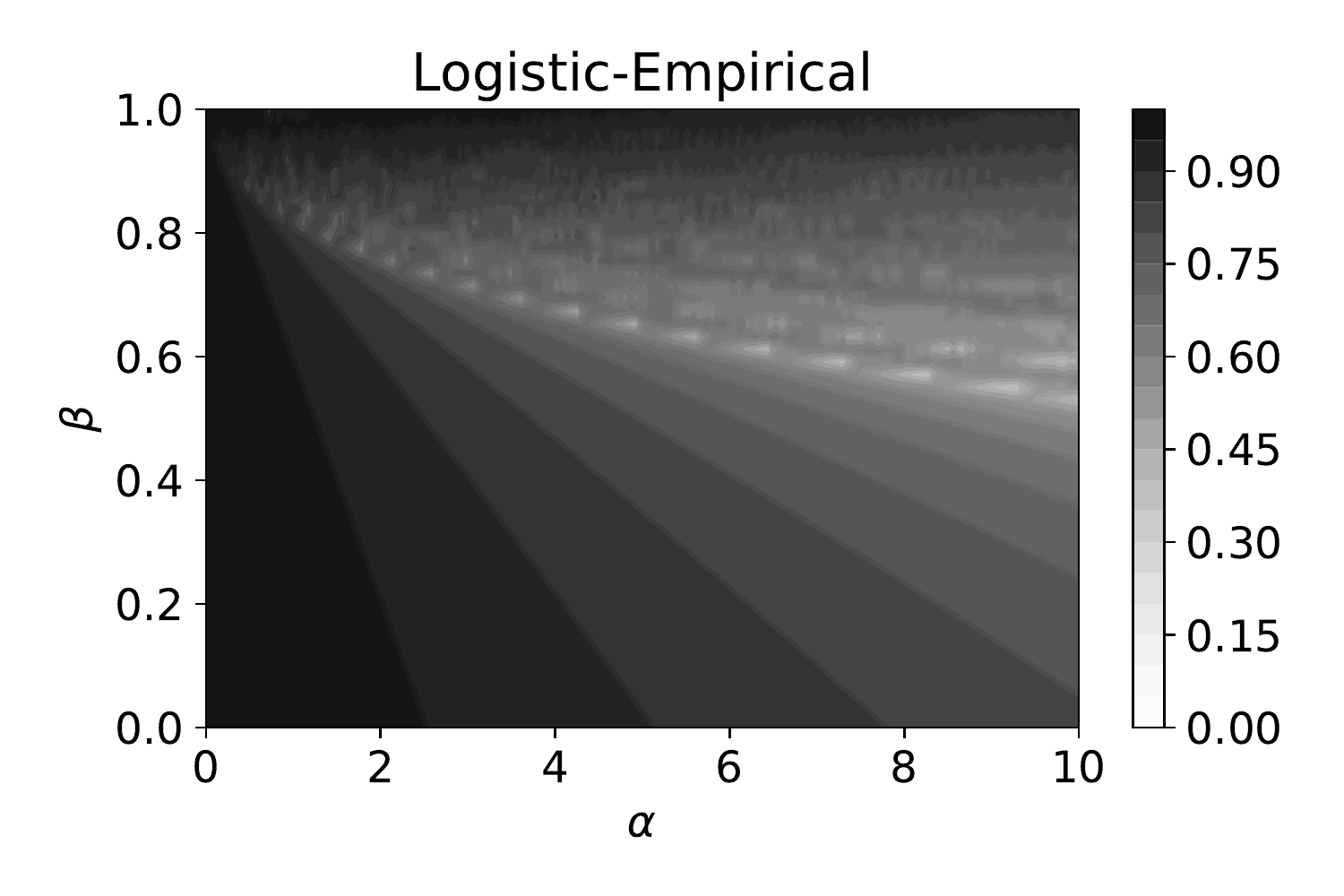}
    \label{sub-fig:logistic-sa-3}}
\caption{
    Visualizing the convergence rate for the \NASG method (momentum $\beta > 0$) and the \SGD method (momentum $\beta = 0$), for various randomly generated $\ell_2$ regularized multinomial logistic-regression problems.
    Multi-class classification problems consist of 5 classes and 100 data samples with 10 features each, only 5 of which are are discriminative.
    We create one data-cluster per class, and vary the cluster separation and regularization parameter to vary the condition number $Q$.
    For reporting purposes, we estimate the condition number $Q$ during training by evaluating the eigenvalues of the Hessian at each iteration.
    The smoothness constant $L$ is taken to be the maximum eigenvalue seen during training, and the modulus of strong-convexity $\mu$ is taken to be the minimum eigenvalue seen during training.
    The faster convergence rates (brighter regions) correspond to $\beta > 0$, indicating that the \NASG method provides acceleration over \SGD in this stochastic approximation setting.
    Moreover, for a given step-size, the contrast between the brighter regions ($\beta > 0$) and darker regions ($\beta = 0$) increases as the condition number grows, supporting theoretical findings that the convergence rate of the \NASG method exhibits a better dependence on the condition number than \SGD. 
}
\label{fig:logistic-stochastic-approximation}
\end{figure*}

Visualizations are provided in Figure~\ref{fig:logistic-stochastic-approximation}. Each pixel corresponds to an independent run of the \NASG method for a specific choice of constant step-size and momentum parameters.
Pixel intensities denote the linear convergence rates observed in practice. Brighter regions correspond to faster rates, and darker regions correspond to slower rates.

The parameter setting $\beta$ equals 0 corresponds to \SGD, and the parameter setting $\beta > 0$ corresponds to the \NASG method.
The faster convergence rates (brighter regions) correspond to $\beta > 0$, indicating that the \NASG method provides acceleration over \SGD in this stochastic approximation setting.
Moreover, for a given step-size, the contrast between the brighter regions ($\beta > 0$) and darker regions ($\beta = 0$) increases as the condition number grows, supporting theoretical findings that the convergence rate of the \NASG method exhibits a better dependence on the condition number than \SGD. 
\end{document}